\documentclass{article}
\usepackage{xcolor}

\usepackage{float}
\usepackage[accepted]{icml2024}
\usepackage{algorithm}
\usepackage{algorithmic}
\usepackage[utf8]{inputenc} 
\usepackage[T1]{fontenc}    
\usepackage{hyperref}       
\usepackage{url}            
\usepackage{booktabs}       
\usepackage{amsfonts}       
\usepackage{nicefrac}       
\usepackage{microtype}      
\usepackage{amsmath}
\usepackage{amsthm}
\usepackage{natbib}
\usepackage{nomentbl}
\usepackage{thmtools}
\usepackage{thm-restate}
\usepackage{cleveref}
\usepackage{booktabs}
\usepackage{bm}
\usepackage[title]{appendix}
\usepackage{titletoc}
\usepackage{minitoc}
\usepackage{caption}
\usepackage{subcaption}
\usepackage{comment}
\usepackage{authblk}
\usepackage{wrapfig}
\makenomenclature
\hypersetup{colorlinks,linkcolor={blue},citecolor={green!50!black},urlcolor={blue}}
\definecolor{mypink2}{RGB}{219, 48, 122}
\newtheorem{definition}{Definition}[section]
\newtheorem{theorem}{Theorem}[section]
\newtheorem{proposition}{Proposition}[section]
\allowdisplaybreaks
\usepackage[textwidth=1.8cm, textsize=scriptsize]{todonotes}
\newcommand{\tmo}{m}
\newcommand{\qmo}{u}
\newcommand{\Qmo}{U}

\renewcommand{\cal}[1]{\mathcal{#1}}

\renewcommand{\rm}[1]{\mathrm{#1}}
\renewcommand{\sf}[1]{\mathsf{#1}}
\newcommand{\bb}[1]{\mathbb{#1}}
\renewcommand{\r}{\mathbb{R}}

\newcommand{\sbrac}[1]{\left [ #1 \right ]}

\newcommand{\norm}[1]{\left|\left|{#1}\right|\right|}
\newcommand{\iprod}[2]{\left\langle{#1},{#2}\right\rangle}

\newcommand{\Ebbs}[2]{\mathbb{E}_{#1}\left[#2\right]}

\newtheorem*{remark*}{Remark}
\newtheorem{assumption}{Assumption}

\newtheorem{example}{Example}

\definecolor{mplPurple}{RGB}{125, 0, 125}

\icmltitlerunning{Momentum Particle Maximum Likelihood}

\begin{document}

\twocolumn[
\icmltitle{Momentum Particle Maximum Likelihood}
\icmlsetsymbol{equal}{*}

\begin{icmlauthorlist}
\icmlauthor{Jen Ning Lim}{warwick}
\icmlauthor{Juan Kuntz}{poly}
\icmlauthor{Samuel Power}{bris}
\icmlauthor{Adam M. Johansen}{warwick}
\end{icmlauthorlist}
\icmlaffiliation{warwick}{University of Warwick}
\icmlaffiliation{bris}{University of Bristol}
\icmlaffiliation{poly}{Polygeist}
\icmlkeywords{Machine Learning, ICML}
\vskip 0.3in
\icmlcorrespondingauthor{Jen Ning Lim}{Jen-Ning.Lim@warwick.ac.uk}
]
\printAffiliationsAndNotice{}
\begin{abstract}
Maximum likelihood estimation (MLE) of latent variable models is often recast as the minimization of a free energy functional over an extended space of parameters and probability distributions. This perspective was recently combined with insights from optimal transport to obtain novel particle-based algorithms for fitting latent variable models to data. Drawing inspiration from prior works which interpret `momentum-enriched' optimization algorithms as discretizations of ordinary differential equations, we propose an analogous dynamical-systems-inspired approach to minimizing the free energy functional. The result is a dynamical system that blends elements of Nesterov's Accelerated Gradient method, the underdamped Langevin diffusion, and particle methods. Under suitable assumptions, we prove that the continuous-time system minimizes the functional. By discretizing the system, we obtain a practical algorithm for MLE in latent variable models. The algorithm outperforms existing particle methods in numerical experiments and compares favourably with other MLE algorithms.
\end{abstract}

\section{Introduction}

In this work, we study parameter estimation for (probabilistic) latent variable models $p_{\theta} \left( y, x\right)$ with parameters $\theta \in \mathbb{R}^{d_{\theta}}$, unobserved (or latent) variables $x \in \mathbb{R}^{d_{x}}$, and observed variables $y \in \mathbb{R}^{d_{y}}$ (which we treat as fixed throughout). The type II maximum likelihood approach \citep{good1983} estimates $\theta$ by maximizing the \emph{marginal likelihood} $p_{\theta} \left( y \right) := \int p_{\theta} \left( y, x \right) \, \mathrm{d}x$. However, for most models of practical interest, the integral has no known closed-form expressions and we are unable to optimize the marginal likelihood directly.

This issue is often overcome (e.g.,\ \citet{neal1998view}) by constructing an objective defined over an extended space, whose optima are in one-to-one correspondence with those of the MLE problem. To this end, we define the `free energy' functional:
\begin{align}\label{eq:freeenergy}
    \mathcal{E} \left( \theta, q \right) := \int  \log \left( \frac{q \left( x \right)}{p_{\theta} \left( y, x \right)} \right)q \left( x \right) \, \mathrm{d}x.
\end{align}
The infimum of $\cal{E}$ over the \emph{extended space} $\r^{d_\theta}\times\cal{P}(\r^{d_x})$ (where $\cal{P}(\r^{d_x})$ denotes the space of probability distributions over $\r^{d_x}$) coincides with negative of the log marginal likelihood's supremum:
\begin{equation}\label{eq:freeenergyinf}\cal{E}^*:=\inf_{(\theta,q)}\cal{E}(\theta,q)=-\sup_{\theta\in \mathbb{R}^{d_\theta}}\log p_\theta(y).\end{equation}
To see this, note that 
\begin{equation}\label{eq:freeenergyrKL}\mathcal{E} \left( \theta, q \right) = - \log p_{\theta} \left( y \right) + \mathsf{KL}\left( q, p_{\theta} \left( \cdot \mid y \right) \right),\end{equation}
where $\mathsf{KL}$ denotes the Kullback--Leibler divergence and $p_\theta(\cdot|y):=p_\theta(y,\cdot)/p_\theta(y)$ the posterior distribution. So, if  $(\theta^*,q^*)$ minimizes the free energy, then $\theta^*$ maximizes the marginal likelihood and $q^*(\cdot)=p_{\theta^*}(\cdot|y)$. In other words, by minimizing the free energy, we solve our MLE problem.

This perspective motivates the search for practical procedures that minimize the free energy $\mathcal{E}$. One such example is the classical Expectation Maximization (EM) algorithm applicable to models for which the posterior distributions $p_{\theta} \left( \cdot | y \right)$ are available in closed form. As shown in \citet{neal1998view}, its iterates coincide with those of coordinate descent applied to the free energy $\cal{E}$. Recently, \citet{Kuntz2022} sought analogues of gradient descent (GD) applicable to the free energy $\cal{E}$. In particular, building on ideas popular in optimal transport (e.g., see \citet{ambrosio2005gradient}), they identified an appropriate notion for the free energy's gradient $\nabla \cal{E}$, discretized the corresponding gradient flow $(\dot{\theta}_t,\dot{q}_t)=-\nabla \cal{E}(\theta_t,q_t)$, and obtained a practical algorithm they called particle gradient descent (PGD).

In optimization, GD is well-known to be suboptimal: other practical first-order methods achieve better worst-case convergence guarantees and practical performance \citep{nemirovskij1983problem, nesterov1983method}. A common feature among algorithms that achieve the optimal `accelerated' convergence rates is the presence of `momentum' effects in the dynamics of the algorithm. Roughly speaking, if GD for a function $f$ is analogous to solving the ordinary differential equation (ODE)
\begin{equation}\label{eq:eucgf}
    \dot{\theta}_t = -\nabla_\theta f \left( \theta_t \right)\end{equation}
with $\nabla_\theta f$ denoting $f$'s Euclidean gradient, then a `momentum' method is akin to solving a second-order ODE like 
%
\begin{align}
    \dot{\theta}_t &= \eta m_t, \quad 
    \dot{m}_t = -\gamma \eta m_t - \nabla_\theta f \left( \theta_t \right),
    \label{eq:dhamil_euclidean}
\end{align}
where $\gamma, \eta$ denote positive hyperparameters~\citep{su2014differential,wibisono2016variational}. In the context of machine learning, \citet{sutskever2013importance} demonstrated empirically that incorporating momentum effects into stochastic gradient descent (SGD) often has substantial benefits such as eliminating the performance gap between standard SGD and competing `Hessian-free' deterministic methods based on higher-order differential information \citep{martens2010deep}.

Given the successes of momentum methods, it is natural to ask whether PGD can be similarly accelerated. In this work, we affirmatively answer this question. Our contributions are: (1) we construct a continuous-time flow that incorporates momentum effects into the gradient flow studied in \citet{Kuntz2022}; (2) we derive a discretization that outperforms PGD and compares competitively against other methods. As theoretical contributions: (3) we prove in \Cref{prop:existence_uniqueness} the existence and uniqueness of the flow's solutions; (4) under suitable conditions, we show in \Cref{prop:convergence_flow} that the flow converges to $\cal{E}$'s minima; and, (5) we asymptotically justify in \Cref{prop:chaos} the particle approximations we use when discretizing the flow. The proofs of \Cref{prop:existence_uniqueness} and \Cref{prop:chaos} require generalizations of proofs pertaining to McKean--Vlasov SDEs, and the proof of \Cref{prop:convergence_flow} is a generalization of the \citet{ma2019there}'s convergence argument. These may be of independent interest for future works that operate in the extended space $\mathbb{R}^{d_\theta} \times \cal{P}(\mathbb{R}^{d_x})$.

The structure of this manuscript is as follows: in  \Cref{sec:background}, we review gradient flows on the Euclidean space $\mathbb{R}^{d_\theta}$, probability space $\cal{P}(\mathbb{R}^{d_x})$, and extended space $\mathbb{R}^{d_\theta} \times \cal{P}(\mathbb{R}^{d_x})$. We also review momentum methods on $\mathbb{R}^{d_\theta}$ and $\cal{P}(\mathbb{R}^{d_x})$. In \Cref{sec:proposal}, we describe how momentum can be incorporated into the gradient flow on $\mathbb{R}^{d_\theta} \times \cal{P}(\mathbb{R}^{d_x})$, resulting in a dynamical system we refer to as the `momentum-enriched flow'. In \Cref{sec:mt_convergence}, we establish the flow's convergence. In \Cref{sec:mt_discretization}, we discretize the flow and obtain an MLE algorithm we refer to as Momentum Particle Descent (MPD). In \Cref{sec:experiments}, we demonstrate empirically the efficacy of our method and study the effects of various design choices. As a large-scale experiment, we benchmark our proposed method for training Variational Autoencoders \citep{kingma2013auto} against current methods for training latent variable models. We conclude in Section~\ref{sec:conc} with a discussion of our results, their limitations, and future work.

\section{Background}
\label{sec:background}

Our goal is to incorporate momentum into the gradient flow on the extended space $\r^{d_\theta}\times\cal{P}(\r^{d_x})$ considered in \citet{Kuntz2022} and obtain an algorithm for type II MLE that outperforms PGD. We begin by surveying the standard gradient flows on $\r^{d_\theta}\times\cal{P}(\r^{d_x})$'s component spaces and how momentum is incorporated in those settings.

\subsection{Gradient Flows on $\mathbb{R}^d$ and their Acceleration}\label{sec:euclidean}

Given a differentiable function $f : \mathbb{R}^d \to \mathbb{R}$, $f$'s Euclidean gradient flow is~\eqref{eq:eucgf}.
%
%
%
%
Discretizing \eqref{eq:eucgf} in time using a standard forward Euler integrator yields GD. As observed in~\citet{su2014differential,shi2021understanding}, algorithms such as \citet{nesterov1983method}'s accelerated gradient (NAG) that achieve the optimal convergence rates are instead obtained by discretising the second-order ODE~\eqref{eq:dhamil_euclidean}; cf.\ \Cref{sec:nag_discretization}. This ODE models the evolution of a kinetic particle with both position and momentum. In particular, for convex quadratic $f$, \eqref{eq:dhamil_euclidean} reproduces the dynamics of a damped harmonic oscillator \citep{mccall2010classical}. Hence, the ODE's hyperparameters have intuitive interpretations: $\eta^{-1}$ represents the particle's mass and $\gamma$ determines the damping's strength. 

We can rewrite~\eqref{eq:dhamil_euclidean} concisely as 
\begin{equation}\label{eq:euhamflow}(\dot{\theta}_t,\dot{m}_t) = - \mathsf{D}_\gamma \nabla_{ \left( \theta, m \right)} f_\eta \left( \theta_t, m_t \right),\end{equation}
where $f_\eta$ denotes the `Hamiltonian function' and $\mathsf{D}_\gamma$ the `damping matrix':
\begin{equation}\label{eq:hamilanddamp}f_\eta \left( \theta, m \right)=f \left( \theta \right) + \frac{\eta}{2} \| m \|^2,\enskip \mathsf{D}_\gamma :=  \left(
        \begin{array}{cc}
            0_{d} & -I_{d}\\
            I_{d} & \gamma I_{d}
        \end{array}\right).
\end{equation}
In other words,~\eqref{eq:dhamil_euclidean} is a `damped' or `conformal' Hamiltonian flow~\citep{mclachlan2001conformal, maddison2018hamiltonian}. In particular, we can interpret NAG as a discretisation of $F_\eta$'s `$\gamma$-damped Hamiltonian flow'~\citep{wibisono2016variational, wilson2016lyapunov}.

\subsection{Gradient Flows on $\mathcal{P} \left( \mathbb{R}^d \right)$ and their Acceleration}\label{sec:w2gradflow}
We can follow a similar program to solve optimization problems over $\mathcal{P} \left( \mathbb{R}^d \right)$. To do so, we require an analogue of the Euclidean gradient $\nabla_\theta f$ for (differentiable) functions $f$ on $\r^d$  applicable to (sufficiently-regular) functionals $F$ on $\cal{P}(\r^d)$. One way of defining $\nabla_\theta f(\theta)$ at a given point $\theta$  is as the unique vector satisfying 
\begin{equation}\label{eq:euclideanmetric}f ( \theta' ) =  f \left( \theta \right) + \langle \nabla_\theta f \left( \theta \right), \theta' - \theta \rangle + o(\sf{d}_E(\theta, \theta')),
\end{equation}
%
%
where $\iprod{\cdot}{\cdot}$ and $\sf{d}_E$ denote the Euclidean inner product and metric related by the formula
\begin{equation}\label{eq:length}
    \sf{d}_E(\theta,\theta')^2=\inf_\lambda \int_0^1\iprod{\lambda_t}{\lambda_t}\,\rm{d}t,
\end{equation}
where the infimum is taken over all (sufficiently-regular) curves $\lambda:[0,1]\to\r^d$ connecting $\theta$ and ${\theta}'$. Roughly speaking, we can reverse this process to define $F$'s gradient for a given metric on $\cal{P}(\r^d)$. Here, we focus on the Wasserstein-2 metric $\mathsf{W}_2$ for which \eqref{eq:length}'s analogue is the Benamou-Brenier formula~\citep[Eq.7.2]{peyre2019}:
\begin{align*}
    \mathsf{W}_2(q, \tilde{q})^2=\inf_\rho \int_0^1g_{\rho_t}(\dot{\rho}_t,\dot{\rho}_t)\,\rm{d}t
\end{align*}
where the infimum is taken over all (sufficiently-regular) curves $\rho:[0,1]\to\r^d$ connecting $q$ and $\tilde{q}$ and 
$$g_{q}(m,m'):=\int\iprod{ v(x)}{ v'(x)}q(x)\rm{d}x,$$
where $v$ and $v'$ solve the respective continuity equation, i.e.,
$m=-\nabla_x\cdot(q v),\quad m' =-\nabla_x\cdot(q v').$
Then, replacing $(\theta,\theta',f,\nabla_\theta,\iprod{\cdot}{\cdot},\sf{d}_E)$ with $(q, q',F,\nabla_{q},g_q(\cdot,\cdot),$ $\mathsf{W}_2)$ in~\eqref{eq:euclideanmetric} and solving for $\nabla_{q}F$, we obtain
$$\nabla_{q} F \left( q \right) = - \nabla_x \cdot \left( q \nabla_x \delta_q F \left [ q \right ] \right),$$
where $\delta_q F [q]:\r^d\to\r$ denotes the functional's \emph{first variation}: the function satisfying
$$
     {F} \left( q+\varepsilon \chi \right) = {F} \left( q \right) + \varepsilon\int \delta_q F \left[ q \right] \left( x \right) \, \chi(x) \, \mathrm{d} x + o(\varepsilon)
$$
for all (sufficiently-regular) $\chi:\r^d\to\r$ s.t.\ $\int \chi(x)\mathrm{d}x=0$. This is, of course, not the only choice; other metrics induce other `geometries' and lead to other algorithms (e.g., see \citet{duncan2019geometry, liu2017stein, sharrock2023coinem} for the Stein case). For a more rigorous presentation of these concepts, see \Cref{appen:gf_w2} and references therein.

We can now define $F$'s Wasserstein gradient flow just as in the Euclidean case (except we have a PDE rather than an ODE): 
$
    \dot{q}_t = - \nabla_{q} F \left( q_t \right).
$
A well-known example of such a flow is that for which $F(\cdot) := \mathsf{KL} \left( \cdot, p \right)$ for a fixed $p$ in $ \mathcal{P} \left( \mathbb{R}^d \right)$. In this case,  $\delta_q F \left[ q \right] = \log (q/p)$ and the flow reads
$
    \dot{q}_t =   \nabla_x \cdot \left( q_t \nabla_x \log (q_t/p) \right).
$
As noted in~\citet{jordan1998variational}, this is  the Fokker-Planck equation satisfied by the law of the overdamped Langevin SDE with invariant measure $p$: 
$
    \mathrm{d} X_t = \nabla_x \ell \left( X_t \right) \, \mathrm{d} t + \sqrt{2} \, \mathrm{d} W_t,
$ 
where $\ell(x):=\log p(x)$ and $(W_t)_{t\geq0}$ denotes the standard Wiener process. For this reason, papers such as~\citet{wibisono2018,ma2019there} view the overdamped Langevin algorithm obtained by discretizing this SDE using the Euler--Maruyama scheme as an analogue of GD for $F$.

\citet{ma2019there} go a step further and search for `accelerated' methods for minimizing $F$. To do so, they consider distributions over the momentum-enriched space $\r^{2d}$ rather than $\r^d$ and they replace $F$ on $\cal{P}(\r^d)$ with the following `Hamiltonian' on  $\cal{P}(\r^d\times\r^d$):
\begin{equation}\label{eq:w2hamiltonian}
    F_\eta (\cdot):=\mathsf{KL} \left( \cdot, p(\rm{d}x) \otimes r_\eta (\rm{d}u) \right)\,,\end{equation}
where $\eta>0$ denotes a fixed hyperparameter, $r_\eta:=\mathcal{N} \left(0, \eta^{-1} I_d \right)$ a zero-mean $\eta^{-1}$-variance isotropic Gaussian distribution over the momentum variables $u$, and $p\otimes r_\eta$ the product measure. They then define the following $\cal{P}(\r^{2d})$-valued analogue of the $\gamma$-damped Hamiltonian flow~\eqref{eq:euhamflow}:
\begin{align}\label{eq:w2hamflow}
    \dot{q}_t=-\sf{D}_\gamma\nabla_{q} F_\eta(q_t)=\nabla_{(x,u)} \cdot \left( q_t  \mathsf{D}_\gamma \nabla_{(x,u)} \delta_q F_\eta \left[ q_t \right] \right).
\end{align}
This is the Fokker--Planck equation satisfied by the law of the underdamped Langevin SDE:
\begin{equation}\label{eq:underdampedla}
\rm{d}X_t=\eta U_t\rm{d}t,\enskip \rm{d}U_t=[\nabla_x\ell(X_t)-\gamma \eta U_t]\rm{d}t+\sqrt{2\gamma}\rm{d}W_t.\end{equation}
Discretizing the SDE, \citet{ma2019there} recover the underdamped Langevin algorithm~(e.g.,~see~\citet{cheng2018underdamped}) and show that, under suitable conditions, it achieves faster convergence rates than the overdamped Langevin algorithm.
\subsection{Gradient Flows on $\mathbb{R}^{d_\theta} \times \mathcal{P} \left( \mathbb{R}^{d_x} \right)$}\label{sec:pgd}
To minimize a functional $\cal{E}$ on $\mathbb{R}^{d_\theta} \times \mathcal{P} \left( \mathbb{R}^{d_x} \right)$, we can follow steps analogous to those in the previous two sections. First, given a metric $\sf{d}$ on this space, we can define a gradient $\nabla \cal{E}$ w.r.t.\ $\sf{d}$ similarly as in the beginning of \Cref{sec:w2gradflow}; cf.\ \citet[Appendix A]{Kuntz2022}. Setting $\sf{d}((\theta,q),(\theta',q')):=\sqrt{\sf{d}_E(\theta,\theta')^2+\sf{W}_{2}(q,q')^2}$, one finds that $\nabla\cal{E}=(\nabla_\theta\cal{E},\nabla_q\cal{E})$, and the corresponding flow reads
\begin{align*}
    \dot{\theta}_t &= -\nabla_\theta \mathcal{E} \left( \theta_t, q_t \right), \quad
    \dot{q}_t = -\nabla_{q} \mathcal{E} \left( \theta_t, q_t \right).
\end{align*}
\citet{Kuntz2022} were interested in MLE of latent variable problems and set $\cal{E}$ to be the free energy in~\eqref{eq:freeenergy}. In this case, the above reduces to
$$\dot{\theta}_t = \int\nabla_{\theta} \ell \left( \theta_t, x \right) q_t \left(\mathrm{d} x\right),\enskip \dot{q}_t = \nabla_{x} \cdot \left( q_t\nabla_{x} \log\frac{q_t}{\rho_{\theta_t}} \right),$$
%
where $\rho_\theta(\cdot) := p_\theta (y, \cdot)$. The above is the Fokker--Planck equation of the following McKean--Vlasov SDE:
\begin{align*}
    d\theta_t &=\int  \nabla_{\theta} \ell \left( \theta_t, x \right)q_t\left( \mathrm{d} x  \right) , \\
    \mathrm{d} X_t &= \nabla_x \ell \left( \theta_t, X_t \right) \, \mathrm{d} t + \sqrt{2} \, \mathrm{d} W_t,
\end{align*}
where $q_t := \mathrm{Law} \left( X_t \right)$ and $\ell(\theta,x):=\log \rho_\theta(x)$. Because $\theta_t$'s drift depends on the $X_t$'s law, this is a `McKean--Vlasov' process \citep{kac1956foundations, mckean1966class}. Approximating the SDE as described in \Cref{sec:pgd_discetization}, we obtain PGD.
\section{Momentum-Enriched Flow for MLE}
\label{sec:proposal}

Following an approach analogous to those previously taken to obtain accelerated algorithms for minimizing functionals on $\r^{d_\theta}$ and $\cal{P}(\r^{d_x})$ (cf.~\Cref{sec:euclidean,sec:w2gradflow}, respectively), we now derive a momentum-enriched flow that minimizes the free energy functional $\cal{E}$ over $\r^{d_\theta}\times\cal{P}(\r^{d_x})$ defined in~\eqref{eq:freeenergy}. Approximating this flow, we will obtain an algorithm that experimentally outperforms PGD (\Cref{sec:pgd}).

We begin by enriching both components of the space with momentum variables: we replace $\theta$ and $x$ in $\r^{d_\theta}$ and $\r^{d_x}$ with $(\theta,m)$ and $(x,u)$ in $\r^{2d_\theta}$ and $\r^{2d_x}$; and we correspondingly substitute $\cal{P}(\r^{d_x})$ with $\cal{P}(\r^{2d_x})$. Next, we define a `Hamiltonian' on the momentum-enriched space $\r^{2d_\theta}\times\cal{P}(\r^{2d_x})$ analogous to  $f_\eta$ and $F_\eta$ in~(\ref{eq:hamilanddamp},\ref{eq:w2hamiltonian}):
\begin{align}
    \label{eq:hamilboth}
    \mathcal{F} \left( \theta, m, q \right) :=& \int  \log \left( \frac{q \left(x, u \right)}{\rho_{\theta, \eta_x} \left( x, u \right)} \right) \, q \left( x, u \right)\mathrm{d} x \, \mathrm{d} u \nonumber\\
    &+ \frac{\eta_\theta}{2} \| m \|^2.
\end{align}
where $\left( \eta_\theta, \eta_x \right)$ denote positive hyperparameters and $\rho_{\theta, \eta_x} \left( x, u \right) := p_\theta \left( y, x \right) r_{\eta_x} \left( u \right)$ with  $r_{\eta_x} := \mathcal{N} ( 0, $ $\eta_x^{-1} I_{d_x})$.  Consider now the following $\left( \gamma_\theta, \gamma_x \right)$-damped Hamiltonian flow of $\mathcal{F}$ analogous to~(\ref{eq:euhamflow},\ref{eq:w2hamflow}): 
\begin{subequations}\label{eq:mpd_pde}\begin{align}
    &(\dot{\theta}_t,\dot{m}_t) = - \mathsf{D}_{\gamma_\theta}  \nabla_{(\theta,m)} \mathcal{F} \left( \theta_t,m_t,  q_t \right), \\
    &\dot{q}_t = \nabla_{(x,u)} \cdot \left( q_t \mathsf{D}_{\gamma_x} \nabla_{(x,u)} \mathcal\delta_q \cal{F} \left[ \theta_t,m_t,  q_t \right] \right).\end{align}
\end{subequations}
This is the Fokker--Planck equation for the following McKean--Vlasov SDE:
\begin{subequations}
\label{eq:mpd_flow}
    \begin{align}
    \mathrm{d} \theta_{t} &=\eta_{\theta} m_{t} \, \mathrm{d} t, \\
    \mathrm{d} m_{t} &= \left[\int  \nabla_{\theta} \ell \left( \theta_t, x \right)q_{t,X} \left( \mathrm{d} x \right)- \gamma_{\theta} \eta_{\theta} m_{t}\right] \, \mathrm{d} t,\\
    \mathrm{d} X_{t} &= \eta_{x}  U_{t} \, \mathrm{d} t, \\
    \mathrm{d} U_{t} &= [\nabla_{x} \ell \left( \theta_{t}, X_t \right)-\gamma_{x} \eta_{x} U_{t}] \, \mathrm{d} t + \sqrt{2  \gamma_{x}} \, \mathrm{d}W_{t},
\end{align}
\end{subequations}%
where $q_{t,X} := \mathrm{Law} \left( X_{t} \right)$. Assuming that $\ell$ has a Lipschitz gradient, the SDE has a unique strong solution:
\begin{assumption}[$K$-Lipschitz gradient]
	\label{ass:gradlip} We assume that the potential $\ell$ has a $K$-Lipschitz gradient, i.e., there is some $K>0$ such that the inequality holds
\begin{align*}
    \left \| \nabla_{(\theta, x)} \ell(\theta, x) - \nabla_{(\theta, x)} \ell(\theta', x') \right \| \le K  \|(\theta,x) - (\theta',x') \|,
\end{align*}
for all $(\theta, x)$ and $(\theta', x')$ in $\mathbb{R}^{d_\theta} \times \mathbb{R}^{d_x}$.
\end{assumption} 
\begin{proposition}[Existence and uniqueness of strong solutions to \eqref{eq:mpd_flow}]
    Under \Cref{ass:gradlip}, there exists a unique strong solution to \eqref{eq:mpd_flow} for any initial condition $(\theta_0, m_0, q_0)$ in $\bb{R}^{2d_\theta}\times \cal{P}(\bb{R}^{2d_x})$.
    \label{prop:existence_uniqueness}
\end{proposition}
\begin{proof}The proof is a generalization of \citet[Theorem 1.7]{carmona2016lectures}'s proof; see \Cref{{proof:existence_uniqueness}}.\end{proof}

\section{Convergence of Momentum-Enriched Flow}
\label{sec:mt_convergence}

To evidence that our approach is theoretically well-founded, we wish to show the momentum-enriched flow~\eqref{eq:mpd_pde} solves the maximum likelihood problem. We do this by showing
that the log marginal likelihood evaluated along $(\theta_t)_{t\geq0}$ (of the momentum-enriched flow~\eqref{eq:mpd_pde}) converges exponentially fast to its supremum. Since directly working with the marginal likelihood is difficult, we exploit the following inequality: for all $m \in \mathbb{R}^{d_\theta}$, $q\in \cal{P}(\bb{R}^{2d_x})$,
\begin{equation}
    \label{eq:relationship_evidence_free_energy}
    - \log p_\theta(y) \le \cal{E}(\theta, q_X) \le \cal{F}(\theta, m, q), \enskip 
\end{equation}
where $q_X$ denotes $q$'s $x$-marginal. 
From \eqref{eq:relationship_evidence_free_energy} and the fact that both $\cal{E}$ and $\cal{F}$ share the same infimum $\cal{E}^*$, it holds that if
$\cal{F}$ in~\eqref{eq:hamilboth} decays exponentially fast to its infimum along the solutions $(\theta_t,m_t,q_t)_{t\geq0}$ of~\eqref{eq:mpd_pde}, then we have the desired result of convergence in log marginal likelihood (as summarized in \eqref{eq:relationship_evidence_free_energy}). Hence, 
we dedicate the remainder of this section to establishing that $\cal{F}$ convergence exponentially along the momentum-enriched flow.

Because the momentum-enriched flow~\eqref{eq:mpd_pde} is not a gradient flow, we are unable to prove the convergence using the techniques commonly applied to study gradient flows; e.g., see~\citet{otto2000}.  Instead, we need to resort to an involved argument featuring a Lyapunov function $\cal{L}$ along the lines of those in \citet{wilson2016lyapunov,wibisono2016variational,ma2019there}. 

Two immediate choices for Lyapunov functions are available: $\cal{E}$ and $\cal{F}$. However, unlike the gradient case, $\cal{E}$ along the momentum-enriched flow \eqref{eq:mpd_flow} is not monotonic and can fluctuate. As for $\cal{F}$, we show in \Cref{prop:contraction} that $\cal{F}$ along the flow  denoted by $\cal{F}_t:=\cal{F}(\theta_t,m_t,q_t)$ satisfies
\begin{align}\label{eq:hamiltoniantimeder}
    \dot{\mathcal{F}}_t = -\gamma_{\theta} \left\Vert \nabla_{m} \mathcal{F}_t \right\Vert ^{2} - \gamma_{x} \left\Vert \nabla_{u} \delta_{q} \mathcal{F}_t \right\Vert ^{2} \leq 0;
\end{align}
implying that $\cal{F}_t$ is non-increasing. This shows that the flow is in some sense stable, but it is not enough to prove the convergence: the two gradients could vanish without the flow necessarily having reached a minimizer. To preclude this happening, we need to force $(\theta,m)$-gradients to feature in~\eqref{eq:hamiltoniantimeder} similarly to the $(x,u)$-gradients. To achieve this, we follow~\citet{wilson2016lyapunov, ma2019there} and introduce the following `twisted' gradient terms:
\begin{align*}
    \left\Vert \nabla_{\left( \theta, m \right)} \mathcal{F} \right\Vert _{T_{\left( \theta, m \right)}}^{2} &:= \left\langle \nabla_{\left( \theta, m \right)} \mathcal{F}, T_{\left( \theta, m \right)} \nabla_{\left( \theta, m \right)} \mathcal{F} \right\rangle, \\
    \left\Vert \nabla_{\left( x, u \right)} \delta_{q} \mathcal{F} \right\Vert _{T_{\left( x, u \right)}}^{2} &:= \left\langle \nabla_{\left( x, u \right)} \delta_{q} \mathcal{F}, T_{\left( x, u \right)} \nabla_{\left( x, u \right)} \delta_{q} \mathcal{F} \right\rangle,
\end{align*}
where 
\begin{align*}
    T_{\left( \theta, m \right)} &= K^{-1} \left(
    \begin{array}{cc}
        \tau_{\theta} I_{d_{\theta}} & \frac{\tau_{\theta m}}{2} I_{d_{\theta}} \\
        \frac{\tau_{\theta m}}{2} I_{d_{\theta}} & \tau_{m} I_{d_{\theta}}
    \end{array}\right), \\
    T_{\left( x, u \right)} &= K^{-1} \left(
    \begin{array}{cc}
        \tau_{x} I_{d_{x}} & \frac{\tau_{xu}}{2}  I_{d_{x}}\\
        \frac{\tau_{xu}}{2} I_{d_{x}} & \tau_{u} I_{d_{x}}
    \end{array}\right),
\end{align*}
with $K$ denoting the Lipschitz constant in \Cref{ass:gradlip} and $\left( \tau_{\theta}, \tau_{\theta m}, \tau_{m}, \tau_{x}, \tau_{xu}, \tau_{x}\right)$  non-negative constants such that the above matrices are positive-semidefinite. We then add the above terms to $\cal{F}-\cal{E}^*$ to obtain our Lyapunov function:
\begin{align}
    \label{eq:lyapunov}
        \mathcal{L}:= \mathcal{F}-\cal{E}^*+ \left\Vert \nabla_{\left( \theta, m \right)} \mathcal{F} \right\Vert _{T_{\left( \theta, m \right)}}^{2}+ \left\Vert \nabla_{\left( x, u \right)} \delta_{q} \mathcal{F} \right\Vert _{T_{\left( x, u \right)}}^{2}.
\end{align}
We prove the convergence for models satisfying the a ``log Sobolev inequality'' (see \citet{caprio2024error} for a detailed discussion). It extends the log Sobolev inequality popular in optimal transport and the Polyak--Łojasiewicz inequality often used to argue linear convergence of GD algorithms.
\begin{assumption}[Log Sobolev inequality]\label{ass:logsobolev} There exists a constant $C_{\cal{E}}>0$, s.t.~for all $(\theta,q)$ in $\r^{d_\theta}\times\in\cal{P}(\r^{d_x})$,
\begin{equation*}
    \cal{E}(\theta, q)-\cal{E}^* \le \frac{1}{2C_{\cal{E}}} \norm{\nabla \cal{E}(\theta, q)}^2_q,
\end{equation*}
where $
\norm{\nabla \cal{E}(\theta, q)}_q := \norm{\nabla_\theta \cal{E}(\theta, q)}^2 + \norm{\nabla_x \log (q/\rho_\theta)}^2_q$ with $\norm{\cdot}_q$ denoting the $L^2_q$ norm.
\end{assumption}
\begin{theorem}[Exponential convergence of the momentum-enriched flow]
\label{prop:convergence_flow} Suppose that Assumptions~\ref{ass:gradlip},~\ref{ass:logsobolev} hold and there exists  $\left(\varphi, \tau_\theta, \tau_{\theta \tmo}, \tau_\tmo, \tau_x, \tau_{x\qmo},\tau_\qmo \right)$ satisfying the inequalities \eqref{eq:flow_conditions} in the supplement. If $\cal{L}_t:=\cal{L}(\theta_t,m_t,q_t)$,
\begin{align*}
    \dot{\mathcal{L}}_t \le - \varphi C \mathcal{L}_t,
\end{align*}
where $C = \min \{C_\cal{E}, \eta_\theta, \eta_x\}$. Moreover, 
\begin{align}
\label{eq:convergence_logp}
\begin{split}
\sup_{\theta\in\r^{d_\theta}} \log p_\theta(y) - \log p_{\theta_t}(y)
&\\\leq \cal{E}_t-\cal{E}^*
    \leq  \mathcal{F}_t -\cal{E}^*\le \cal{L}_t\leq \mathcal{L}_0 \exp \left ( - \varphi C t \right )&.
    \end{split}
\end{align}
\end{theorem}
\begin{proof}
    The proof extends \citet[Proposition~1]{ma2019there}'s proof. It requires generalizing from $\cal{P}(\bb{R}^{d_x})$ to the product space and controlling the terms that arise from the interaction between the two spaces; cf.~\Cref{sec:convergence_flow}.
\end{proof}
It is not difficult to find hyperparameter choices for which 
 the inequalities in \eqref{eq:flow_conditions} have solutions; cf.~ \Cref{sec:eq_flow_conditions true} for some concrete examples.
\section{Momentum Particle Descent}
\label{sec:mt_discretization}
In order to realize an actionable algorithm, we need to discretize the dynamics of the momentum-enriched flow \eqref{eq:mpd_flow} both in space and in time. In this section, we derive the algorithm called Momentum Particle Descent (MPD) via discretization.

For the space discretization, following the approach of \citet{Kuntz2022},  we approximate 
$
    q_{t,X} \approx \frac{1}{M} \sum_{i \in \left[ M \right]} \delta_{X^{i,M}_t} =: q^M_{t,X}
$ using a collection of particles.
This yields the following system: for  $i \in [M]$,
\begin{subequations}
    \label{eq:space_discretization}
\begin{align}
    \mathrm{d} \theta^M_{t} =&\eta_{\theta} m^M_{t} \, \mathrm{d} t \\
    \mathrm{d} m^M_{t} =& - \gamma_{\theta} \eta_{\theta} m^M_{t} \, \mathrm{d} t -\nabla_\theta \cal{E} \left(\theta^M_t, q^M_{t,X} \right) \, \mathrm{d} t \\
    \mathrm{d} X_t^{i,M} =& \eta_{x} U_t^{i,M} \, \mathrm{d} t \\
    \begin{split}
    \mathrm{d} U_t^{i,M} =& -\gamma_{x} \eta_{x} U_t^{i,M} \, \mathrm{d} t + \nabla_{x} \ell \left( \theta_{t}^{M}, X_t^{i,M} \right) \, \mathrm{d} t \\ &+ \sqrt{2  \gamma_{x}} \, \mathrm{d} W_t^{i}.
    \end{split}
\end{align}
\end{subequations}
where $\{W_t^{i}\}_{i=1}^M$ comprises $M$ independent standard Wiener processes. \Cref{prop:chaos} establishes asymptotic pointwise propagation of chaos which provides asymptotic justification for using a particle approximation for $q_{t,X}$ to approximate the flow in \eqref{eq:mpd_flow}. The proof can be found in the \Cref{proof:chaos}.
\begin{proposition} Under \Cref{ass:gradlip}, we have
    $$
    \lim_{M\rightarrow \infty} \mathbb{E} \sup_{t\in [0,T]} \{\|\vartheta_t - \vartheta_t^M\|^2 + \sf{W}^2_2(q_t^{\otimes M},q^M_t)\} = 0,
    $$
    where $\vartheta_t := (\theta_t, \tmo_t)$, $q_t^{\otimes M} := \Pi_{i=1}^Mq_t$ with $q_t = \rm{Law}(X_t, U_t)$ are defined by the SDE in \eqref{eq:mpd_flow}; and  $\vartheta^M_t:=(\theta^M_t, \tmo_t^M)$ ,  and $q^M_t = \rm{Law}(\{(X_t^{i,M}, U_t^{i,M})\}_{i=1}^M)$ in \eqref{eq:space_discretization}.
    \label{prop:chaos}
\end{proposition}
As for the time discretization, as noted by \citet{ma2019there}, naive application of an Euler--Maruyama scheme to momentum-enriched dynamics may be insufficient for attaining an accelerated rate of convergence. We appeal to the literature on the discretization of underdamped Langevin dynamics to design an appropriate integrator. One integrator that can achieve accelerated rates is the (Euler) Exponential Integrator \citep{cheng2018underdamped, sanz2021wasserstein, hochbruck2010exponential}. The main strategy is to `freeze' the nonlinear components of the dynamics in a suitable way, and then solve the resulting linear SDE. We also incorporate a partial update inspired by \citet{sutskever2013importance}'s interpretation of NAG. More precisely, given $\left( \theta_0, m_0, \{ \left( X^i_0, U_0^i \right) \}_{i=1}^M \right)$, define an approximating SDE on the time interval $t \in \left[ 0, h \right]$ as follows: for $i \in [M]$,
\begin{subequations}
    \begin{align}
    \mathrm{d} \tilde{\theta}_t &=\eta_\theta \tilde{m}_t \, \mathrm{d} t \label{eq:theta_approx_sde}\\
    \mathrm{d} \tilde{m}_t &= - \gamma_\theta \eta_\theta \tilde{m}_t \, \mathrm{d} t - \nabla_\theta \cal{E} \left( \bar{\theta}_0, \tilde{q}_{0,X}^M\right) \, \mathrm{d} t \label{eq:tmo_approx_sde}\\
    \mathrm{d} \tilde{X}_t^i &= \eta_x \tilde{U}_t^i \, \mathrm{d} t \label{eq:x_approx_sde}\\
     \quad \mathrm{d} \tilde{U}_t^i &= -\gamma_x \eta_x \tilde{U}_t^i \, \mathrm{d} t + \nabla_{x} \ell \left( \bar{\bar{\theta}}_0, X_0^i \right) \, \mathrm{d} t + \sqrt{2 \gamma_{x}} \, \mathrm{d} W_t^i \label{eq:qmo_approx_sde},
\end{align}
\label{eq:approximate_sde}
\end{subequations}
where $\tilde{q}_{t,X}^M := \frac{1}{M}\sum_{i=1}^M \delta_{\tilde{X}_t^i}$, 
and the fixed variables $\bar{\theta}$ and $\bar{\bar{\theta}}_0$ (for reasons which will become clear when faced with the solved SDE) can be thought of as a partial update to $\tilde{\theta}_t$. This is a \textit{linear} SDE, which can then be solved analytically, yielding the following iteration (see \Cref{sec:our_discretization} for details): for $i \in [M]$, we have
\begin{align*}
    \begin{split}
    \tilde{\theta}_k =& \tilde{\theta}_{k - 1} + \frac{\iota_\theta \left( h \right)}{\gamma_\theta}  \tilde{m}_{k - 1} \\
    &- \frac{1}{\gamma_\theta}  \left( h - \frac{\iota_\theta \left( h \right)}{\gamma_\theta \eta_\theta} \right) \nabla_\theta \cal{E} \left(\bar{\theta}_{k - 1} , \tilde{q}_{k-1,X}^M\right)
    \end{split} \\
    \begin{split}
    \tilde{m}_k =& (1-\iota_\theta \left( h \right)) \tilde{m}_{k - 1} - \frac{\iota_\theta \left( h \right)}{\gamma_\theta \eta_\theta} \nabla_\theta \cal{E} \left(\bar{\theta}_{k - 1}, \tilde{q}_{k-1,X}^M\right) 
    \end{split}\\
    \begin{split}
    \tilde{X}_k^i =& \tilde{X}_{k - 1}^i + \frac{\iota_x \left( h \right)}{\gamma_x} \tilde{U}_{k - 1}^i \\
    &+ \frac{1}{\gamma_x} \left( h - \frac{\iota_x \left( h \right)}{\gamma_x \eta_x} \right) \nabla_x \ell \left( \bar{\bar{\theta}}_{k-1}, \tilde{X}_{k - 1}^i \right) + L_\Sigma^{XX} \xi_k^i 
    \end{split}\\
    \begin{split}
     \tilde{U}_k^i =& (1-\iota_x \left( h \right) )\tilde{U}_{k - 1}^i + \frac{\iota_x \left( h \right)}{\gamma_x \eta_x} \nabla_x \ell \left( \bar{\bar{\theta}}_{k-1}, \tilde{X}_{k - 1}^i \right) \\
     &+ L_\Sigma^{XU} \xi_k^i + L_\Sigma^{UU} \xi_k^{\prime, i},
    \end{split}
   \end{align*}
where $\iota_x \left( h \right) := 1-\exp \left( - \eta_x \gamma_x h \right)$, $\left\{ \xi_k^i, \xi_k^{\prime, i} \right\}_{i \in \left[ M \right]} \overset{\text{i.i.d.}}{\sim} \mathcal{N} \left( 0_{d_x}, I_{d_x} \right)$, and $\left( L_\Sigma^{XX}, L_\Sigma^{XU}, L_\Sigma^{UU} \right)$ are suitable scalars that depend on $(\eta_x, \gamma_x)$ that can be found in the appendix as \Cref{eq:sampling_constants}. We set 
\begin{equation}
    \label{eq:partial_updates}
    \bar{\theta}_{k-1} := \tilde{\theta}_{k-1} + \frac{\iota_\theta(h)}{\gamma_\theta} \tilde{m}_{k-1}, \enskip \bar{\bar{\theta}}_{k-1} =\tilde{\theta}_k.
\end{equation}
\begin{figure*}
	\begin{subfigure}[b]{0.32\textwidth}
		\centering
		\includegraphics[width=\textwidth]{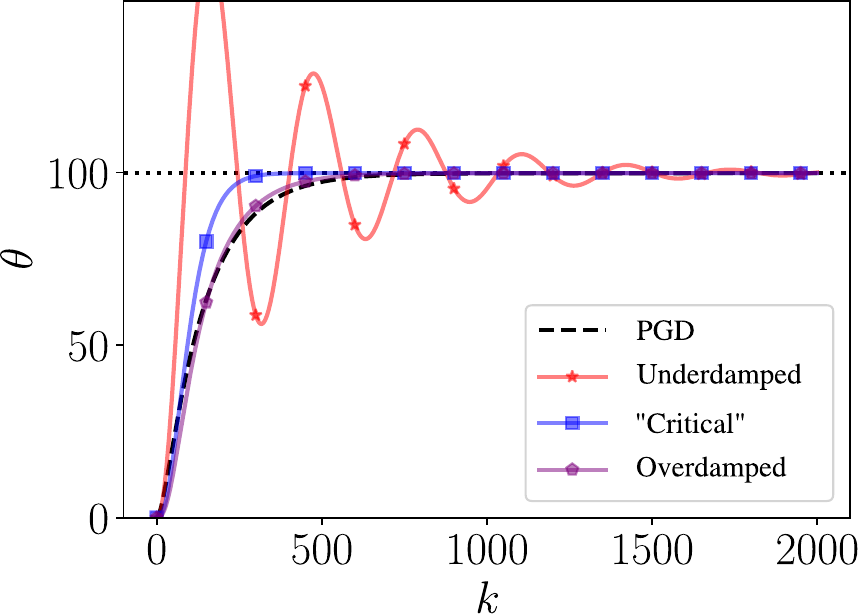}
		\caption{Different regimes in MPD.}
		\label{fig:different_regimes}
	\end{subfigure}
		\begin{subfigure}[b]{0.32\textwidth}
		\centering
		\includegraphics[width=\textwidth]{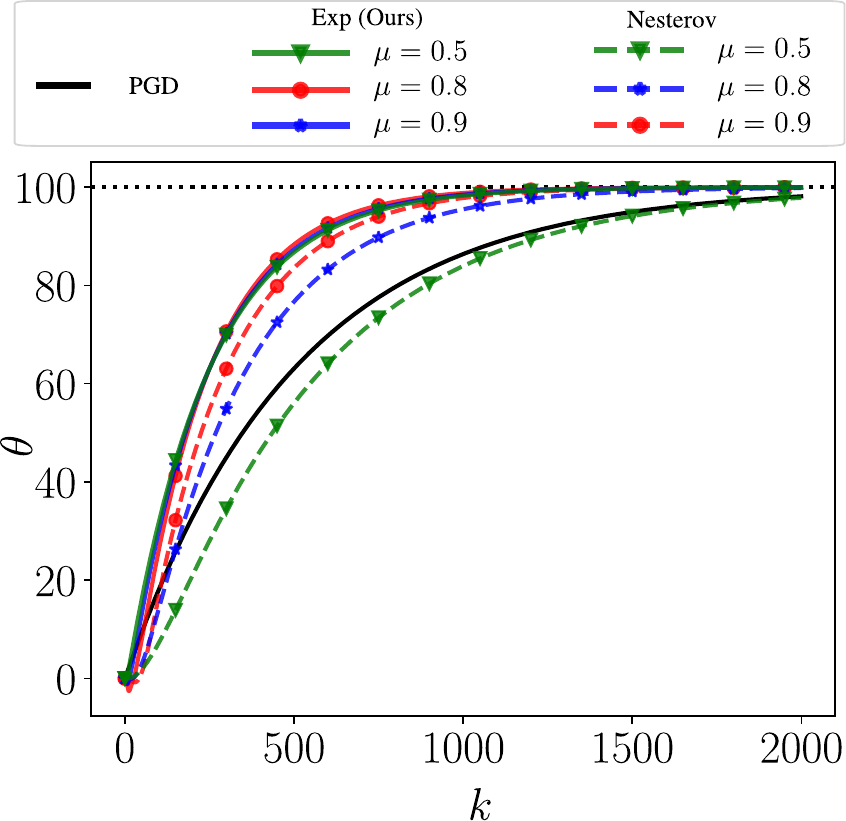}
		\caption{$\theta_t$ integrator comparison.}
		\label{fig:theta_integrators}
	\end{subfigure}
		\begin{subfigure}[b]{0.32\textwidth}
		\centering
		\includegraphics[width=\textwidth]{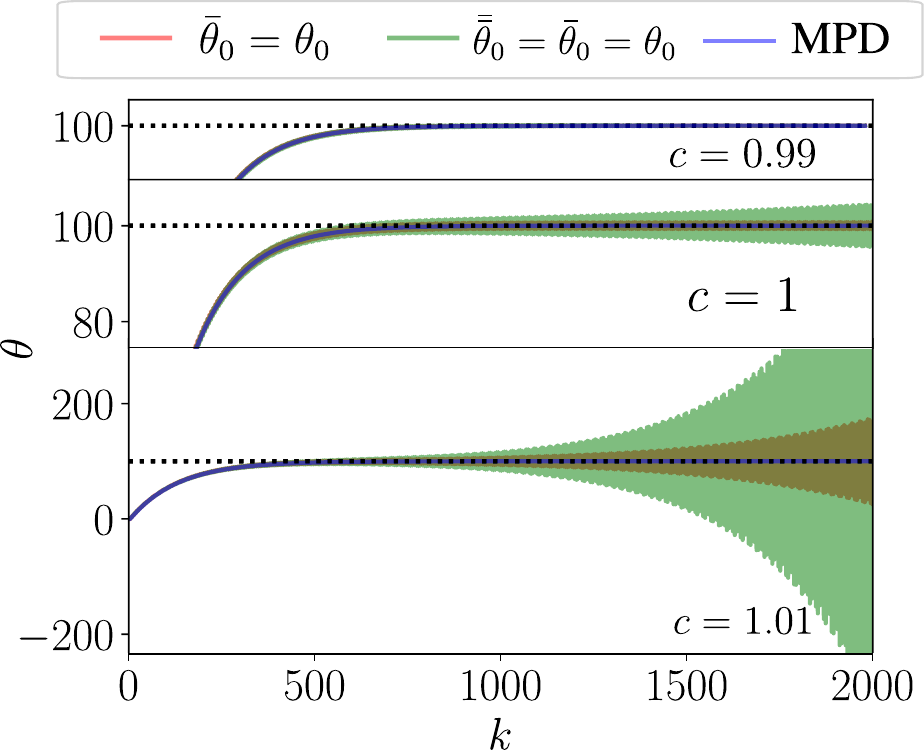}
		\caption{Role of gradient correction.}
		\label{fig:gc}
	\end{subfigure}
 \vspace{-1mm}
	\caption{\textbf{Toy Hierarchical Model}. (a) Different regimes that arise from varying the momentum parameters; (b) comparison between our Exponential (Exp) integrator and a Nesterov-like integrator for different momentum parameters; (c) we compare the performance of the MPD with and without gradient correction.}
\label{fig:toy_hmm}
\end{figure*}
\begin{figure*}[ht]
    \includegraphics[width=\linewidth]{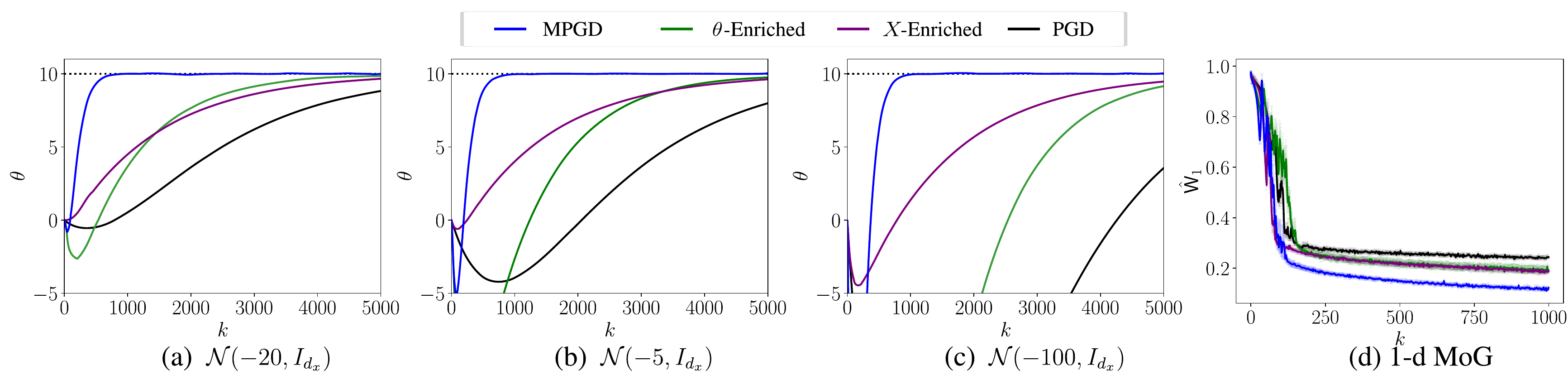}
    \vspace{-3mm}
    \caption{\textbf{Comparison of MPD with algorithms that only enrich one component.} (a), (b), (c) shows the performance on the $\rm{ToyHM}(10,12)$ with different initialization of the particle cloud $\{X_0^i\}_{i=1}^M$. In (d) shows the $\hat{\sf{W}}_1$ vs iterations of a density estimation problem on a Mixture of Gaussian (MoG) dataset. MPD is shown in \textcolor{blue}{blue}, $\theta$-only-enriched shown in \textcolor{green}{green}, $X$-only-enriched in \textcolor{mplPurple}{purple}, and PGD in black. The results are averaged over $10$ independent trials.}
    \vspace{-3mm}
    \label{fig:why_accelerate_both}
\end{figure*}
The astute reader will notice that our approximating ODE-SDE (i.e., \Cref{eq:approximate_sde}) deviates from \citet{cheng2018underdamped}'s style of discretization. Specifically, the difference lies in where we compute the gradient in \Cref{eq:tmo_approx_sde,eq:qmo_approx_sde} which we refer to as \textit{gradient correction}. The choices of $\bar{\theta}_{k-1}$ and $\bar{\bar{\theta}}_{k-1}$ in \Cref{eq:partial_updates} are chosen to reflect partial and full updates to $\tilde{\theta}_k$ respectively. This is reminiscent of NAG's usage of a partial parameter update to compute the gradient, unlike Polyak's Heavy Ball, which does not \citep[Section 2.1]{sutskever2013importance}. In our case, we empirically found that this choice is critical for a more stable discretization and enables the algorithm to be ``faster'' by taking larger step sizes  (see \Cref{sec:toy_hmm}).

The choice of the momentum parameters $\left(\eta_{\theta}, \gamma_{\theta}, \eta_{x}, \gamma_{x}\right)$ is crucial to the practical performance of MPD. We observe that (as is also the case for NAG, the underdamped Langevin SDE, and other similar systems) there are three different qualitative regimes for the dynamics: i) the underdamped regime, in which the parameter values oscillate, ii) the overdamped regime, in which one recovers PGD-type behaviour, and iii) the critically-damped regime, in which oscillations are largely suppressed, but the momentum effects are still able to accelerate the convergence behaviour relative to PGD. While rigorous approaches are limited to simple targets (e.g., \citet{dockhorn2022score}), one can utilize cross-validation to obtain these parameters in practice.%

\section{Experiments}
\label{sec:experiments}

In this section, we study various design choices of MPD and demonstrate the efficacy of our proposal through empirical results. The structure of this section is as follows:
in \Cref{sec:toy_hmm}, we demonstrate on the Toy Hierarchical model the effects of various design choices; then, in \Cref{sec:momentum_or_no}, we compare MPD with algorithms that incorporate momentum in either $\theta$ or $X$ only; finally, in \Cref{sec:vae}, we compare our proposal on training VAEs with other methods. The code for reproducibility is available online: \url{https://github.com/jenninglim/mpgd}.

\subsection{Toy Hierarchical Model}
\label{sec:toy_hmm}
As a toy example, we consider the hierarchical model proposed in \citet{Kuntz2022}. The model $\rm{ToyHM}(\theta,\sigma)$ is given by $p_\theta(y, x) := \prod_{i=1}^N \mathcal{N}(y_i|x_i, 1)\cal{N}(x_i| \theta, \sigma^2)$, where $N=100$ is the number of data points. The dataset $y$ is sampled from a model with $\theta = 100, \sigma=1$. 
In this experiment, we wish to understand the behaviour of MPD compared against PGD. We are particularly interested in (1) how the momentum parameters $\left(\eta_{\theta}, \gamma_{\theta}, \eta_{x}, \gamma_{x}\right)$ affect the optimization process;  (2) comparing our proposed discretization to another that follows in the style of NAG (detailed in \Cref{sec:cheng_nag_mpd}); and (3) the role of the gradient correction term described in \Cref{sec:mt_discretization}. The results are shown in \Cref{fig:toy_hmm}. Further experiment details can be found in \Cref{appendix:toyhmm}.

 In \Cref{fig:different_regimes}, we show that different regimes can arise from different choices of hyper-parameters (as discussed in \Cref{sec:mt_discretization}). In \Cref{fig:theta_integrators}, we compare the performance of MPD using (our) Exponential integrator with a NAG-like discretization for $(\theta_t, \tmo_t)$-component. We vary the ``momentum coefficient'', i.e., we have $\mu \in \{0.5, 0.8, 0.9\}$ with $\mu_\theta = \mu_x= \mu$. It can be seen that MPD with our exponential integrator for $\theta_t$ performs better than NAG-like integrator (see \Cref{appen:heuristic} for more discussion). In \Cref{fig:gc}, we show the effect of the gradient correction term for three different step sizes in $(\theta, \tmo)$-components and $(X,\Qmo)$-components. The different lines in the figure are generated from multiplying the stepsize of each component with a constant $c \in \{0.99, 1, 1.01\}$. It can be seen that our proposed gradient correction results in a more effective discretization. In \textcolor{red}{red}, we show MPD with the absence of gradient correction in \Cref{eq:tmo_approx_sde} (i.e., when we use $\bar{\theta}_0 = \theta_0$ in \Cref{eq:tmo_approx_sde}), and, in \textcolor{green}{green}, we show the MPD when the gradient correction is absent in both \Cref{eq:tmo_approx_sde} and \Cref{eq:qmo_approx_sde} (i.e. when we use $\bar{\bar{\theta}}_0 = \bar{\theta}_0 = \theta_0$ in \Cref{eq:tmo_approx_sde} and \Cref{eq:qmo_approx_sde}). It can be seen the gradient correction term results in a more stable algorithm.

\begin{figure*}[h]
	\centering
	\includegraphics[width=0.8\textwidth]{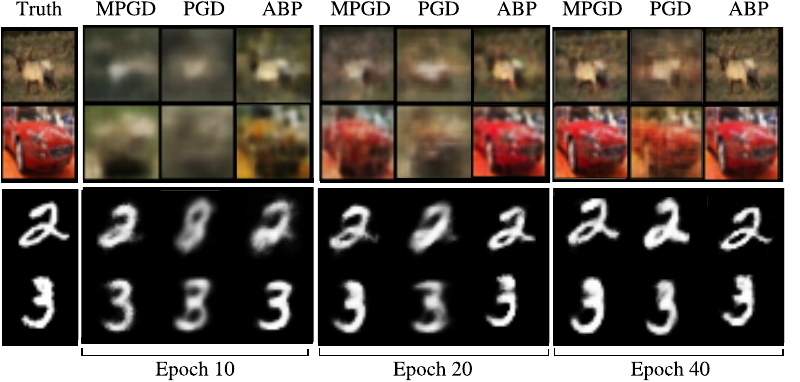}
 \vspace{-3mm}
	\caption{\textbf{Posterior Cloud  vs Epochs}. We show the evolution of the reconstruction of a particle for \textit{persistent} methods. The particle is taken at epoch $\{10, 20, 40\}$ on MNIST and CIFAR-10.}
	\label{fig:particle_vs_epoch}
\end{figure*}
\begin{table*}[ht]
\centering
\begin{tabular}{@{}cccccc@{}}
\toprule
Dataset & MPD           & PGD             & ABP            & SR               & VI             \\ \midrule
MNIST  & $\bm{47.9} \pm 0.6$ & $104.1 \pm 1.2$ & $\bm{45.9} \pm 0.8$  & $109.97 \pm 12.5$ & $72.6 \pm 1.3$ \\
CIFAR  & $\bm{93.2} \pm 1.5$ & $106.2 \pm 4.8$ & $97.0 \pm 8.0$ & $126.25 \pm 13.0$ & $\bm{94.2} \pm 6.7$ \\ \bottomrule
\end{tabular}
 \caption{\textbf{FID scores} on MNIST and CIFAR-10 after the final epoch. In \textbf{bold}, we indicate the lowest two scores on average. We write $(\mu \pm \sigma)$ to indicate the mean $\mu$ and standard deviation $\sigma$ of the FID score calculated over three independent trials.}
 \label{tab:fid_scores}
 \vspace{-3mm}
\end{table*}

\subsection{Why Accelerate Both Components?}
\label{sec:momentum_or_no}
One may question the importance of momentum-enriching both components of PGD. We show in two experiments that enriching only one component results in a suboptimal algorithm. To show this, we consider the $\rm{ToyHM}(10, 12)$, and, as a more realistic example, a $1$-d density estimation problem using a VAE on a Mixture of Gaussians with two equally weighted components which are $\cal{N}(2, 0.5)$ and $\cal{N}(-2, 0.5)$. To measure the performance of each algorithm, for the first problem, we examine the parameter of the model which should converge to the true value of $10$; and, for the second, we compute the empirical Wasserstein-1 $\hat{\sf{W}}_1$ using the empirical CDF.

\Cref{fig:why_accelerate_both} shows the results of both experiments. Overall, it can be seen that the presence of momentum results in a better performance. Furthermore, we observe that the algorithms which only enrich one component can exhibit problems which are not present for MPD. It can be seen that the poor initialization of the cloud drastically lengthens the transient phase of the $\theta$-only enriched algorithm, whereas MPD can recover more rapidly. This is a typical setting in high-dimensional settings (see \Cref{fig:why_accelerate_both} (d)). Conversely, while the $X$-only enriched algorithm does not suffer as much from poor initialization, it can be observed to suffer from slow convergence in the $\theta$-component (intuitively, one pays a larger price for poor conditioning of the ideal MLE objective). Overall, it can be observed that MPD noticeably outperforms PGD and all other intermediate methods in all settings. For details/discussion, see \Cref{appendix:more_experiments}.
\subsection{Image Generation}
\label{sec:vae}
For this task, we consider two datasets: MNIST \citep{lecun1998gradient} and CIFAR-10 \citep{krizhevsky2009learning}. For the model, we use a variational autoencoder \citep{kingma2013auto} with a VampPrior \citep{tomczak2018vae}. As baselines, we compare our proposed method MPD against PGD, Alternating Backpropagation (ABP) \citep{han2017alternating}, Short Run (SR) \citep{nijkamp2020learning}, and amortized variational inference (VI) \citep{kingma2013auto}. ABP is the most similar to PGD. They differ in that ABP takes multiple steps of the (unadjusted and overdamped) Langevin algorithm (ULA) instead of PGD's single step. They (MPD, PGD, ABP) are also \textit{persistent}, meaning that the ULA chain starts at the previous particle location, whereas SR restarts the chain at a random location sampled from the prior but like ABP runs the chain for several steps. Further experiment details can be found in \Cref{appendix:vae}.

 We are interested in the generative performance of the resulting models. For a qualitative measure, we show the samples produced by the model in \Cref{fig:cifar_samples} for CIFAR (for MNIST samples, see \Cref{fig:mnist_samples} in the Appendix). As a quantitative measure, we report the Fr\'{e}chet inception distance (FID) \citep{heusel2017gans} as shown in \Cref{tab:fid_scores}. It can be seen that our proposed method does well compared against other baselines. The closest competitor is ABP. As noted by \citet{Kuntz2022}, ABP can reap the benefits of taking multiple ULA steps to locate the posterior mode quickly and reduce the transient phase. This hypothesis is further confirmed in \Cref{fig:particle_vs_epoch}, where we visualize the evolution of a single particle across various epochs. It can be seen that ABP's and MPD's distinct methods of reducing the transient phase (by taking multiple steps or utilizing momentum) are effective. As a competitor SR does perform badly, which we attribute to its non-persistent property: by restarting the particles at the prior and running short chains, it is unable to overcome the bias of short chains and locate the posterior mode. Variational inference performs well for CIFAR but surprisingly falls slightly short in MNIST.
%
\section{Conclusion, Limitations, and Future work}\label{sec:conc}
We presented Momentum Particle Descent MPD, a method of incorporating momentum into particle gradient descent. We establish several theoretical results such as the convergence in continuous time; existence and uniqueness; as well as some guarantees for the usage of the particle approximation. Through experiments, we showed that, for suitably chosen momentum parameters, the resulting algorithm achieves better performance than PGD.

The main limitation that may impede the widespread adoption of MPD over PGD is the requirement for tuning the momentum parameters. This issue is inherited from the Underdamped Langevin dynamics where it is generally understood to be somewhat more challenging than its overdamped counterpart, as witnessed by the dearth of papers proposing practical tuning strategies for this class of algorithms. Although we found that the momentum coefficient heuristic (see \Cref{appen:heuristic}) worked well in our experiments, future work includes a systematic method of tuning momentum parameters $\left(\eta_{\theta}, \gamma_{\theta}, \eta_{x}, \gamma_{x}\right)$ following \citet{riou2023adaptive} (or some justification for using the momentum coefficient heuristic). Another future work is a theoretical characterization of the difference between our discretization scheme compared with other potential schemes akin to \citet{sanz2021wasserstein}.

\section*{Acknowledgements}
We thank our anonymous reviewers for their helpful comments which improved the quality of the paper. JL acknowledge the support from the Feuer International Scholarship.
JK and AMJ acknowledge support from the Engineering and Physical Sciences Research Council (EPSRC; grant EP/T004134/1).
AMJ acknowledges further support from the EPSRC (grants EP/R034710/1 and EP/Y014650/1).
SP acknowledges support from the Engineering and Physical Sciences Research Council (EPSRC; grant EP/R018561/1).

\section*{Impact Statement}
This paper presents work whose goal is to advance the field of Machine Learning. There are many potential societal consequences of our work, none of which we feel must be specifically highlighted here.

\bibliographystyle{apalike}
\bibliography{ref}

\begin{thebibliography}{}

\bibitem[Ambrosio et~al., 2005]{ambrosio2005gradient}
Ambrosio, L., Gigli, N., and Savar{\'e}, G. (2005).
\newblock {\em Gradient flows: in metric spaces and in the space of probability
  measures}.
\newblock Springer Science \& Business Media.

\bibitem[Andrieu and Thoms, 2008]{andrieu2008tutorial}
Andrieu, C. and Thoms, J. (2008).
\newblock A tutorial on adaptive {MCMC}.
\newblock {\em Statistics and computing}, 18:343--373.

\bibitem[Bakry and {\'E}mery, 2006]{bakry2006diffusions}
Bakry, D. and {\'E}mery, M. (2006).
\newblock Diffusions hypercontractives.
\newblock In {\em S{\'e}minaire de Probabilit{\'e}s XIX 1983/84: Proceedings},
  pages 177--206. Springer.

\bibitem[Bernton, 2018]{bernton2018langevin}
Bernton, E. (2018).
\newblock Langevin {M}onte {C}arlo and {JKO} splitting.
\newblock In {\em Conference on learning theory}, pages 1777--1798. PMLR.

\bibitem[Caprio et~al., 2024]{caprio2024error}
Caprio, R., Kuntz, J., Power, S., and Johansen, A.~M. (2024).
\newblock Error bounds for particle gradient descent, and extensions of the
  log-{S}obolev and {T}alagrand inequalities.
\newblock {\em arXiv preprint arXiv:2403.02004}.

\bibitem[Carmona, 2016]{carmona2016lectures}
Carmona, R. (2016).
\newblock {\em Lectures on {BSDE}s, stochastic control, and stochastic
  differential games with financial applications}.
\newblock SIAM.

\bibitem[Chen et~al., 2024]{chen2022uniform}
Chen, F., Lin, Y., Ren, Z., and Wang, S. (2024).
\newblock Uniform-in-time propagation of chaos for kinetic mean field
  {L}angevin dynamics.
\newblock {\em Electronic Journal of Probability}, 29:1--43.

\bibitem[Cheng et~al., 2018]{cheng2018underdamped}
Cheng, X., Chatterji, N.~S., Bartlett, P.~L., and Jordan, M.~I. (2018).
\newblock Underdamped {L}angevin {MCMC}: A non-asymptotic analysis.
\newblock In {\em Conference on Learning Theory}, pages 300--323. PMLR.

\bibitem[Chizat, 2022]{chizat2022mean}
Chizat, L. (2022).
\newblock Mean-field langevin dynamics : Exponential convergence and annealing.
\newblock {\em Transactions on Machine Learning Research}.

\bibitem[Chizat and Bach, 2018]{chizat2018global}
Chizat, L. and Bach, F. (2018).
\newblock On the global convergence of gradient descent for over-parameterized
  models using optimal transport.
\newblock {\em Advances in neural information processing systems}, 31.

\bibitem[De~Bortoli et~al., 2021]{de2021efficient}
De~Bortoli, V., Durmus, A., Pereyra, M., and Vidal, A.~F. (2021).
\newblock Efficient stochastic optimisation by unadjusted {L}angevin {M}onte
  {C}arlo: {A}pplication to maximum marginal likelihood and empirical
  {B}ayesian estimation.
\newblock {\em Statistics and Computing}, 31:1--18.

\bibitem[Delyon et~al., 1999]{delyon1999convergence}
Delyon, B., Lavielle, M., and Moulines, E. (1999).
\newblock Convergence of a stochastic approximation version of the {EM}
  algorithm.
\newblock {\em Annals of statistics}, pages 94--128.

\bibitem[Diao et~al., 2023]{diao2023forward}
Diao, M.~Z., Balasubramanian, K., Chewi, S., and Salim, A. (2023).
\newblock Forward-backward {G}aussian variational inference via {JKO} in the
  {B}ures-{W}asserstein space.
\newblock In {\em International Conference on Machine Learning}, pages
  7960--7991. PMLR.

\bibitem[Dockhorn et~al., 2022]{dockhorn2022score}
Dockhorn, T., Vahdat, A., and Kreis, K. (2022).
\newblock Score-based generative modeling with critically-damped {L}angevin
  diffusion.
\newblock In {\em International Conference on Learning Representations (ICLR)}.

\bibitem[Duncan et~al., 2023]{duncan2019geometry}
Duncan, A., N{\"u}sken, N., and Szpruch, L. (2023).
\newblock On the geometry of {S}tein variational gradient descent.
\newblock {\em Journal of Machine Learning Research}, 24(56):1--39.

\bibitem[Garbuno-Inigo et~al., 2020]{garbuno2020affine}
Garbuno-Inigo, A., N\"{u}sken, N., and Reich, S. (2020).
\newblock Affine invariant interacting {L}angevin dynamics for {B}ayesian
  inference.
\newblock {\em SIAM Journal on Applied Dynamical Systems}, 19(3):1633--1658.

\bibitem[Gelfand and Silverman, 2000]{gelfand2000calculus}
Gelfand, I.~M. and Silverman, R.~A. (2000).
\newblock {\em Calculus of Variations}.
\newblock Courier Corporation.

\bibitem[Good, 1983]{good1983}
Good, I.~J. (1983).
\newblock {\em Good thinking: The foundations of probability and its
  applications}.
\newblock University of Minnesota Press.

\bibitem[Han et~al., 2017]{han2017alternating}
Han, T., Lu, Y., Zhu, S.-C., and Wu, Y.~N. (2017).
\newblock Alternating back-propagation for generator network.
\newblock In {\em Proceedings of the AAAI Conference on Artificial
  Intelligence}, volume~31.

\bibitem[Heusel et~al., 2017]{heusel2017gans}
Heusel, M., Ramsauer, H., Unterthiner, T., Nessler, B., and Hochreiter, S.
  (2017).
\newblock {GAN}s trained by a two time-scale update rule converge to a local
  nash equilibrium.
\newblock {\em Advances in Neural Information Processing Systems}, 30.

\bibitem[Hinton, 2002]{hinton2002training}
Hinton, G.~E. (2002).
\newblock Training products of experts by minimizing contrastive divergence.
\newblock {\em Neural computation}, 14(8):1771--1800.

\bibitem[Hochbruck and Ostermann, 2010]{hochbruck2010exponential}
Hochbruck, M. and Ostermann, A. (2010).
\newblock Exponential integrators.
\newblock {\em Acta Numerica}, 19:209--286.

\bibitem[Hu et~al., 2021]{hu2021mean}
Hu, K., Ren, Z., {\v{S}}i{\v{s}}ka, D., and Szpruch, {\L}. (2021).
\newblock Mean-field {L}angevin dynamics and energy landscape of neural
  networks.
\newblock In {\em Annales de l'Institut Henri Poincare (B) Probabilites et
  statistiques}, volume~57, pages 2043--2065. Institut Henri Poincar{\'e}.

\bibitem[Jordan et~al., 1998]{jordan1998variational}
Jordan, R., Kinderlehrer, D., and Otto, F. (1998).
\newblock The variational formulation of the {F}okker--{P}lanck equation.
\newblock {\em SIAM Journal on Mathematical Analysis}, 29(1):1--17.

\bibitem[Kac, 1956]{kac1956foundations}
Kac, M. (1956).
\newblock Foundations of kinetic theory.
\newblock In {\em Proceedings of The third Berkeley symposium on mathematical
  statistics and probability}, volume~3, pages 171--197.

\bibitem[Kingma and Welling, 2014]{kingma2013auto}
Kingma, D.~P. and Welling, M. (2014).
\newblock Auto-encoding variational {B}ayes.
\newblock In Bengio, Y. and LeCun, Y., editors, {\em 2nd International
  Conference on Learning Representations, {ICLR} 2014, Banff, AB, Canada, April
  14-16, 2014, Conference Track Proceedings}.

\bibitem[Krizhevsky and Hinton, 2009]{krizhevsky2009learning}
Krizhevsky, A. and Hinton, G. (2009).
\newblock Learning multiple layers of features from tiny images.

\bibitem[Kruger, 2003]{kruger2003frechet}
Kruger, A.~Y. (2003).
\newblock On {F}r{\'e}chet subdifferentials.
\newblock {\em Journal of Mathematical Sciences}, 116(3):3325--3358.

\bibitem[Kuntz et~al., 2023]{Kuntz2022}
Kuntz, J., Lim, J.~N., and Johansen, A.~M. (2023).
\newblock Particle algorithms for maximum likelihood training of latent
  variable models.
\newblock In {\em Proceedings on 26th International Conference on Artificial
  Intelligence and Statistics (AISTATS)}, volume 206 of {\em Proceedings of
  Machine Learning Research}, pages 5134--5180.

\bibitem[Lambert et~al., 2022]{lambert2022variational}
Lambert, M., Chewi, S., Bach, F., Bonnabel, S., and Rigollet, P. (2022).
\newblock Variational inference via {W}asserstein gradient flows.
\newblock {\em Advances in Neural Information Processing Systems},
  35:14434--14447.

\bibitem[LeCun et~al., 1998]{lecun1998gradient}
LeCun, Y., Bottou, L., Bengio, Y., and Haffner, P. (1998).
\newblock Gradient-based learning applied to document recognition.
\newblock {\em Proceedings of the IEEE}, 86(11):2278--2324.

\bibitem[Liu, 2017]{liu2017stein}
Liu, Q. (2017).
\newblock Stein variational gradient descent as gradient flow.
\newblock {\em Advances in Neural Information Processing Systems}, 30.

\bibitem[Lo{\`e}ve, 1977]{Loève1977}
Lo{\`e}ve, M. (1977).
\newblock {\em Probability Concepts}, pages 151--176.
\newblock Springer New York, New York, NY.

\bibitem[Ma et~al., 2021]{ma2019there}
Ma, Y.-A., Chatterji, N.~S., Cheng, X., Flammarion, N., Bartlett, P.~L., and
  Jordan, M.~I. (2021).
\newblock {Is there an analog of Nesterov acceleration for gradient-based
  {MCMC}?}
\newblock {\em Bernoulli}, 27(3):1942 -- 1992.

\bibitem[Maddison et~al., 2018]{maddison2018hamiltonian}
Maddison, C.~J., Paulin, D., Teh, Y.~W., O'Donoghue, B., and Doucet, A. (2018).
\newblock Hamiltonian descent methods.
\newblock {\em arXiv preprint arXiv:1809.05042}.

\bibitem[Martens, 2010]{martens2010deep}
Martens, J. (2010).
\newblock Deep learning via {H}essian-free optimization.
\newblock In {\em International Conference on Machine Learning}, volume~27,
  pages 735--742.

\bibitem[McCall, 2010]{mccall2010classical}
McCall, M.~W. (2010).
\newblock {\em Classical Mechanics: From {N}ewton to {E}instein: A Modern
  Introduction}.
\newblock John Wiley \& Sons.

\bibitem[McKean~Jr, 1966]{mckean1966class}
McKean~Jr, H.~P. (1966).
\newblock A class of {M}arkov processes associated with nonlinear parabolic
  equations.
\newblock {\em Proceedings of the National Academy of Sciences},
  56(6):1907--1911.

\bibitem[McLachlan and Perlmutter, 2001]{mclachlan2001conformal}
McLachlan, R. and Perlmutter, M. (2001).
\newblock Conformal {H}amiltonian systems.
\newblock {\em Journal of Geometry and Physics}, 39(4):276--300.

\bibitem[Mei et~al., 2018]{mei2018mean}
Mei, S., Montanari, A., and Nguyen, P.-M. (2018).
\newblock A mean field view of the landscape of two-layer neural networks.
\newblock {\em Proceedings of the National Academy of Sciences},
  115(33):E7665--E7671.

\bibitem[Neal and Hinton, 1998]{neal1998view}
Neal, R.~M. and Hinton, G.~E. (1998).
\newblock A view of the {EM} algorithm that justifies incremental, sparse, and
  other variants.
\newblock {\em Learning in graphical models}, pages 355--368.

\bibitem[Nemirovskij and Yudin, 1983]{nemirovskij1983problem}
Nemirovskij, A.~S. and Yudin, D.~B. (1983).
\newblock {\em Problem complexity and method efficiency in optimization}.
\newblock Wiley-Interscience.

\bibitem[Nesterov, 2003]{nesterov2003introductory}
Nesterov, Y. (2003).
\newblock {\em Introductory Lectures on Convex Optimization: A Basic Course}.
\newblock Springer Science \& Business Media.

\bibitem[Nesterov, 1983]{nesterov1983method}
Nesterov, Y.~E. (1983).
\newblock A method of solving a convex programming problem with convergence
  rate $\mathcal{O}\bigl(\frac{1 }{k^2}\bigr)$.
\newblock In {\em Doklady Akademii Nauk}, volume 269, pages 543--547. Russian
  Academy of Sciences.

\bibitem[Nijkamp et~al., 2020]{nijkamp2020learning}
Nijkamp, E., Pang, B., Han, T., Zhou, L., Zhu, S.-C., and Wu, Y.~N. (2020).
\newblock Learning multi-layer latent variable model via variational
  optimization of short run {MCMC} for approximate inference.
\newblock In {\em Computer Vision--ECCV 2020: 16th European Conference,
  Glasgow, UK, August 23--28, 2020, Proceedings, Part VI 16}, pages 361--378.
  Springer.

\bibitem[Nitanda and Suzuki, 2017]{nitanda2017stochastic}
Nitanda, A. and Suzuki, T. (2017).
\newblock Stochastic particle gradient descent for infinite ensembles.
\newblock {\em arXiv preprint arXiv:1712.05438}.

\bibitem[Nitanda et~al., 2022]{nitanda2022convex}
Nitanda, A., Wu, D., and Suzuki, T. (2022).
\newblock Convex analysis of the mean field {L}angevin dynamics.
\newblock In {\em International Conference on Artificial Intelligence and
  Statistics}, pages 9741--9757. PMLR.

\bibitem[Otto and Villani, 2000]{otto2000}
Otto, F. and Villani, C. (2000).
\newblock {Generalization of an Inequality by {T}alagrand and Links with the
  Logarithmic {S}obolev Inequality}.
\newblock {\em Journal of Functional Analysis}, 173(2):361--400.

\bibitem[O’Donoghue and Candes, 2015]{o2015adaptive}
O’Donoghue, B. and Candes, E. (2015).
\newblock Adaptive restart for accelerated gradient schemes.
\newblock {\em Foundations of Computational Mathematics}, 15:715--732.

\bibitem[Pang et~al., 2020]{pang2020learning}
Pang, B., Han, T., Nijkamp, E., Zhu, S.-C., and Wu, Y.~N. (2020).
\newblock Learning latent space energy-based prior model.
\newblock {\em Advances in Neural Information Processing Systems},
  33:21994--22008.

\bibitem[Peyr\'e and Cuturi, 2019]{peyre2019}
Peyr\'e, G. and Cuturi, M. (2019).
\newblock Computational optimal transport: With applications to data science.
\newblock {\em Foundations and Trends in Machine Learning}, 11(5-6):355--607.

\bibitem[Platen and Bruti-Liberati, 2010]{platen2010numerical}
Platen, E. and Bruti-Liberati, N. (2010).
\newblock {\em Numerical Solution of Stochastic Differential Equations with
  Jumps in Finance}.
\newblock Springer Science \& Business Media.

\bibitem[Riou-Durand et~al., 2023]{riou2023adaptive}
Riou-Durand, L., Sountsov, P., Vogrinc, J., Margossian, C., and Power, S.
  (2023).
\newblock Adaptive tuning for {M}etropolis adjusted {L}angevin trajectories.
\newblock In {\em International Conference on Artificial Intelligence and
  Statistics}, pages 8102--8116. PMLR.

\bibitem[Roberts and Rosenthal, 2009]{roberts2009examples}
Roberts, G.~O. and Rosenthal, J.~S. (2009).
\newblock Examples of adaptive {MCMC}.
\newblock {\em Journal of computational and graphical statistics},
  18(2):349--367.

\bibitem[Santambrogio, 2017]{santambrogio2017euclidean}
Santambrogio, F. (2017).
\newblock $\{$Euclidean, metric, and Wasserstein$\}$ gradient flows: an
  overview.
\newblock {\em Bulletin of Mathematical Sciences}, 7:87--154.

\bibitem[Sanz-Serna and Zygalakis, 2021]{sanz2021wasserstein}
Sanz-Serna, J.~M. and Zygalakis, K.~C. (2021).
\newblock Wasserstein distance estimates for the distributions of numerical
  approximations to ergodic stochastic differential equations.
\newblock {\em The Journal of Machine Learning Research}, 22(1):11006--11042.

\bibitem[Sharrock and Nemeth, 2023]{sharrock2023coinem}
Sharrock, L. and Nemeth, C. (2023).
\newblock Coin sampling: Gradient-based {B}ayesian inference without learning
  rates.
\newblock In Krause, A., Brunskill, E., Cho, K., Engelhardt, B., Sabato, S.,
  and Scarlett, J., editors, {\em Proceedings of the 40th International
  Conference on Machine Learning}, volume 202 of {\em Proceedings of Machine
  Learning Research}, pages 30850--30882. PMLR.

\bibitem[Shi et~al., 2021]{shi2021understanding}
Shi, B., Du, S.~S., Jordan, M.~I., and Su, W.~J. (2021).
\newblock Understanding the acceleration phenomenon via high-resolution
  differential equations.
\newblock {\em Mathematical Programming}, 195:79--148.

\bibitem[Silvester, 2000]{silvester2000determinants}
Silvester, J.~R. (2000).
\newblock Determinants of block matrices.
\newblock {\em The Mathematical Gazette}, 84(501):460--467.

\bibitem[Staib et~al., 2019]{staib2019escaping}
Staib, M., Reddi, S., Kale, S., Kumar, S., and Sra, S. (2019).
\newblock Escaping saddle points with adaptive gradient methods.
\newblock In {\em International Conference on Machine Learning}, pages
  5956--5965. PMLR.

\bibitem[Su et~al., 2014]{su2014differential}
Su, W., Boyd, S., and Candes, E. (2014).
\newblock A differential equation for modeling {N}esterov’s accelerated
  gradient method: theory and insights.
\newblock {\em Advances in Neural Information Processing Systems}, 27.

\bibitem[Sutskever et~al., 2013]{sutskever2013importance}
Sutskever, I., Martens, J., Dahl, G., and Hinton, G. (2013).
\newblock On the importance of initialization and momentum in deep learning.
\newblock In {\em International Conference on Machine Learning}, pages
  1139--1147. PMLR.

\bibitem[Suzuki et~al., 2023]{suzuki2023convergence}
Suzuki, T., Wu, D., and Nitanda, A. (2023).
\newblock Mean-field langevin dynamics: Time-space discretization, stochastic
  gradient, and variance reduction.
\newblock In {\em Thirty-seventh Conference on Neural Information Processing
  Systems}.

\bibitem[Tieleman and Hinton, 2012]{tieleman2012lecture}
Tieleman, T. and Hinton, G. (2012).
\newblock Lecture 6.5-rmsprop, coursera: Neural networks for machine learning.
\newblock {\em University of Toronto, Technical Report}, 6.

\bibitem[Tomczak and Welling, 2018]{tomczak2018vae}
Tomczak, J. and Welling, M. (2018).
\newblock {VAE} with a {V}amp{P}rior.
\newblock In {\em International Conference on Artificial Intelligence and
  Statistics}, pages 1214--1223. PMLR.

\bibitem[Van~Handel, 2014]{van2014probability}
Van~Handel, R. (2014).
\newblock Probability in high dimension.
\newblock {\em Lecture Notes (Princeton University)}.

\bibitem[Villani, 2009]{villani2009optimal}
Villani, C. (2009).
\newblock {\em Optimal transport: old and new}, volume 338.
\newblock Springer.

\bibitem[Wibisono, 2018]{wibisono2018}
Wibisono, A. (2018).
\newblock Sampling as optimization in the space of measures: The {L}angevin
  dynamics as a composite optimization problem.
\newblock In {\em Conference on Learning Theory}, pages 2093--3027. PMLR.

\bibitem[Wibisono et~al., 2016]{wibisono2016variational}
Wibisono, A., Wilson, A.~C., and Jordan, M.~I. (2016).
\newblock A variational perspective on accelerated methods in optimization.
\newblock {\em Proceedings of the National Academy of Sciences},
  113(47):E7351--E7358.

\bibitem[Wilson et~al., 2021]{wilson2016lyapunov}
Wilson, A.~C., Recht, B., and Jordan, M.~I. (2021).
\newblock A {L}yapunov analysis of accelerated methods in optimization.
\newblock {\em Journal of Machine Learning Research}, 22(113):1--34.

\bibitem[Yao and Yang, 2022]{yao2022mean}
Yao, R. and Yang, Y. (2022).
\newblock Mean field variational inference via {W}asserstein gradient flow.
\newblock {\em arXiv preprint arXiv:2207.08074}.

\end{thebibliography}

\clearpage

\appendix
\onecolumn
\appendixpage
\appendixtitleon
\startcontents[sections]
\printcontents[sections]{l}{1}{\setcounter{tocdepth}{2}}

\section{Notation}

The following table summarizes some key notation used throughout.

\bgroup
\def\arraystretch{1.5}
\begin{tabular}{p{1in}p{5.25in}}
	$\cal{E}$ & Free energy defined in \Cref{eq:freeenergy}. \\
	$\cal{F}$ & Momentum-enriched free energy defined in \Cref{eq:hamilboth}. \\
	$z_t$ & The tuple $(\theta_t, m_t, q_t) \in \mathbb{R}^{d_\theta} \times \mathbb{R}^{d_\theta} \times \mathcal P(\mathbb{R}^{d_x} \times \mathbb{R}^{d_x})$.\\
        $\nabla f$ & Euclidean gradient of $f$. If $f:\r^{n} \rightarrow \r^{m}$ is vector-valued, we have $\nabla f \in \r^{n \times m}$. \\
	$\nabla_a f$ & Euclidean gradient of $f$ w.r.t.\ $a$.  \\
	$\nabla_{(a,b)} f$ &  $\left [\nabla_{a} f, \nabla_{b} f\right ]^\top$.\\
	$\ell$ & $\ell (\theta, x) :=  \log p_\theta (y, x)$. \\
	$\rho_\theta$ & $\rho_\theta(x) := p_\theta(y, x)$. \\
        $\nabla \cdot f$ & If $f: \r^{n} \rightarrow \r^{n}$, we have $\nabla\cdot f := \partial_i f_i$. If $f: \r^{n} \rightarrow \r^{n} \times \r^{m}$, we have $(\nabla \cdot f)_i := \partial_j f_{ji}$.\\
	$\nabla_{a} \cdot $ & Divergence operator w.r.t. $a$. \\
	$\nabla^*_{a}$ & Adjoint operator of $\nabla_a$ (see \Cref{sec:adjoint}).\\
	$[n]$ & $[n] := \{1,..., n\}$.\\
	$\Delta$ & Laplacian operator, $\Delta = \nabla\cdot\nabla$. \\
	$\cal{P}(\mathbb{R}^d)$ & The space of probability measures that are absolutely continuous w.r.t. Lebesgue measure (have densities) and possess finite second moments. \\
	$\iprod{\cdot}{\cdot}$ (and $\|\cdot\|$)  & Euclidean inner product or Frobenius inner product (and its inner norm).\\
	$\iprod{\cdot}{\cdot}_\rho$ (and $\|\cdot\|_\rho $) & $L^2(\rho)$ inner product (and its norm). \\
	$o(\epsilon)$ & Bachmann--Landau little-o notation. \\
	$\partial F(\mu)$ & Fr\'{e}chet subdifferential (see \Cref{def:frechet_wasserstein}).
\end{tabular}

\section{Related Work}

The present work sits at the juncture of i) deterministic gradient flows for optimising
objectives over ``parameter’' spaces, typically expressible through the discretization of ODEs, and
ii) stochastic gradient flows for optimising objectives over the space of probability measures,
typically expressible through discretisation of mean-field SDEs. The former class of problems
is too vast to be properly surveyed here (for an overview, see \citet{santambrogio2017euclidean, ambrosio2005gradient}), effectively including a large proportion of modern
continuous optimisation problems. The latter class has seen substantial growth over the past
few years in particular, with various problems related to sampling \citep{liu2017stein, bernton2018langevin, garbuno2020affine, duncan2019geometry}, variational inference
\citep{yao2022mean, lambert2022variational, diao2023forward}, and the training of shallow neural networks \citep{mei2018mean, chizat2018global, nitanda2017stochastic, hu2021mean, nitanda2022convex, chizat2022mean, chen2022uniform, suzuki2023convergence} being studied in this framework. While
there exist earlier works which combine optimisation with Markovian sampling (e.g. stochastic
approximation approaches to the EM algorithm \citep{delyon1999convergence, de2021efficient}, training of energy-based models \citep{hinton2002training}, and
hyperparameter tuning in MCMC \citep{andrieu2008tutorial, roberts2009examples}), the connection to gradient flows remains somewhat
under-developed at present. We hope that the present work can encourage further
exploration of these connections.

\section{Gradient Flow on $\mathbb{R}^{d_\theta} \times \cal{P}(\mathbb{R}^{d_x})$}

\label{appen:gf}
In this section, we describe gradient flows on the extended space $\mathbb{R}^{d_\theta} \times \cal{P}(\mathbb{R}^{d_x})$. We begin with an exposition of gradient flows in Wasserstein space, i.e. the space of distributions $\cal{P}(\mathbb{R}^{d_x})$ endowed with the Wasserstein-$2$ metric. We aim to give an intuitive introduction as opposed to a rigorous one; those readers with an interest in the latter are directed to \citet{ambrosio2005gradient}. Using the ideas in \citet{ambrosio2005gradient}, we show how the notion of gradients can be generalized to the product space $\mathbb{R}^{d_\theta} \times \cal{P}(\mathbb{R}^{d_x})$.

\subsection{Gradient Flow on $\cal{P}(\mathbb{R}^{d_x})$}
\label{appen:gf_w2}
We begin by describing the gradient flow on $\cal{P}(\mathbb{R}^{d_x})$ endowed with the Wasserstein-$2$ metric. The Wasserstein distance is defined as
\begin{align*}
    \sf{W}_2^2(p,q) = \inf_{\pi \in \Pi(p,q)} \int_{\r^d \times \r^d}\left \| x - y \right \|^2 \pi(\mathrm{d} x  \times \mathrm{d} y),
\end{align*}
where $\Pi(p,q)$ is the set of all couplings between $p$ and $q$.
 
In \citet[Chapter 11]{ambrosio2005gradient}, the authors discuss various approaches for adapting gradient flows on well-studied spaces (such as Euclidean and Riemannian spaces) to the Wasserstein space. One of these approaches proceeds by first defining suitable notions of tangent space and subdifferential, following which the simple definition of gradient flow modelled on Riemannian manifolds can then be reproduced. In this case, the Fr\'{e}chet subdifferential \citep[Definition 10.1.1]{ambrosio2005gradient} is defined as follows:
\begin{definition}[Fr\'{e}chet differential on Wasserstein Space] \label{def:frechet_wasserstein}
Let $F:\cal{P}(\mathbb{R}^{d_x}) \rightarrow \mathbb{R}$ be a sufficiently regular function. We say that $\xi \in L^2(q)$ belongs to the Fr\'{e}chet subdifferential $\partial F (q)$ if for all $q' \in \cal{P}(\mathbb{R}^{d_x})$, we have
\begin{align*}
    F(q') - F(q) \ge \iprod{\xi}{t_q^{q'}- i}_q + o(\sf{W}_2(q, q')),
\end{align*}
where $t_p^q$ is the optimal map between $p$ and $q$ \citep[see (7.1.4)]{ambrosio2005gradient} and $i$ is the identity map. Furthermore, if $\xi \in \partial F (q)$ also satisfies
\begin{align*}
    F(t_{\#}q) - F(q) \ge \iprod{\xi}{t-i}_q + o(\|t-i\|_q),
\end{align*}
for all $t \in L^2(q)$, then we say that $\xi$ is a \textbf{strong} subdifferential.
\end{definition}

See \citet[Definition 10.1.1]{ambrosio2005gradient} for more details. The strong subdifferential can be thought of as the (Wasserstein) ``gradient'' of $F$. Equipped with this notion of gradient, we can define the gradient (descent) flow of $F:\cal{P}(\mathbb{R}^{d_x}) \rightarrow \mathbb{R}$ as follows:
\begin{definition}[Gradient Flow]
	We say that  a curve $p_t: [0,1] \rightarrow \cal{P}(\mathbb{R}^{d_x})$ is a gradient flow of $F: \cal{P}(\mathbb{R}^{d_x}) \rightarrow \mathbb{R}$ if for all $t>0$, it satisfies the continuity equation $\partial_t p_t + \nabla_x \cdot (v_t p_t) = 0$, where the tangent vector $v_t$ satisfies $- \,v_t \in \partial F(p_t)$ for all $t$.
\end{definition}
Thus, for our application, we are interested in computing the strong subdifferential of $F$. If $F$ is an integral of the type
\begin{equation}\label{eq:f_type_w}
    F(p) = \int f(x, p(x), \nabla_x p(x))\mathrm{d} x ,
\end{equation}
where $f:\mathbb{R}^{d_x}\times\mathbb{R}\times\mathbb{R}^{d_x}\to\mathbb{R}$ is sufficiently regular, then we will see that its strong subdifferential admits an analytic solution.

Functionals of this form of great interest in the calculus of variations (e.g., see \citet{gelfand2000calculus}). A vital quantity which is used to study these functionals is the first variation. Writing $\delta_p F[p] : \mathbb{R}^{d_x} \rightarrow \mathbb{R}$ for the first variation of $F$, it is a function that satisfies
\begin{equation*}
\left. \frac{\rm{d}}{\rm{d}\epsilon} F(q_\epsilon)\right\vert_{\epsilon=0} = - \int \delta_q F[q] (x) \, \nabla_x \cdot (q(x)v(x)) \rm{d}x,
\end{equation*}
for all $v$ such that $q_\epsilon := (i + \epsilon v)_\# q \in \mathcal{P}(\mathbb{R}^{d_x})$ for sufficiently small $\epsilon$.
One can readily establish that for \eqref{eq:f_type_w}-typed $F$, it is given by
\begin{align*}
	\delta_p F[p](x) := \nabla_{(2)} f(x, p(x), \nabla_x p(x)) - \nabla_x \cdot (\nabla_{(3)} f(x, p(x), \nabla_x p(x)),
\end{align*}
where $\nabla_{(i)}$ denotes the partial derivative w.r.t.~the $i$-th argument \citep[Eq.\ (10.4.2)]{ambrosio2005gradient}. It can be shown that for any $\xi \in \partial F(p)$ which is a strong subdifferential, it holds that $\xi(x) \overset{p-a.e.}{=} \nabla_x \delta_p F[p](x)$  \citep[Lemma 10.4.1]{ambrosio2005gradient}.

\subsection{Gradient flow on  $\mathbb{R}^{d_\theta} \times \cal{P}(\mathbb{R}^{d_x})$} \label{sec:gc_on_product_space}

We provided here a generalization of the Fr\'{e}chet differential (for a broad survey of Fr\'{e}chet differentials, see \citet{kruger2003frechet}) to the extended space $\mathbb{R}^{d_\theta} \times \cal{P}(\mathbb{R}^{d_x})$:
\begin{definition}[Fr\'{e}chet differential on $\mathbb{R}^{d_\theta} \times \cal{P}(\mathbb{R}^{d_x})$]\label{eq:strong_extended_sd}
Let $\cal{G} :\mathbb{R}^{d_\theta}\times \cal{P}(\mathbb{R}^{d_x}) \rightarrow \mathbb{R}$ be a sufficiently regular function. We say that $(\xi_\theta, \xi_q) \in \mathbb{R}^{d_\theta} \times L_2(q)$ belongs to the Fr\'{e}chet subdifferential $\partial \cal{G} (\theta, q)$ if for all $(\theta', q') \in \mathbb{R}^{d_\theta} \times \cal{P}(\mathbb{R}^{d_x})$, it holds that
\begin{align*}
    \cal{G}(\theta',q') - \cal{G}(\theta, q) \ge \iprod{\xi_\theta}{\theta' - \theta}+ \iprod{\xi_q}{t_{q}^{q'} - i}_q + o(\sf{W}_2(q', q) + \|\theta'-\theta\|).
\end{align*}
Furthermore, if $(\xi_\theta, \xi_q) \in \partial \cal{G} (\theta, q)$ also satisfies
\begin{align*}
\cal{G}(\theta + \tau,t_{\#}q) - \cal{G}(\theta, q) \ge \iprod{\xi_\theta}{\tau}+ \iprod{\xi_q}{t - i}_q + o(\|t-i\|_q + \|\tau\|),
\end{align*}
for all $(\tau, t) \in \mathbb{R}^{d_\theta} \times L_2(q)$, then we say that $(\xi_\theta, \xi_q)$ is a \textbf{strong} subdifferential. 
\end{definition}
For the remainder of the section, we shall assume that for any perturbation $(\sigma, v)$, there exists some $\delta_q F[\theta, q]:\mathbb{R}^{d_x} \rightarrow \mathbb{R}$ (called the first variation) such that the following expansion holds:
\begin{align} \label{eq:fv_extended}
	\frac{\rm{d}}{\rm{d}\epsilon} \cal{G}(\theta_\epsilon, q_\epsilon)\big \vert_{\epsilon = 0} =& \iprod{\nabla_\theta \cal{G} (\theta, q)}{\sigma} - \int \delta_q \cal{G}[\theta, q](x)\, \nabla_x \cdot (v(x) q(x))\, \mathrm{d} x, 
\end{align}
where $q_\epsilon := (i + \epsilon v)_\# q$ and $\theta_\epsilon := \theta + \epsilon \sigma$. For all $\cal{G}$ of interest in this paper, this assumption holds; for instance, see \Cref{prop:fv_L1}.

We follow in the argument of \citet[Lemma 10.4.1]{ambrosio2005gradient} to show that if $(\xi_\theta, \xi_q) \in \partial \cal{G}(\theta, q)$ is a strong subdifferential, then we have that $\xi_\theta = \nabla_\theta \cal{G}(\theta, q)$ and $\xi_q(x) \overset{q-a.e.}{=} \nabla_x \delta_q \cal{G}[\theta, q](x)$. Let $(\xi_\theta, \xi_p) \in \partial \cal{G}(\theta, p)$ be a strong subdifferential.
By using the strong subdifferential property for some perturbation $(\epsilon\sigma, \epsilon v)$ and taking left and right limits of \eqref{eq:strong_extended_sd}, we obtain that
\begin{align*}
    \limsup_{\epsilon \uparrow 0} \frac{\cal{G}(\theta_\epsilon, q_\epsilon) - \cal{G}(\theta, q)}{\epsilon} \le \iprod{\xi_\theta}{\sigma} + \iprod{\xi_q}{v}_q \le \liminf_{\epsilon \downarrow 0} \frac{\cal{G}(\theta_\epsilon, q_\epsilon)- \cal{G}(\theta, q)}{\epsilon}.
\end{align*}
On the other hand, from our assumption \eqref{eq:fv_extended}, we have that
\begin{align*}
	\lim_{\epsilon \rightarrow 0} \frac{\cal{G}(\theta_\epsilon, q_\epsilon) - \cal{G}(\theta, q)}{\epsilon} &= \iprod{\nabla_\theta \cal{G} (\theta, q)}{\sigma} - \int \delta_q \cal{G}[\theta, q](x)\, \nabla_x \cdot (v(x) q(x)) \, \mathrm{d} x \\
	&= \iprod{\nabla_\theta \cal{G} (\theta, q)}{\sigma} + \iprod{\nabla_x \delta_q \cal{G}[\theta, q]}{v}_q,
\end{align*}
where the last equality follows from integration by parts and the divergence theorem. Hence, we have that $\xi_\theta = \nabla_\theta \cal{G}(\theta, q)$, and $\xi_q \overset{q-a.e.}{=} \nabla_x \delta_q \cal{G}[\theta, q]$.

\subsection{First Variation}
\label{sec:fv}

In this section, we derive the first variation for integral expressions taking a certain form. In particular, we are interested in computing the strong subdifferential of $\cal{G}$ of the following types:
\begin{align}
	\cal{G}(\theta, q) &:= \int {f}(\theta, x, q(x), \nabla_x q(x))\,\mathrm{d} x , \text{ where } f:\r^{d_\theta} \times \r^{d_x} \times \r \times \r^{d_x}; \tag{VI-I}
	\label{eq:vi_type1} \\
	\cal{G}(\theta, q) &:= \int \int f(\theta, x, x^\prime, q(x), q(x^\prime))\,\mathrm{d} x  \mathrm{d} x^\prime, \text{ where } f:\r^{d_\theta} \times \r^{d_x}  \times \r^{d_x} \times \r \times \r. \tag{VI-II}
	\label{eq:vi_type2}
\end{align}
An exampled of \eqref{eq:vi_type1}-typed $\cal{G}$ is when $\cal{G}$ is the free energy $\cal{E}$. Following standard techniques from the calculus of variations \citep{gelfand2000calculus}, consider a perturbation $(\sigma, v) \in \mathbb{R}^{d_\theta} \times L_2(q)$ and define a mapping $\Phi$ by
\begin{align*}
    \Phi (\epsilon) := \cal{G}(\theta_\epsilon, q_\epsilon ),
\end{align*}
where $q_\epsilon := (i + \epsilon v)_\# q$ and $\theta_\epsilon := \theta + \epsilon \sigma$.

\textbf{\eqref{eq:vi_type1}-typed $\cal{G}$}. By application of the change-of-variables formula (for instance, see \citet[Lemma  5.5.3]{ambrosio2005gradient}), we can compute the density $q_\epsilon$ and its derivative $\frac{\partial}{\partial \epsilon} q_\epsilon (y)$ as follows:
\begin{align*}
	q_\epsilon(y) &= \frac{q}{\rm{det}(I+\epsilon \nabla_x v )} \circ (i + \epsilon v)^{-1}(y), \quad
	\frac{\partial}{\partial \epsilon} q_\epsilon(y) \big |_{\epsilon = 0}  = - \nabla_x \cdot [q v ] (y).
\end{align*}
Thus, we can compute the derivative of $\Phi$, provided that the interchange of derivative and integral can be justified, as
%
\begin{align*}
	\frac{\rm{d}}{\rm{d}\epsilon}\Phi(\epsilon) \bigg \vert_{\epsilon =0} =& \int \frac{\rm{d}}{\rm{d}\epsilon} f(\theta_\epsilon, x, q_\epsilon(x), \nabla_x q_\epsilon(x))\,\mathrm{d} x  \\
	=& \int \iprod{\nabla_\theta f(\theta, x, q(x), \nabla_x q(x))}{\sigma} \,\mathrm{d} x  \\
	&- \int \nabla_{(3)} f(\theta, x, q(x), \nabla_x q(x)) \, \nabla_x \cdot [vq](x) \, \mathrm{d} x  \\
	& - \int \iprod{\nabla_{(4)} f(\theta, x, q(x), \nabla_x q(x))}{\nabla_x [\nabla_x \cdot [vq]](x)}\, \mathrm{d} x.
\end{align*}
Applying integration by parts and the divergence theorem, we obtain that
\begin{align*}
	\frac{\rm{d}}{\rm{d}\epsilon}\Phi(\epsilon) \bigg \vert_{\epsilon =0} =& \iprod{\nabla_\theta \cal{G}(\theta, q)}{\sigma} \\
	&- \int \nabla_{(3)} f(\theta, x, q(x), \nabla_x q(x)) \, [\nabla_x \cdot (vq)](x) \, \mathrm{d} x \\
	& + \int \nabla_x \cdot [\nabla_{(4)} f](\theta, x, q(x), \nabla_x q(x)) [\nabla_x \cdot (vq)](x) \, \mathrm{d} x.
\end{align*}
We can hence write the first variation of \eqref{eq:vi_type1}-typed $\cal{G}$ as
\begin{equation}\label{eq:fv_t1}
    \delta_q \cal{G}[\theta, q](x) = \nabla_{(3)} f(\theta, x, q(x), \nabla_x q(x)) - \nabla_x \cdot (\nabla_{(4)} f(\theta, x, q(x), \nabla_x q(x)).
\end{equation}

\textbf{\eqref{eq:vi_type2}-typed $\cal{G}$}. Similarly to above, we can define $\Phi(\epsilon) := \cal{G}(\theta_\epsilon, q_\epsilon )$, whose derivative is given by
\begin{align*}
	\frac{\rm{d}}{\rm{d}\epsilon}\Phi(\epsilon) \bigg \vert_{\epsilon =0} =& \iprod{\nabla_\theta \cal{G} (\theta, q)}{\sigma} \\
	&- \int \int \nabla_{(4)} f(\theta, x, y, q(x), q(y))\, \nabla_x \cdot [qv] (x)\,\mathrm{d} x \, \mathrm{d} y\\
	&- \int \int \nabla_{(5)} f(\theta, x, y, q(x), q(y)) \, \nabla_y \cdot [qv] (y)\,\mathrm{d} x \, \mathrm{d} y.
\end{align*}
Hence, the first variation is given by
\begin{align}
	\delta_q \cal{G}[\theta, q](x) =& \int \left [\nabla_{(4)} f(\theta, x, y, q(x), q(y)) + \nabla_{(5)} f(\theta, y, x, q(y), q(x)) \right ] \mathrm{d} y.
	\label{eq:fv_t2}
\end{align}
\begin{example}[First Variation of $\cal{E}$]
    It can be seen that the free energy is \eqref{eq:vi_type1}-typed with $f(\theta, x, a, g) = a \cdot \log \frac{a}{p_\theta(y, x)}$. Hence, by combining the formula \eqref{eq:fv_t1} with the fact that
        \begin{align*}
        	\nabla_{3} f(\theta,x, a, b) &= \log \frac{a}{p_\theta} + 1, \\
        	\nabla_{4}(f(\theta,x,a,g) &= 0,
        \end{align*}
    we obtain the following expression for the first variation:
        \begin{align*}
            \delta_q \cal{E} [\theta, q](x) = \log \frac{q(x)}{p_\theta(y,x)} + 1.
        \end{align*}
    This provides an alternative derivation of this expression to that given in \citet[Lemma 1]{Kuntz2022}.
\end{example}

\section{A log Sobolev inequality for $\cal{E}$ can be transferred to $\cal{F}$}
In this section, we show that nice properties of the functional $\cal{E}$ transfer to the functional $\cal{F}$, i.e. if $\cal{E}$ satisfies a log Sobolev inequality (\Cref{ass:logsobolev}), then so does $\cal{F}$ (with a modified constant):
\begin{proposition}\label{prop:logsobF}
    If \Cref{ass:logsobolev} holds, then upon defining $C:=\min\{C_{\cal{E}},\eta_\theta,\eta_x\}$, it holds that
	$$
	\cal{F}(\theta,\tmo,q)-\cal{E}^* \le \frac{1}{2C} \norm{\nabla\cal{F}(\theta,\tmo,q)}_q^2, \quad \forall \theta, \tmo\in \r^{d_\theta}, \enskip q\in\cal{P}(\r^{2d_x}),
	$$
 where 
\begin{equation}
\label{eq:gradZnorm}
  \norm{\nabla\cal{F}(\theta,\tmo,q)}_q^2 := \|\nabla_{(\theta, m)} \cal{F}(\theta, m, q)\|^2 + \|\nabla_{(x,\qmo)} \log q - \nabla_{(x,\qmo)} \log \rho_{\theta, \eta_x}\|^2_q.
\end{equation}
\end{proposition}
\begin{proof}
    As we show at the end of the proof, we have the fact that
	\begin{equation}
		\norm{\nabla_{(x,\qmo)} \log\left (\frac{q}{\rho_\theta\otimes  r_{\eta_x}} \right )}_q^2 
		\ge {\norm{\nabla_{x} \log \left ( \frac{q_X}{\rho_\theta} \right )}^2}_{q_X}
		+ {\norm{ \nabla_{\qmo} \log \frac{q_{\Qmo|X}}{r_{\eta_x}}}^2}_q.
		\label{eq:lya_bound}
	\end{equation}
	From the definition, we have
	\begin{equation*}
		\norm{\nabla\cal{F}(\theta, \tmo, q)}_{q}^2 = 
		\norm{\nabla_\theta  \cal{F} (\theta, \tmo, q)}^2
		+ \norm{\eta_\theta m}^2
		+ \norm{\nabla_{(x,\qmo)} \log\frac{q}{\rho_{\theta, \eta_x}}}_q^2.\end{equation*}
	Using ~\eqref{eq:lya_bound} and the fact that 
	$$\nabla_\theta\cal{F}(\theta, m, q) = -\int \nabla_\theta \log \rho_{\theta, \eta_x}(x,\qmo) q(\rm{d}x,\rm{d}\qmo) =-\int \nabla_\theta \log \rho_\theta(x)q_X(\rm{d}x)= \nabla_\theta\cal{E}(\theta, q_X),$$ 
	then it follows that
	\begin{align*}
		\norm{\nabla\cal{F}(\theta, m, q)}_{q}^2
		\ge&\norm{\nabla_\theta \cal{E}(\theta, q_X)}_{q}^2
		+ \eta_\theta^2\norm{ m}^2
		+ {\norm{ \nabla_{\qmo} \log \frac{q_{U|X}}{r_{\eta}}}^2}_q.
	\end{align*}
	Because $C = \min \{C_{\cal{E}}, \eta_\theta, \eta_x\}$, the above implies that
	\begin{align}\label{eq:nfe987anf8uafdsa}
		\frac{1}{2C} \norm{\nabla\cal{F}(\theta, \tmo, q)}_{q}^2 &\ge 
		\frac{1}{2C_{\cal{E}}} \norm{\nabla_\theta \cal{E}(\theta, q)}_{q}^2
		+\frac{\eta_\theta}{2}\norm{m}^2
		+ \frac{1}{2\eta_x} {\norm{ \nabla_{\qmo} \log \frac{q_{\Qmo|X}}{r_{\eta_x}}}^2_q}.\end{align}
	Since $r_{\eta_x} = \mathcal{N}(0, \eta_x^{-1}I_{d_x})$ is $\eta_x$-strongly log-concave, by the Bakry--{\'E}mery criterion of \citet{bakry2006diffusions}, it holds that
	$$\frac{1}{2\eta_x} {\norm{ \nabla_{\qmo} \log \left ( \frac{q(\cdot |x)}{r_{\eta_x}}\right ) }^2_{q(\rm{d}\qmo|x)} }\geq \mathsf{KL}(q_{\Qmo|x}|r_{\eta_x})\quad\forall x\in\cal{X},$$
	and so we obtain that
	\begin{align*}
		\frac{1}{2\eta_x} {\norm{ \nabla_{\qmo} \log \frac{q_{\Qmo|X}}{r_{\eta_x}}}^2}_q &= \frac{1}{2\eta_x} \Ebbs{q_X}{\norm{ \nabla_{\qmo} \log \frac{q(\cdot |X)}{r_{\eta_x}}}^2_{q(\rm{d}u|X)}}\\
		&\ge \Ebbs{q_X}{\mathsf{KL}(q_{\Qmo|X}|r_{\eta_x})}.
	\end{align*}
	Plugging the above into \eqref{eq:nfe987anf8uafdsa} and using the log Sobolev inequality for $\cal{E}$, we obtain that for $\cal{F}$:
	\begin{align*}
		\frac{1}{2C} \norm{\nabla\cal{F}(\theta, \tmo, q)}_{q}^2
		&\ge
		\cal{E}(\theta, q_X)-\cal{E}^* + \frac{\eta_\theta}{2}\| m\|_2^2 +   
		\Ebbs{q_X}{\mathsf{KL}(q_{\Qmo|X}|r_{\eta_x})} \\
		&\geq\cal{F}(\theta, m, q) -\cal{E}^*.
	\end{align*}
	
	We have one loose end to tie up: proving \eqref{eq:lya_bound}. To do so, note that
	\begin{align}
		\norm{\nabla_{(x,\qmo)} \log \left ( \frac{q}{\rho_\theta  \otimes r_{\eta_x}} \right )}_q^2 
		=& {\norm{\nabla_{(x,\qmo)} \log \left (\frac{q_X}{\rho_\theta} \right ) + \nabla_{(x,\qmo)} \log \left (\frac{q_{\Qmo|X}}{r_{\eta_x}} \right ) }^2_q}\nonumber\\
		=& {\norm{\nabla_{x} \log \left (\frac{q_X}{\rho_\theta} \right ) + \nabla_{x} \log q_{\Qmo|X}}^2_q} \\
		&+ {\norm{\nabla_{\qmo} \log \left (\frac{q_{\Qmo|X}}{r_{\eta_x}} \right )}^2_q}.\label{eq:lya_bound_decomp}
	\end{align}
	Taking the first term and expanding the square, we have
	\begin{align*}
		\norm{\nabla_{x} \log \left( \frac{q_X}{\rho_\theta} \right ) + \nabla_{x} \log q_{U|X}}^2_{q}
		=& {\norm{\nabla_{x} \log \left( \frac{q_X}{\rho_\theta} \right )}^2_q} \\
		&+ 2 \iprod{\nabla_{x} \log \left ( \frac{q_X}{\rho_\theta} \right )}
		{\nabla_{x} \log q_{\Qmo|X}}_q \\
		&+ {\norm{\nabla_{x} \log q_{\Qmo|X}}^2_q}.
	\end{align*}
	Applying Jensen's inequality to the final term, we obtain
	\begin{align*}
		\norm{\nabla_{x} \log \left ( \frac{q_X}{\rho_\theta} \right ) + \nabla_{x} \log q_{U|X}}^2_q
		\ge& {\norm{\nabla_{x} \log \left ( \frac{q_X}{\rho_\theta} \right ) }^2}_{q_X} \\
		&+ 2 {\iprod{\nabla_{x} \log \left ( \frac{q_X}{\rho_\theta} \right ) } 
			{\Ebbs{q(\rm{d}\qmo|\cdot )}{\nabla_{x} \log q(\Qmo|\cdot )}}_{q_X}} \\
		&+ {\norm{\Ebbs{q(\rm{d}\qmo|\cdot )}{ \nabla_{x} \log q(\Qmo|\cdot )}}^2_{q_X}}.
	\end{align*}
    One can compute explicitly that for all $x$, it holds that $\Ebbs{q(\rm{d}\qmo|x)}{\nabla_{x} \log q(\Qmo|x)} = 0$; we thus obtain that
    \begin{align*}
        \norm{\nabla_{x} \log \frac{q_X}{\rho_\theta} + \nabla_{x} \log q_{U|X}}^2_q 
	\ge \norm{\nabla_{x} \log \frac{q_X}{\rho_\theta}}^2_{q_X},
    \end{align*}
	and so \eqref{eq:lya_bound} follows from \eqref{eq:lya_bound_decomp}. 
\end{proof}
\section{$\cal{F}(\theta_t, \tmo_t, q_t)$ is non-increasing}\label{sec:non-increasing}
In the following proposition, we show that $\cal{F}_t:= \cal{F}(\theta_t, \tmo_t, q_t)$ is non-increasing in time.
\begin{proposition}\label{prop:contraction}
	For any $\gamma_x \ge 0$ and $\gamma_\theta \ge 0$, it holds that
    \begin{align*}
        \dot{\mathcal{F}}_t = -\gamma_\theta \|\nabla_{m}\cal{F}_t\|^2  -  \gamma_x \|\nabla_{\qmo}\delta_{q}\cal{F}_t\|^2_{q_t}\le 0.
    \end{align*}
\end{proposition}
\begin{proof}
	We begin by computing the time derivative
	\begin{align}
	   \dot{\cal{F}}_t =& - \left\langle \nabla_{(\theta, m)} \cal{F}_t,\begin{pmatrix}
			0 & -I_{d_\theta}\\
			I_{d_\theta} & \gamma_{\theta}I_{d_\theta}
		\end{pmatrix}
		\nabla_{(\theta, m)}
		\cal{F}_t \right\rangle \label{eq:dFdt_t1}\\
		& - \left\langle \nabla_{(x,\qmo)}\delta_{q}\cal{F}_t, \begin{pmatrix}
			0 & -I_{d_x}\\
			I_{d_x} & \gamma_{x}I_{d_x}
		\end{pmatrix} \nabla_{(x,\qmo)}\delta_{q}\cal{F}_t \right\rangle_{q_t}.
		\label{eq:dFdt_t2}
	\end{align}
	Decomposing the matrix $\begin{pmatrix}
		0 & -I_{d_\theta}\\
		I_{d_\theta} & \gamma_{\theta}I_{d_\theta}
	\end{pmatrix}$ into symmetric and skew-symmetric components (write $D$, $Q$ respectively), i.e.
    \begin{align*}
        \begin{pmatrix}
		0 & -I_{d_\theta}\\
		I_{d_\theta} & \gamma_{\theta}I_{d_\theta}
	\end{pmatrix}
	= 	\underbrace{\begin{pmatrix}
			0 & -I_{d_\theta}\\
			I_{d_\theta} & 0
	\end{pmatrix}}_{=:Q}
	+
	\underbrace{\begin{pmatrix}
			0 & 0\\
			0 &  \gamma_{\theta}I_{d_\theta}
	\end{pmatrix}}_{=:D},
    \end{align*}
	we can simplify the RHS of \eqref{eq:dFdt_t1} to
    \begin{align*}
        \text{RHS } \eqref{eq:dFdt_t1} =- \left\langle \nabla_{(\theta, \tmo)} \cal{F}_t,D
	\nabla_{(\theta, \tmo)} \cal{F}_t \right\rangle = -\gamma_\theta \|\nabla_{\tmo}\cal{F}_t\|^2.
    \end{align*}
	Similarly, we can show that 
    \begin{align*}
        \eqref{eq:dFdt_t2} = -\gamma_x\|\nabla_{\qmo}\delta_{q}\cal{F}_t\|^2_{q_t}.
    \end{align*}
    As such, for $\gamma_x \ge 0$, $\gamma_\theta \ge 0$, the claim follows.
\end{proof}

\section{Proof of \Cref{prop:convergence_flow}}
\label{sec:convergence_flow}
The outline of the proof of \Cref{prop:convergence_flow} is:
\begin{itemize}
	\item \textbf{Step 1} (\Cref{sec:step_one_proof}): Explicitly computing an upper bound of the time derivative of $\cal{F}$ as a quadratic form.
	\item  \textbf{Step 2} (\Cref{sec:step_two_proof}): Under the conditions specified in \eqref{eq:flow_conditions}, we show that this time derivative is bounded above by another quadratic form that allows us to apply the log Sobolev inequality of \Cref{prop:logsobF}.
	\item \textbf{Step 3}: Using log Sobolev inequality,  \Cref{prop:logsobF}, and Gr\"{o}nwall's inequality, we obtain the desired result.
\end{itemize}
The remainder of the sections are dedicated to supporting the proof of Steps 1 and 2. This is done by developing the technical tools and carrying out explicit computations. The particular roles of the following sections are:
\begin{enumerate}
	\item \Cref{sec:first_variation}: Computing the time derivative of $\cal{L}$ for Step 1.
	\item \Cref{sec:commutator,sec:adjoint,sec:b}: Introducing the adjoint, commutator and another operator, as well as their explicit forms.
	\item \Cref{sec:dldqdt_t2}: Using the operators and explicit forms in \Cref{sec:commutator,sec:adjoint,sec:b}, we can upper bound terms introduced in \Cref{sec:first_variation} for Step 1. We utilize this in Step 1.
	\item \Cref{sec:xterm}. Bounding the cross terms (or interaction terms) between $\theta$ and $x$ that arises in the time derivative.
	\item \Cref{sec:psd_conditions}. Establishing sufficient conditions for a matrix with a particular form to be positive semi-definite given in \Cref{prop:psd_conds}. This is used in Step 2.
\end{enumerate}

Recall that the Lyapunov function is given by:
\begin{equation*}
	\mathcal{L} :=\cal{F}
	-\cal{E}^*
    + \left\Vert \nabla_{\left( \theta, m \right)} \mathcal{F} \right\Vert _{T_{\left( \theta, m \right)}}^{2} + \left\Vert \nabla_{\left( x, u \right)} \delta_{q} \mathcal{F} \right\Vert _{T_{\left( x, u \right)}}^{2}
\end{equation*}
where 
\begin{align}
T_{(\theta,\tmo)}
:= \frac{1}{K}
\begin{pmatrix}
    \tau_\theta I_{d_\theta} & \frac{\tau_{\theta\tmo}}{2}I_{d_\theta} \\
    \frac{\tau_{\theta\tmo}}{2}I_{d_\theta}  & \tau_\tmo I_{d_\theta}
\end{pmatrix},\quad 
T_{(x,\qmo)}
:= \frac{1}{K}
\begin{pmatrix}
    \tau_x I_{d_x} & \frac{\tau_{x\qmo}}{2}I_{d_x} \\
    \frac{\tau_{x\qmo}}{2}I_{d_x}  & \tau_\qmo I_{d_x}
\end{pmatrix}
\label{eq:t_def}
\end{align}
\begin{proof}[Proof of \Cref{prop:convergence_flow}]
	 The proof is completed in the following three steps:
	 
	 \textbf{Step 1.} In Section \ref{sec:step_one_proof}, we show that the time derivative of $\mathcal{L}_t := \cal{L}(z_t)$ satisfies the following upper bound:
	 	\begin{align*}
	 		\frac{\rm{d}}{\rm{d}t}\mathcal{L}_t \le& - \iprod{\nabla_{(\theta, \tmo)}\cal{F}_t}{S_{(\theta, \tmo)}  \nabla_{(\theta, \tmo)}\cal{F}_t}
	 		- \iprod{\nabla_{(x,\qmo)}  \delta_{q}\cal{F}_t}{S_{(x,\qmo)} \nabla_{(x,\qmo)}  \delta_{q}\cal{F}_t}_{q_t}.
	 	\end{align*}
	 	where $S_{(x,\qmo)}$ and $S_{(\theta, \tmo)}$ are suitable matrices defined in \eqref{eq:s_x}  and \eqref{eq:s_theta}, respectively.
	
	\textbf{Step 2.} Then, in Section \ref{sec:step_two_proof}, when \Cref{eq:flow_conditions} holds for some rate $\varphi > 0$, we have that
	\begin{align*}
		S_{(\theta, \tmo)} &\succeq \varphi C\left ( T_{(\theta, \tmo)} +  \frac{1}{2C}I_{d_\theta} \right), \\
		S_{(x,\qmo)} &\succeq \varphi C\left (T_{(x,\qmo)} +  \frac{1}{2C}I_{d_x} \right),
	\end{align*}
	where $C:=\min\{C_{\cal{E}},\eta_\theta,\eta_x\}$.
	
	\textbf{Step 3.} By Step $1$ and Step $2$, we obtain 
	\begin{align*}
		\frac{\rm{d}}{\rm{d}t}\mathcal{L}_t \le& - \varphi C \left (\frac{1}{2C}\norm{\nabla_{(\theta, m)}\cal{F}_t}^2 
		+  \left\Vert \nabla_{\left( \theta, m \right)} \mathcal{F}_t \right\Vert _{T_{\left( \theta, m \right)}}^{2} \right)\\ 
		&- \varphi C \left (  \frac{1}{2C}{\norm{\nabla_{(x,\qmo)}  \delta_{q} \cal{F}_t }^2_{q_t}} +  \left\Vert \nabla_{\left( x, u \right)} \delta_{q} \mathcal{F}_t \right\Vert _{T_{\left( x, u \right)}}^{2} \right ) .
	\end{align*}
	Comparing the above to~\eqref{eq:gradZnorm}, applying \Cref{prop:logsobF} and the \Cref{eq:lyapunov},
    %
    \begin{align*}
        \frac{\rm{d}}{\rm{d}t}\mathcal{L}_t \le - \varphi C \mathcal{L}_t,
    \end{align*}
	and application of Gronwall's inequality yields the final result
    \begin{align*}
        \mathcal{F}_t -\cal{E}^* \le \mathcal{L}_0\exp \left ( - \varphi C t \right ).
    \end{align*}
\end{proof}
\subsection{Proof of Step 1}
\label{sec:step_one_proof}
As shown in \Cref{prop:fv_lyapunov}, the time derivative of the Lyapunov function $\mathcal{L}$ is given by
\begin{align}
    \frac{\rm{d}}{\rm{d}t}\mathcal{L}_t =& - \iprod{\nabla_{(\theta,m)} \mathcal{L}_t} {
        \begin{pmatrix}
            0 & -I_{d_\theta}\\
            I_{d_\theta} & \gamma_{\theta}I_{d_\theta}
        \end{pmatrix}
    \nabla_{(\theta,m)}\cal{F}_t} \label{eq:dldthetadt} \\
    &- \iprod{\nabla_{(x,\qmo)} \delta_{q} \mathcal{L}_t} {
        \begin{pmatrix}
            0 & -I_{d_x}\\
            I_{d_x} & \gamma_{\theta}I_{d_x}
        \end{pmatrix}
    \nabla_{(x,\qmo)} \delta_{q} \cal{F}_t}_{q_t} \label{eq:dldqdt},
\end{align}
where we abbreviate $\mathcal{L}_t:=\mathcal{L}(\theta_t, m_t, q_t)$ and $\cal{F}_t:=\mathcal{F}(\theta_t, m_t, q_t)$.

We deal with the inner products in (\ref{eq:dldthetadt},\ref{eq:dldqdt}) one-by-one, starting with~\eqref{eq:dldthetadt}: we first compute the gradient of $\cal{L}$ w.r.t. to $(\theta, \tmo)$, finding that
%
\begin{align*}
    \nabla_{\left(\theta,m\right)}\cal{L} = \nabla_{(\theta, m)}\cal{F}
    + 2[\nabla_{(\theta, m)}^2 \cal{F}]\, T_{(\theta, \tmo)}\nabla_{(\theta, m)} \cal{F}
    + 2 \Ebbs{q}{[\nabla_{(\theta, m)}\nabla_{(x,\qmo)} \delta_q \cal{F}] \,  T_{(x,\qmo)} \nabla_{(x,\qmo)} \delta_q \cal{F}},
\end{align*}
where, as before, assume that the derivative-integral exchange can justified (see \Cref{sec:fv}).
It then follows that
\begin{align}
    \eqref{eq:dldthetadt} &= - \iprod{\nabla_{(\theta,m)} \cal{F}_t}{\tilde{\Gamma}_{\theta} \nabla_{(\theta, m)} \cal{F}_t}\label{eq:dldthetadt_t1} \\
    &-2 \iprod{\nabla_{(\theta,m)} \nabla_{(x,\qmo)}[\delta_{q} \cal{F}_t]\, T_{(x, \qmo)} \nabla_{(x,\qmo)} \delta_{q} \cal{F}_t}{
        \begin{pmatrix}
            0 & -I_{d_\theta}\\
            I_{d_\theta} & \gamma_{\theta} I_{d_\theta}
        \end{pmatrix}
    \nabla_{(\theta, m)} \cal{F}_t}_{q_t}.\label{eq:dldthetadt_t2}
\end{align}
where 
\begin{align*}
    \tilde{\Gamma}_\theta :=
        \begin{pmatrix}
            0 & -I_{d_\theta}\\
            I_{d_\theta} & \gamma_{\theta}I_{d_\theta}
        \end{pmatrix}
    + 2 T_{(\theta,\tmo)} \nabla_{(\theta,m)}^2 \cal{F}_t
        \begin{pmatrix}
                0 & -I_{d_\theta} \\
                I_{d_\theta} & \gamma_{\theta} I_{d_\theta}
        \end{pmatrix}.
\end{align*}
By $\cal{F}$'s definition,
\begin{align*}
    \nabla^2_{(\theta,\tmo)} \cal{F} =
        \begin{pmatrix}
            \nabla^2_\theta \cal{F} & 0_{d_\theta} \\
            0_{d_\theta} &\eta_\theta I_{d_\theta}
        \end{pmatrix};
\end{align*}
whence we see that
\begin{align*}
    \nabla_{(\theta,m)}^2 \cal{F} 
        \begin{pmatrix}
        	0_{d_\theta} & -I_{d_\theta}\\
    		I_{d_\theta} & \gamma_{\theta} I_{d_\theta}
        \end{pmatrix}
    &=
        \begin{pmatrix}
            0_{d_\theta} &-\nabla_\theta^2 \cal{F} \\
            \eta_\theta I_{d_\theta} &\gamma_\theta \eta_\theta I_{d_\theta}
        \end{pmatrix}\\
    \Rightarrow T_{(\theta,\tmo)} \nabla_{(\theta,m)}^2 \cal{F} 
        \begin{pmatrix}
            0_{d_\theta} & - I_{d_\theta}\\
            I_{d_\theta} & \gamma_{\theta} I_{d_\theta}
        \end{pmatrix}
    &= 
        \begin{pmatrix}
            \frac{\tau_{\theta\tmo}  \eta_\theta}{2} I_{d_\theta}
                &\frac{\tau_{\theta\tmo} \gamma_\theta \eta_\theta}{2} I_{d_\theta} 
                    -\tau_{\theta} \nabla_\theta^2  \cal{F} \\
            \tau_{\tmo}  \eta_\theta  I_{d_\theta}
                &\tau_{\tmo} \gamma_\theta\eta_\theta I_{d_\theta} - \frac{\tau_{\theta\tmo}}{2} \nabla_\theta^2 \cal{F}
        \end{pmatrix},\\
    \Rightarrow  \tilde{\Gamma}_\theta &=
        \begin{pmatrix}
            \tau_{\theta\tmo} \eta_\theta I_{d_\theta}
                &[\tau_{\theta\tmo} \gamma_\theta \eta_\theta - 1] I_{d_\theta} - 2 \tau_{\theta} \nabla_\theta^2 \cal{F}_t \\
            [2 \tau_{\tmo} \eta_\theta + 1] I_{d_\theta}
                &[2 \tau_{\tmo} \eta_\theta + 1] \gamma_\theta I_{d_\theta} -\tau_{\theta\tmo} \nabla_\theta^2 \cal{F}_t
        \end{pmatrix}.
\end{align*}
Given that, for any matrix $A$, we have $\iprod{\cdot}{A\cdot} = \iprod{\cdot}{A^{sym}\cdot}$ where $A^{sym} = \left (A + A^\top \right )/2$, then with $\Gamma_\theta :=\frac{\tilde{\Gamma}_\theta + \tilde{\Gamma}_\theta^\top}{2}$, we obtain that
\begin{equation}\label{eq:dldthetadt_t1_v2}
    \eqref{eq:dldthetadt_t1} =- \iprod{\nabla_{(\theta,m)} \cal{F}_t}{\Gamma_\theta \nabla_{(\theta, m)} \cal{F}_t},
\end{equation}
and hence that
\begin{equation}
    \Gamma_\theta = 
        \begin{pmatrix}
            \tau_{\theta\tmo}  \eta_\theta  I_{d_\theta}
                &\left[ \tau_{\tmo} + \frac{\tau_{\theta\tmo} \gamma_\theta}{2} \right]  \eta_\theta   I_{d_\theta} - \tau_{\theta}  \nabla_\theta^2 \cal{F}_t\\
            \left[ \tau_{\tmo} + \frac{\tau_{\theta\tmo} \gamma_\theta}{2} \right]  \eta_\theta  I_{d_\theta} - \tau_{\theta}  \nabla_\theta^2 \cal{F}_t
                &[ 2  \tau_{\tmo}  \eta_\theta + 1]  \gamma_\theta  I_{d_\theta} - \tau_{\theta\tmo}  \nabla_\theta^2 \cal{F}_t
        \end{pmatrix}.
\end{equation}

We now turn our attention to~\eqref{eq:dldqdt}.  Defining $\cal{H}_t(x, \qmo) := \sqrt{\frac{q_t(x, \qmo)}{\rho_{\theta_t, \eta_x}(x, \qmo)}}$ and \Cref{prop:fv_lyapunov}, we see that
\begin{align}
    \text{\eqref{eq:dldqdt}} =& - 4  \iprod{\nabla_{(x,\qmo)} \log \cal{H}_t}{
        \begin{pmatrix} 
            0_{d_x} & -I_{d_x} \\
            I_{d_x} & \gamma_{x}  I_{d_x}
        \end{pmatrix}
    \nabla_{(x,\qmo)} \log \cal{H}_t}_{q_t} \label{eq:dldqdt_t1} \\
    &- 8  \iprod{\nabla_{(x,\qmo)} \frac{\nabla_{(x,\qmo)} ^* T_{(x,\qmo)}  \nabla_{(x,\qmo)} \cal{H}_t} {\cal{H}_t}}{
        \begin{pmatrix}
    		0_{d_x} & -I_{d_x}\\
    		I_{d_x} & \gamma_{x}  I_{d_x}
        \end{pmatrix}
    \nabla_{(x,\qmo)} \log \cal{H}_t}_{q_t} \label{eq:dldqdt_t2} \\
    &+ 4  \iprod{ \nabla_{(x,\qmo)} \nabla_{(\theta,m)} \sbrac{\log {\rho_{\theta_t}}}  T_{(\theta,\tmo)}  \nabla_{(\theta,m)}\cal{F}_t}{
        \begin{pmatrix}
            0_{d_x} & -I_{d_x}\\
            I_{d_x} & \gamma_{x}  I_{d_x}
        \end{pmatrix}
    \nabla_{(x,\qmo)} \log \cal{H}_t}_{q_t}.\label{eq:dldqdt_t3}
\end{align}
Furthermore, we have
\begin{align*}
    \text{\eqref{eq:dldqdt_t1}} = -4  \gamma_x  \norm{\nabla_\qmo \log \cal{H}_t}^2_{q_t} = - \gamma_x  \norm{\nabla_{\qmo} \log\frac{q_t}{\rho_{\theta_t, \eta_x}} }^2_{q_t}.
\end{align*}
In \Cref{prop:simplifying_dldqdt_t2}, we show that \eqref{eq:dldqdt_t2} can be further simplified as follows:
\begin{align*}
    \text{\eqref{eq:dldqdt_t2}}  &= - 2  \gamma_x  \left\langle \nabla_{(x,\qmo)}\nabla_{\qmo} \left [\log \frac{q_t}{\rho_{\theta_t, \eta_x}}\right ], T_{(x,\qmo)}  \nabla_{(x,\qmo)} \nabla_{\qmo} \left[\log  \frac{q_t}{\rho_{\theta_t, \eta_x}}  \right] \right\rangle_{q_t} \\
    &- \left\langle \nabla_{(x,\qmo)} \log \frac{q_t}{\rho_{\theta_t, \eta_x}}, \tilde\Gamma_x  \nabla_{(x,\qmo)} \log \frac{q_t}{\rho_{\theta_t, \eta_x}} \right\rangle_{q_t},
\end{align*}
where
\begin{align*}
    \tilde\Gamma_x =
        \begin{pmatrix}
            \tau_{x \qmo}\eta_x I_{d_x}
                &\frac{\tau_{\qmo} +  \tau_{x \qmo}  \gamma_x}{2}  \eta_x  I_{d_x} 
                    + {\tau_{x}}  \nabla^2_x \ell \\
            \frac{\tau_{\qmo} + \tau_{x \qmo}  \gamma_x}{2}  \eta_x  I 
                + {\tau_{x}}  \nabla^2_x \ell
                &2  \gamma_x  \tau_{\qmo}  \eta_x  I_{d_x} + {\tau_{x \qmo}}  \nabla^2_x \ell
        \end{pmatrix}.
\end{align*}
We can then combine \eqref{eq:dldqdt_t1} and \eqref{eq:dldqdt_t2} to obtain 
\begin{align*}
    \text{\eqref{eq:dldqdt_t1}} + \text{ \eqref{eq:dldqdt_t2}} =& - 2  \gamma_x   \left\langle \nabla_{(x,\qmo)} \nabla_{\qmo} \sbrac{\log \frac{q_t}{\rho_{\theta_t, \eta_x}} }, T_{(x,\qmo)}  \nabla_{(x,\qmo)} \nabla_{\qmo} \sbrac{ \log \frac{q_t}{\rho_{\theta_t, \eta_x}} } \right\rangle_{q_t} \\
    &- \left\langle \nabla_{(x,\qmo)} \log \frac{q_t}{\rho_{\theta_t, \eta_x} }, \Gamma_x  \nabla_{(x,\qmo)} \log \frac{q_t}{\rho_{\theta_t, \eta_x} } \right\rangle_{q_t} .
\end{align*}
where $\Gamma_x = \tilde\Gamma_x + 
    \begin{pmatrix}
        0_{d_x} & 0_{d_x} \\
        0_{d_x} & \gamma_x I_{d_x}
    \end{pmatrix}$ is again positive semi-definite.

Given any p.s.d.~matrix $M \in \mathbb{R}^{d\times d}$ and a general matrix $A\in \mathbb{R}^{d\times d}$, it holds generically that 
\begin{align*}
    \langle A, MA\rangle = \iprod{A}{LL^\top A} = \iprod{L^\top A}{L^\top A} = \|L^\top A\|^2 \ge 0,
\end{align*}
where $M = L L^\top$ is Cholesky decomposition of $M$. Thus, we have the following upper bound:
\begin{align}
    \text{\eqref{eq:dldqdt_t1}} + \text{ \eqref{eq:dldqdt_t2}} &\le - \left\langle \nabla_{(x,\qmo)} \log \frac{q_t}{\rho_{\theta_t, \eta_x}}, \Gamma_{x} \nabla_{(x,\qmo)} \log \frac{q_t}{\rho_{\theta_t, \eta_x} } \right\rangle_{q_t} \\
    &= - \left \langle \nabla_{(x,u)} \delta_q \cal{F}_t, \Gamma_{x}  \nabla_{(x,u)} \delta_q \cal{F}_t \right \rangle_{q_t}. \label{eq:q_bound_summary}
\end{align}

\textbf{Cross Terms \eqref{eq:dldthetadt_t2} and \eqref{eq:dldqdt_t3}.} We will now deal with the cross terms of \eqref{eq:dldqdt_t3} and \eqref{eq:dldthetadt_t2}. In \Cref{prop:xterm_bound}, we show the following upper-bound: for all $\varepsilon > 0$, it holds that
\begin{align}
    \text{\eqref{eq:dldqdt_t3}} + \text{\eqref{eq:dldthetadt_t2}} &\le \varepsilon \left \langle \nabla_{(\theta, m)} \cal{F}_t, \Gamma_{\times}  \nabla_{(\theta, m)} \cal{F}_t \right\rangle
        + \frac{1}{\varepsilon} \left\langle \nabla_{(x,\qmo)} \delta_{q}\cal{F}_t, \nabla_{(x,\qmo)} \delta_{q} \cal{F}_t \right\rangle_{q_t}, \label{eq:xterm_bound_summary}
\end{align}
where $\Gamma_{\times} = K^2  
    \begin{pmatrix}
        \tau_\theta^2  I_{d_\theta} & \tau_\theta  \left ( \frac{\tau_{xu}+\tau_{\theta m }}{2}\right)  I_{d_\theta}  \\
        \tau_\theta  \left ( \frac{\tau_{xu}+\tau_{\theta m }}{2}\right) I_{d_\theta}  & \left( \left ( \frac{\tau_{xu}+\tau_{\theta m }}{2}\right)^2 + \tau_x^2 \right)  I_{d_\theta} 
    \end{pmatrix}$.

\textbf{Conclusion of Step 1.} Combining the results in (\ref{eq:dldthetadt_t1_v2}, \ref{eq:q_bound_summary}, \ref{eq:xterm_bound_summary}), we upper bound the time derivative of the Lyapunov function as
\begin{align*}
    \frac{\rm{d}}{\rm{d}t}\mathcal{L}_t \le&
    - \left \langle \nabla_{(\theta, m)}  \cal{F}_t, S_{(\theta, \tmo)}  \nabla_{(\theta, m)} \cal{F}_t \right \rangle 
    - \left \langle \nabla_{(x,u)} \delta_q \cal{F}_t, S_{(x,\qmo)}  \nabla_{(x,u)} \delta_q \cal{F}_t \right \rangle_{q_t} ,
\end{align*}
where 
\begin{align}
    S_{(x,\qmo)} &= \Gamma_{x} - \frac{1}{\varepsilon} I_{2d_x} =
    	\begin{pmatrix}
    		\left ( \tau_{x \qmo}\eta_x -\frac{1}{\varepsilon}\right ) I_{d_x} 
                &\frac{\tau_{\qmo} +  \tau_{x \qmo}  \gamma_x}{2}  \eta_x  I_{d_x} 
                    + {\tau_{x}}  \nabla^2_x \ell \\
    		\frac{\tau_{\qmo} +  \tau_{x \qmo}  \gamma_x}{2}  \eta_x  I_{d_x} 
                + {\tau_{x}}  \nabla^2_x \ell 
                &\left (\gamma_x  (2   \tau_{\qmo}  \eta_x + 1) - \frac{1}{\varepsilon}\right )  I_{d_x} + {\tau_{x \qmo}}  \nabla^2_x\ell
    	\end{pmatrix}, \label{eq:s_x} \\
	S_{(\theta, \tmo)} &= \Gamma_{\theta} - \varepsilon \Gamma_{\times} =
        \begin{pmatrix}
		  s_{\theta}  I_{d_\theta}
            &s_{\theta \tmo}  I_{d_\theta} - \tau_{\theta}  \nabla_\theta^2 \cal{F}_t \\
        s_{\theta \tmo}  I_{d_\theta} - \tau_{\theta}  \nabla_\theta^2 \cal{F}_t
            &s_{\tmo}  I_{d_\theta} - \tau_{\theta\tmo}  \nabla_\theta^2  \cal{F}_t \\
		\end{pmatrix}, \label{eq:s_theta}
\end{align}
with constants defined as
\begin{align*}
	s_{\theta} &:= \tau_{\theta\tmo} \eta_\theta 
        - \varepsilon  K^2  \tau_\theta^2, 
    \quad 
	s_{\theta \tmo} := \tau_{\tmo}  \eta_\theta
        + \frac{\tau_{\theta\tmo}  \gamma_\theta}{2}  \eta_\theta 
            - \varepsilon  K^2  \tau_\theta  \frac{\tau_{xu} + \tau_{\theta m }}{2}, \\
	s_{\tmo} &:= (2  \tau_{\tmo}  \eta_\theta + 1)  \gamma_\theta 
    - \varepsilon  K^2  \left( \left( \frac{\tau_{xu} + \tau_{\theta m }}{2} \right)^2 + \tau_x^2 \right ).
\end{align*}

\subsection{Proof of Step 2}
\label{sec:step_two_proof}

By the definitions of $T_{(\theta,\tmo)}$ and $T_{(x,\qmo)}$ in~\eqref{eq:t_def}, and those of $S_{(\theta,\tmo)}$ and $S_{(x,\qmo)}$ in~(\ref{eq:s_x}, \ref{eq:s_theta}),
\begin{align}
	S_{(\theta,\tmo)} - \varphi  C \left (T_{(\theta,\tmo)} + \frac{1}{2C}  I_{d_\theta} \right ) &=
	\begin{pmatrix}
		\alpha_\theta  I_{d_\theta}
            &\beta_\theta  I_{d_\theta} - {\tau_{\theta}}  \nabla_\theta^2 \cal{F} \\
		\beta_\theta  I_{d_\theta} - {\tau_\theta}  \nabla_\theta^2 \cal{F}
    		& \kappa_\theta  I_{d_\theta} - \tau_{\theta m}  \nabla_\theta^2 \cal{F}
	\end{pmatrix} \label{eq:show_psdx}, \\
	S_{(x, u)} - \varphi  C  \left (T_{(x,u)} + \frac{1}{2C}  I_{d_x} \right ) &=
	\begin{pmatrix}
		\alpha_x  I_{d_x}
            &\beta_x  I_{d_x} + {\tau_{x}}  \nabla_x^2\ell \\
        \beta_x  I_{d_x} + {\tau_x}  \nabla_x^2 \ell
    		& \kappa_x  I_{d_x} + {\tau_{x \qmo}}  \nabla_x^2 \ell
	\end{pmatrix}. \label{eq:show_psdtheta}
\end{align}
where
\begin{align*}
	\alpha_\theta &:= \tau_{\theta\tmo}  \eta_\theta 
        - \varepsilon  K^2  \tau_\theta^2 
            - \varphi  C  \left (\tau_{\theta} + \frac{1}{2C} \right ), \\
    \beta_\theta &:= \left (\tau_{\tmo} +\frac{\tau_{\theta\tmo}  \gamma_\theta}{2}\right )  \eta_\theta 
        - \varepsilon  K^2  \tau_\theta  \left ( \frac{\tau_{xu} + \tau_{\theta m }}{2}\right)
	       - \varphi  C   \frac{\tau_{\theta m} }{2},  \\
	\kappa_\theta &:= (2  \tau_{\tmo}  \eta_\theta + 1)  \gamma_\theta 
        - \varepsilon  K^2  \left ( \left ( \frac{\tau_{xu} + \tau_{\theta m }}{2}\right)^2 + \tau_x^2 \right ) 
            - \varphi  C  \left (\tau_{\tmo} + \frac{1}{2C} \right ), \\
	\alpha_x &= \tau_{x \qmo}  \eta_x - \frac{1}{\varepsilon} 
        - \varphi  C  \left (\tau_x + \frac{1}{2C} \right ) ,\\
	\beta_x &= \frac{(\tau_{\qmo} +  \tau_{x \qmo}  \gamma_x)  \eta_x}{2} 
        - \varphi  C  \frac{\tau_{x \qmo}}{2} ,\\
	\kappa_x &= \gamma_x  (2  \tau_{\qmo}  \eta_x + 1) - \frac{1}{\varepsilon}  
        - \varphi  C  \left (\tau_{\qmo} + \frac{1}{2C} \right ).
\end{align*}

\begin{proposition}
    Defining $\left( \alpha_\theta, \beta_\theta, \kappa_\theta, \alpha_x, \beta_x, \kappa_x \right)$ as above, assume that $K \ge 2C_{\cal{E}}$ if
    \begin{subequations}
    	\begin{align}
    		\tau_{\theta m} K - \alpha_\theta - \kappa_\theta &\le 0,  \\
    		\tau_\theta^2 K^2 + (\tau_{\theta m }\alpha_\theta - 2\beta_\theta \tau_\theta) K - \alpha_\theta \kappa_\theta + \beta^2_\theta &\le 0, \\
    		\tau_\theta^2 K^2 - (\tau_{\theta m }\alpha_\theta- 2\beta_\theta \tau_\theta) K - \alpha_\theta \kappa_\theta + \beta^2_\theta &\le 0, \\
    		{\tau_{xu}}K - \alpha_x - \kappa_x &\le 0, \\
    		{\tau_x^2}K^2 - \left ({\tau_{x \qmo}}\alpha_x - 2\beta_x {\tau_x} \right ) K - \alpha_x \kappa_x + \beta^2_x &\le 0, \\
    		{\tau_x^2}K^2 + \left ({\tau_{x \qmo}}\alpha_x - 2\beta_x {\tau_x} \right ) K - \alpha_x \kappa_x + \beta^2_x &\le 0,
    	\end{align}
    	\label{eq:flow_conditions}
    \end{subequations}
    hold for some rate $\varphi > 0$. The following lower bounds for $S_{(x,\qmo)}$ and $S_{(\theta,\tmo)}$ then hold:
	$$
	S_{(x,\qmo)} \succeq \varphi C \left (T_{(x,\qmo)} + \frac{1}{2C}I_{d_x} \right ), \quad S_{(\theta,\tmo)} \succeq \varphi C \left (T_{(\theta,\tmo)} + \frac{1}{2C}I_{d_\theta} \right ),
	$$
	where 
	$C:=\min\{C_{\cal{E}},\eta_\theta,\eta_x\}$.
\end{proposition}
\begin{proof}
Note the matrices in (\ref{eq:show_psdx}, \ref{eq:show_psdtheta}) have the same form as that in \eqref{eq:psd_cond_form}. As we show below,
\begin{equation}\label{eq:fdsu98f0nau8ifai}
    -KI_{d_x} \preceq \nabla_x^2 \ell \preceq KI_{d_x},\quad-KI_{d_\theta} \preceq  \nabla^2_\theta \cal{F} \preceq KI_{d_\theta}.
\end{equation}
Then, applying \Cref{prop:psd_conds} tells us that (\ref{eq:show_psdx}, \ref{eq:show_psdtheta}) are positive semidefinite if the conditions \eqref{eq:flow_conditions} are satisfied.

Now, to prove the inequalities in~\eqref{eq:fdsu98f0nau8ifai}. \Cref{ass:gradlip} and \citep[Lemma 1.2.2]{nesterov2003introductory} imply the first two inequalities and
$$-KI_{d_\theta}\preceq \nabla^2_\theta \ell\preceq KI_{d_\theta}.$$
The other two inequalities in~\eqref{eq:fdsu98f0nau8ifai} follow from the above. To see this, for $\nabla^2_\theta \cal{F} \preceq KI_{d_\theta}$, we have 
$$
\iprod{v}{(KI_{d_\theta} - \nabla^2_\theta \cal{F}) v} = \int \iprod{v}{(KI_{d_\theta} + \nabla^2_\theta \ell) v}q(\rm{d}x\times \rm{d}u),
$$
and since $\iprod{v}{\nabla^2_\theta [\ell] v} \ge -K \|v\|^2$, we have shown that $\iprod{v}{(KI_{d_\theta} - \nabla^2_\theta \cal{F}) v} \ge 0$ implying the desired result of $\nabla^2_\theta \cal{F} \preceq KI_{d_\theta}$. A similar argument can be made for $-KI_{d_\theta} \preceq  \nabla^2_\theta \cal{F}$.
\end{proof}
\subsection{Examples of \eqref{eq:flow_conditions} holding}
\label{sec:eq_flow_conditions true}
We verified the above with a symbolic calculator (see \url{https://github.com/jenninglim/mpgd}) written in Mathematica for the choices of $\eta_x, \eta_\theta, \gamma_x, \gamma_\theta$ hold for the following choices:
\begin{enumerate}
    \item Rate $\varphi=\frac{1}{10}$, momentum parameters $\gamma_x = \frac{3}{2}$, $\gamma_\theta = 3$, $\eta_x = 2K$, $\eta_\theta =2 K$, and elements of the Lyapunov function
    $\tau_\theta=\frac{3}{8}, \tau_{\theta\tmo}=\frac{5}{4},\tau_\tmo = 2,\tau_x=\frac{1}{4} ,\tau_{x\qmo} = \frac{5}{4},\tau_\qmo=\frac{7}{4}$, and $\varepsilon = 15$.
    \item Rate $\varphi=\frac{1}{30}$, momentum parameters $\gamma_x = \frac{6}{5}$, $\gamma_\theta = \frac{5}{2}$, $\eta_x = 2K$, $\eta_\theta =2 K$, and elements of the Lyapunov function
    $\tau_\theta=\frac{3}{8}, \tau_{\theta\tmo}=\frac{5}{4},\tau_\tmo = 2,\tau_x=\frac{1}{4} ,\tau_{x\qmo} = 1,\tau_\qmo=\frac{3}{2}$, and $\varepsilon = 15$.
\end{enumerate}
\subsection{Supporting Proofs and Derivations}
\subsubsection{The derivative of $\cal{L}_t$ along the flow}\label{sec:first_variation}
Here, we prove~(\ref{eq:dldthetadt},\ref{eq:dldqdt}).
\begin{proposition}
	The derivative of the Lyapunov function $\cal{L}$ is given by
	\begin{align*}
		\frac{\rm{d}}{\rm{d}t}\mathcal{L}_t
		=& - \iprod{\nabla_{(\theta,m)}\mathcal{L}_t}{\begin{pmatrix}0 & -I_{d_\theta}\\
				I_{d_\theta} & \gamma_{\theta}I_{d_\theta}
			\end{pmatrix}\nabla_{(\theta,m)}\cal{F}_t} -\iprod{\nabla_{(x,\qmo)}\delta_{q}\mathcal{L}_t}{\begin{pmatrix}0 & -I_{d_x}\\
				I_{d_x} & \gamma_{\theta}I_{d_x}
			\end{pmatrix}\nabla_{(x,\qmo)}\delta_{q}\cal{F}_t}_{q_t},
	\end{align*}
	where the first variation is given by
	\begin{align*}
		\delta_{q}\mathcal{L}_t =&\delta_q \cal{F}_t +\delta_q \cal{L}^1_t + \delta_q \cal{L}^2_t \\
		=& \log \frac{q_t}{\rho_{\theta_t, \eta_x}} + 1
		-2\iprod{ \nabla_{(\theta,\tmo)} \log {\rho_{\theta_t}}
		}{T_{(\theta,\tmo)} \nabla_{(\theta,m)}\cal{F}_t} \\
		&+ \frac{4}{\cal{H}_t}  \nabla_{(x,\qmo)} ^* T_{(x,\qmo)}\nabla_{(x,\qmo)} \cal{H}_t,
	\end{align*}
	with $\nabla^*$ and $\cal{H}_t := \cal{H}_{\theta_t, q_t}$ defined in \Cref{prop:fv_L2}.
	\label{prop:fv_lyapunov}
\end{proposition}
\begin{proof}
	From the linearity of $\cal{L}$, we have
	$$
	\dot{\mathcal{L}}_t= \dot{\cal{F}}_t + \dot{\cal{L}^1_t} + \dot{\cal{L}^2_t},
	$$
	where 
	\begin{align*}
		\mathcal{L}_t^1 &:=\iprod{\nabla_{(\theta,m)} \cal{F}_t}{T_{(\theta,m)}\nabla_{(\theta,m)} \cal{F}_t},\\
		\mathcal{L}_t^2&:={\iprod{\nabla_{(x,\qmo)} \delta_{q} \cal{F}_t}{T_{(x,u)}\nabla_{(x,\qmo)} \delta_{q} \cal{F}_t}}_{q_t}.
	\end{align*}
	From \Cref{sec:gc_on_product_space}, it can be seen that the time derivatives of $\cal{F}_t$, $\cal{L}^1_t$, and $\cal{L}^2_t$, are given by
	\begin{align*}
		\dot{\cal{F}_t} &= \iprod{\nabla_{(\theta, \tmo)} \cal{F}_t}{{(\dot\theta_t, \dot\tmo_t)}} + \iprod{\nabla_{(x,\qmo)} \delta_q \cal{F}_t}{(\dot{x}_t, \dot\qmo_t)}_{q_t}, \\
		\dot{\cal{L}^1_t} &= \iprod{\nabla_{(\theta, \tmo)} \cal{L}^1_t}{{(\dot\theta_t, \dot\tmo_t)}} + \iprod{\nabla_{(x,\qmo)} \delta_q \cal{L}^1_t}{(\dot{x}_t, \dot\qmo_t)}_{q_t},\\
		\dot{\cal{L}^2_t} &=\iprod{\nabla_{(\theta, \tmo)} \cal{L}^2_t}{{(\dot\theta_t, \dot\tmo_t)}} + \iprod{\nabla_{(x,\qmo)} \delta_q \cal{L}^2_t}{(\dot{x}_t, \dot\qmo_t)}_{q_t},
	\end{align*}
	where ${(\dot\theta_t, \dot\tmo_t)} = \begin{pmatrix}0 & -I_{d_\theta}\\
		I_{d_\theta} & \gamma_{\theta}I_{d_\theta}
	\end{pmatrix}\nabla_{(\theta,m)}\cal{F}_t$ and $(\dot{x}_t, \dot\qmo_t) = \begin{pmatrix}0 & -I_{d_x}\\
		I_{d_x} & \gamma_{x}I_{d_x}
	\end{pmatrix}\nabla_{(x,\qmo)}\delta_{q}\cal{F}_t$.
	From \Cref{prop:fv}, \Cref{prop:fv_L1}, and \Cref{prop:fv_L2}, and linearity of the inner product, we obtain as desired.
\end{proof}
\begin{proposition}
	\label{prop:fv}
	The first variation of $\cal{F}$ is given by
	$$
	\delta_q \cal{F} [z](x, \qmo) = \log \frac{q(x,\qmo)}{\rho_{\theta, \eta_x}(x, \qmo)} + 1,
	$$
	where $z= (\theta, \tmo, q)$.
\end{proposition}
\begin{proof}
	It can be seen that $\cal{F}$ is \eqref{eq:vi_type1}-typed equation with $f((\theta, \tmo), (x, \qmo), q, g) = q \log \frac{q}{\rho_{\theta, \eta_x} (x)}$. Thus, using the formula \eqref{eq:fv_t1}, we obtain as its first variation
	$$
	\delta_q \cal{F} [z](x, \qmo) = \log \frac{q(x,\qmo)}{\rho_{\theta, \eta_x}(x, \qmo)} + 1,
	$$
	Hence, we have as desired.
\end{proof}
For the first variation of $\mathcal{L}^1$, we have the following result:
\begin{proposition}
	\label{prop:fv_L1}
	The first variation of  $\mathcal{L}^1$ is given by
	$$
	\delta_q \mathcal{L}^1[z](x, \qmo) = -2\iprod{ \nabla_{(\theta,\tmo)} \log {\rho_{\theta}(x)}
	}{T_{(\theta,\tmo)} \nabla_{(\theta,m)}\cal{F}[z]},
	$$
 where $z= (\theta, \tmo, q)$.
\end{proposition}
\begin{proof}
	It can be seen that $\mathcal{L}^1$ is a variational integral of type \eqref{eq:vi_type2} with
	\begin{align*}
		f((\theta, \tmo), (x, \qmo), (x', \qmo'), q, q') =& qq' \iprod{ \nabla_\theta \log {\rho_\theta(x)}}{T_{\theta \theta} \nabla_\theta \log{\rho_\theta (x')} } \\
		&- q\iprod{\nabla_\theta \log {\rho_\theta (x)}}{T_{\theta \tmo}\tmo} \\
		&- q'\iprod{\tmo}{T_{\tmo \theta }\nabla_\theta \log {\rho_\theta (x')}} \\
		&+ \iprod{\tmo}{T_{\tmo \tmo }\tmo},
	\end{align*}
	where we write $T_{\theta \theta}$ for the upper left block of $T_{(\theta,\tmo)}$, and similarly for $T_{\theta \tmo}$, and $T_{\tmo \tmo}$.
	
	Using the formula \eqref{eq:fv_t2} and the fact that
	\begin{align*}
		\nabla_{(4)} f((\theta, \tmo), (x, \qmo), (x', \qmo'), q, q')
  =& q' \iprod{ \nabla_\theta \log {\rho_\theta(x)}}{T_{\theta\theta} \nabla_\theta \log{\rho_\theta(x')} } - \iprod{\nabla_\theta \log {\rho_\theta(x)}}{T_{\theta \tmo}\tmo},
	\end{align*}
	For the first term of \eqref{eq:fv_t2}, we have 
	\begin{align*}
		\int \nabla_{(4)} f((\theta, \tmo), (x, \qmo), (x', \qmo'), q(x,\qmo), q(x',\qmo')) \rm{d}x' \rm{d}\qmo' \\
  =- \iprod{\nabla_{(\theta,\tmo)} \log {\rho_\theta(x)}}{T_{(\theta, \tmo)}\nabla_{(\theta, \tmo)} \cal{F}[z]},
	\end{align*}
	and by symmetry, it follows similarly for the last term of \eqref{eq:fv_t2}. We have obtained as desired.
\end{proof}

\begin{proposition}
	\label{prop:fv_L2}
	The first variation of $\mathcal{L}^2$ is given by
    \begin{equation*}
        \delta_{q}\mathcal{L}^2[z] = \frac{4}{\cal{H}_{\theta, q}}  \nabla_{(x,\qmo)} ^* T_{(x,\qmo)}\nabla_{(x,\qmo)} \cal{H}_{\theta, q},
    \end{equation*}
	where $\cal{H}_{\theta, q}(x, \qmo) := \sqrt{\frac{q(x,\qmo)}{\rho_{\theta, \eta_x}(x,\qmo)}}$ and the adjoint operator $\nabla_{(x,\qmo)} ^*$ (see \Cref{sec:adjoint}) is defined as
	$$\nabla_{(x,\qmo)} ^*(f)(x,\qmo):= - \iprod{\nabla_{(x,\qmo)}  \log \rho_{\theta, \eta_x}(x,\qmo)}{ f(x,\qmo)} - \nabla_{(x,\qmo)} \cdot [ f](x,\qmo),
	$$
	where $\nabla_{x}\cdot$ is the divergence operator w.r.t. $x$.
\end{proposition}
\begin{proof}
	First note that $\cal{L}^2$ can be written equivalently as
	$$
	\iprod{\frac{\nabla_{(x,\qmo)} q}{q}-\nabla_{(x,\qmo)}\log \rho_{\theta, \eta_x}}{T_{(x,u)}\left [\frac{\nabla_{(x,\qmo)} q}{q}-\nabla_{(x,\qmo)}\log \rho_{\theta, \eta_x} \right ]}_q.
	$$
	It can be seen that $\cal{L}^2$ is a variational integral of type \eqref{eq:vi_type1} with
	$$
	f((\theta, \tmo), (x, \qmo), q, g) = q \iprod{\frac{g}{q} - \nabla_{(x,\qmo)} \log \rho_{\theta, \eta_x}(x,\qmo)}{T_{(x,u)}\left [\frac{g}{q} - \nabla_{(x,\qmo)} \log \rho_{\theta, \eta_x}(x,\qmo) \right ]}.
	$$
As
	\begin{align*}
		\nabla_{(3)}f((\theta, \tmo), (x, \qmo), q, g) =& \iprod{\frac{g}{q} - \nabla_{(x, \qmo)} \log \rho_{\theta, \eta_x}(x,\qmo)}{T_{(x,u)} \left [\frac{g}{q} - \nabla_{(x,\qmo)} \log \rho_{\theta, \eta_x}(x,\qmo) \right ]} \\
		&-2q\iprod{\frac{g}{q^2}}{T_{(x,u)}\left [\frac{g}{q} - \nabla_{(x,\qmo)} \log \rho_{\theta, \eta_x}(x,\qmo) \right ]} \\
		=& - \iprod{\frac{g}{q} + \nabla_{(x,\qmo)} \log \rho_{\theta, \eta_x}(x,\qmo)}{T_{(x,u)}\left [\frac{g}{q} - \nabla_{(x,\qmo)} \log \rho_{\theta, \eta_x}(x,\qmo)\right ]},
	\end{align*}
	and 
	\begin{align*}
		\nabla_{(4)}f((\nabla, \tmo), (x, \qmo), q, g) =& 2T_{(x,u)}\left [\frac{g}{q} - \nabla_{(x,\qmo)} \log \rho_{\theta, \eta_x}(x,\qmo) \right],
	\end{align*}
	using formula \eqref{eq:fv_t1}, we obtain
	\begin{subequations}
		\begin{align}
			\delta_q \cal{L}^2[z](x, \qmo) =&- \iprod{\nabla_{(x,\qmo)} \log {q(x,\qmo)}+ \nabla_{(x,\qmo)} \log {\rho_{\theta, \eta_x}(x,\qmo)}}{T_{(x,u)}\nabla_{(x,\qmo)} \log \frac{q(x,\qmo)}{\rho_{\theta, \eta_x}(x,\qmo)}} \label{eq:l2_fv_t1}\\
			& - \nabla_{(x,\qmo)} \cdot \left ( 2T_{(x,u)} \nabla_{(x,\qmo)} \log\frac{q(x,\qmo)}{\rho_{\theta, \eta_x}(x,\qmo)} \right ). \label{eq:l2_fv_t2}
		\end{align}
	\end{subequations}
	To obtain the desired result, note that
	we can rewrite 
	\begin{align*}
		\eqref{eq:l2_fv_t1} &= -\iprod{\nabla_{(x,u)} \log \sqrt{q} + \nabla_{(x,u)} \log \sqrt{\rho_{\theta, \eta_x}}}{\frac{4}{\cal{H}_{\theta,q}}T_{(x,u)}\nabla_{(x,u)} \log \cal{H}_{\theta,q}} \\
		\eqref{eq:l2_fv_t2} &= - \nabla_{(x,\qmo)} \cdot \left (\frac{4}{\cal{H}_{\theta,q}}T_{(x,\qmo)} \nabla_{(x,\qmo)}\cal{H}_{\theta,q} \right ), \\
		&= - \iprod{\nabla_{(x,\qmo)} \frac{4}{\cal{H}_{\theta,q}}}{T_{(x,\qmo)} \nabla_{(x,\qmo)} \cal{H}_{\theta,q}} - \frac{4}{\cal{H}_{\theta,q}}\nabla_{(x,\qmo)} \cdot (T_{(x,\qmo)}\nabla_{(x,\qmo)}\cal{H}_{\theta,q}) \\
		&= \iprod{\nabla_{(x,\qmo)} \log \cal{H}_{\theta,q}}{\frac{4}{\cal{H}_{\theta,q}}T_{(x,\qmo)} \nabla_{(x,\qmo)} \cal{H}_{\theta,q}} - \frac{4}{\cal{H}_{\theta,q}}\nabla_{(x,\qmo)} \cdot (T_{(x,\qmo)}\nabla_{(x,\qmo)}\cal{H}_{\theta,q}).
	\end{align*}
	Hence, we have obtained the desired result.
\end{proof}

\subsubsection{The Adjoint of $\nabla$}\label{sec:adjoint}
Throughout the proofs, we use $\nabla^*$ to denote the linear map that is an adjoint of $\nabla$ (the Euclidean gradient operator). 

In the view of $\nabla$ as a linear map from $(L^2_{\rho}(\cal{X};\r),\iprod{\cdot}{\cdot}_{\rho})$ to $(L^2_{\rho}(\cal{X};\cal{X}),\iprod{\cdot}{\cdot}_{\rho})$
where $\cal{X} \in \{\r^{d_x}, \r^{d_x}\times \r^{d_x}\}$, $\rho\in \cal{P}(\cal{X})$,
and
\begin{align*}
	L^2_{\rho}(\cal{X};\cal{X}):=\{\text{sufficiently regular }v:\cal{X}\to \cal{X}\text{ such that }\norm{v}_{\rho}<\infty\}.
\end{align*}
Then its adjoint $\nabla^*$ is given by
$$
\nabla ^*v:=-\frac{1}{\rho}\nabla \cdot[\rho v] = - \iprod{ \nabla \log \rho }{v} - \nabla \cdot  v.
$$
where $\nabla \cdot$ denotes the divergence operator. Whatever the gradient $\nabla$ is taken with respect to,  the adjoint $\nabla^*$ is taken with respect to the corresponding quantity.

We show the adjoint property for the case where $\cal{X} = \r^{d_x} \times \r^{d_x}$ as the other case follows similarly. To begin.
\begin{align*}
\iprod{\nabla f}{v}_{\rho}&=\int\iprod{\nabla f(x,u)}{v(x,u)}\rho(\rm{d}x\times du)\\
&=\int\iprod{\nabla f(x,u)}{\rho(x,u)v(x,u)}\, \rm{d}x \times du\\
&=-\int f(x,u) \nabla \cdot [\rho (x,u)v(x,u)]\, \rm{d}x\times du=\iprod{f}{\nabla^*v}_{\rho},
\end{align*}
where we used integration by parts combined with the divergence theorem and the fact that $\rho$ must vanish at the boundary to obtain the second equation.

Another case of interest is the view of $\nabla$ as a linear map from $(L^2_{\rho}(\cal{X};\cal{X}'),\iprod{\cdot}{\cdot}_{\rho})$ to $(L^2_{\rho}(\cal{X};\cal{X} \times \cal{X}'),\iprod{\cdot}{\cdot}_{\rho})$, where the underlying inner product on the latter space is induced by the Frobenius inner product, and as before, $\cal{X}, \cal{X}' \in \{\r^{d_x}, \r^{d_x}\times \r^{d_x}\}$, $\rho\in \cal{P}(\cal{X})$.

The adjoint of $\nabla$ for some $M\in L^2(\cal{X}, \cal{X} \times \cal{X}')$ is defined as
\begin{align*}
(\nabla^* M) &:= -M^\top \nabla \log \rho - \nabla \cdot M \in \cal{X'},
 \end{align*}
where $\nabla \cdot M$ denotes the divergence operator for a matrix field defined as $(\nabla \cdot M)_i = \partial_j M_{ji}$ in Einstein's notation. Each element of $(\nabla^* M)_i $ is given by
\begin{align*}
(\nabla^* M)_i &:= -\iprod{\nabla \log \rho}{M^\top e_i} -  \nabla \cdot (M^\top e_i),
\end{align*}
where $\{e_i\}_i$ is the standard basis of $\cal{X}'$.

For the case $\cal{X}=\r^{d_x} \times \r^{d_x}$, the adjoint $\nabla^*$ is property can be shown as follows: let $M\in L^2(\cal{X}, \cal{X} \times \cal{X}')$ and $v \in L^2(\cal{X},\cal{X}')$
\begin{align*}
	\iprod{\nabla v}{M}_\rho
	&= \int \sum_{i=1}^{d_\cal{X}} \sum_{j=1}^{d_\cal{X'}}(\nabla v)_{ij} M_{ij} \, \rho(\rm{d}x\times \rm{d}\qmo)\\
	&= \int \sum_{i=1}^{d_\cal{X}} \sum_{j=1}^{d_{\cal{X}'}} \partial_{i} [v_j] \, M_{ij} \, \rho(\rm{d}x\times \rm{d}\qmo) = \int \sum_{j=1}^{d_{\cal{X}'}} \iprod{\nabla v_j}{M e_j}  \, \rho(\rm{d}x\times \rm{d}\qmo)\\
	&=-\int \sum_{j=1}^{d_{\cal{X}'}} (v_j) \nabla \cdot (\rho M e_j)\,  \rm{d}x \times \rm{d}\qmo \\
	&=-\int \sum_{j=1}^{d_{\cal{X}'}} (v_j) [ \iprod{\nabla \log \rho}{M e_j} +  \nabla \cdot (M e_j)] \, \rho(\rm{d}x \times \rm{d}\qmo) \\
	&= \iprod{\nabla^* M}{v}_\rho.
\end{align*}
where, for the second line, we apply integration by parts combined with the divergence theorem and the fact that $\rho$ vanishes at the boundary. 
 \subsubsection{Commutator $[\cdot , \cdot]$}
  \label{sec:commutator}
Another quantity widely used in the proof is the commutator denoted by $[\cdot, \cdot]$. It can be thought of as an indicator of whether two operators commute. It is defined as
\begin{equation*}
    [A,B] f = (AB)f - (BA)f.
\end{equation*}
If $[A,B]f=0$, then $A$ and $B$ is said to commute.

In this section, we list propositions and their proofs that will be useful for the proof of \Cref{prop:convergence_flow}.
\begin{proposition}
\label{prop:commute1}
Let $f : \mathbb{R}^{d_x} \times \r^{d_x}\rightarrow \mathbb{R}$ be sufficiently regular, then we have
$$
[\nabla_\qmo, \nabla^*_\qmo] \nabla_\qmo f= \eta_x \nabla_\qmo f.
$$
\end{proposition}
\begin{proof} Expanding out the commutator, we obtain
\begin{align*}
    [\nabla_\qmo, \nabla^*_\qmo] \nabla_\qmo f &= \nabla_\qmo \nabla^*_\qmo \nabla_\qmo f - \nabla^*_\qmo \nabla_\qmo \nabla_\qmo f\\
    &- \nabla_\qmo \iprod{\nabla_\qmo \log \rho_{\theta, \eta_x}}{\nabla_\qmo f} - \nabla_\qmo [\nabla_\qmo \cdot \nabla_\qmo f] + \nabla^2_\qmo f\, \nabla_\qmo \log \rho_{\theta, \eta_x} + \nabla_\qmo \cdot (\nabla_\qmo^2 f),
\end{align*}
and since $(\nabla_\qmo [\nabla_\qmo \cdot \nabla_\qmo f])_i = \partial_{u_i} \partial_{u_j}\partial_{u_j} f = \partial_{\qmo_j} \partial_{\qmo_j} \partial_{u_i} f = (\nabla_\qmo \cdot (\nabla_\qmo^2 f))_i$. We have that 
$$
[\nabla_\qmo, \nabla^*_\qmo] \nabla_\qmo f = - \nabla_u^2 \log r_{\eta_x} \, \nabla_\qmo f = \eta_x \nabla_\qmo f.
$$
\end{proof}

\begin{proposition}
	\label{prop:commute2}
Let $f: \mathbb{R}^{d_x} \times \r^{d_x} \rightarrow \mathbb{R}$ be sufficiently regularly,
we have that $\nabla_u$ and $\nabla^*_x \nabla_x$ commutes. In other words, we have
$$
[\nabla_u, \nabla^*_x \nabla_x]  f = 0.
$$
\end{proposition}
\begin{proof} We begin with the definition
\begin{align*}
    [\nabla_u, \nabla^*_x \nabla_x] f =& \nabla_u \nabla^*_x \nabla_x f - \nabla^*_x \nabla_x \nabla_u f,
\end{align*}
For the first term, we have
\begin{align*}
    \nabla_u \nabla^*_x \nabla_x f &= - \nabla_u [\iprod{\nabla_x \log \rho }{\nabla_x f} + \nabla_x \cdot {\nabla_x f} ] \\
    &= - \nabla_u \nabla_x f \, \nabla_x \log \rho - \nabla_\qmo [\nabla_x \cdot {\nabla_x f}].
\end{align*}
As for the second term, we have
\begin{align*}
    \nabla^*_x \nabla_x \nabla_u f = -  \nabla_u \nabla_x f \nabla_x \log \rho  - \nabla_x \cdot (\nabla_x \nabla_u f ).
\end{align*}
Equality or $ [\nabla_u, \nabla^*_x \nabla_x] f= 0$ follows from the following fact:
	\begin{align*}
		(\nabla_x \cdot (\nabla_x \nabla_\qmo f))_i &= \partial_{x_j} \partial_{x_j} \partial_{u_i} f =\partial_{u_i} \partial_{x_j} \partial_{x_j} f = (\nabla_u \nabla_x \cdot (\nabla_x f))_i.
	\end{align*}
\end{proof}

\begin{proposition}
	\label{prop:commute3}
For $f:\mathbb{R}^{d_x} \times \r^{d_x} \rightarrow \mathbb{R}$ be sufficiently regular, we have that $\nabla_x$ and $\nabla_\qmo^* \nabla_\qmo$ commutes, i.e., we have
\begin{align*}
	[\nabla_x, \nabla_\qmo^* \nabla_\qmo]f = 0.
\end{align*}
\end{proposition}
\begin{proof}
	Expanding out the commutator, we shall show that
	$$
	[\nabla_x, \nabla_u^* \nabla_u]f = \nabla_x  \nabla_u^* \nabla_u f  - \nabla_u^* \nabla_u \nabla_x f = 0.
	$$
	This follows since the first term can be written as
	\begin{align*}
		\nabla_x  \nabla_u^* \nabla_u f = \nabla_x \iprod{\nabla_u \log r_{\eta_x} }{\nabla_u f} + \nabla_x[ \nabla_u \cdot (\nabla_u f)] \\
		= \nabla_x {\nabla_u [f]} \, \nabla_u \log r_{\eta_x}   + \nabla_x [\nabla_u \cdot (\nabla_u f)],
	\end{align*}
	and, for the second term, we have
	\begin{align*}
		\nabla_u^* \nabla_u \nabla_x f =\nabla_u \nabla_x [f]\, \nabla_u \log r_{\eta_x}  + \nabla_u \cdot (\nabla_u \nabla_x f).
	\end{align*}
	Note that $\nabla_u \cdot (\nabla_u \nabla_x f) = \nabla_x [\nabla_u \cdot (\nabla_u f)]$, which can be seen from
	\begin{align*}
		(\nabla_x \nabla_u \cdot (\nabla_u f))_i = \partial_{x_i} \partial_{u_j} \partial_{u_j} f = \partial_{u_j} \partial_{u_j} \partial_{x_i} f = 	(\nabla_u \cdot (\nabla_u \nabla_x f))_i.
	\end{align*}
	Hence, we have as desired.
\end{proof}

\subsubsection{Another operator $B$}
\label{sec:b}
To simplify the notation, we drop the subscripts and write $\rho := \rho_{\theta, \eta_x}$. Let $f: \mathbb{R}^{d_x} \times \mathbb{R}^{d_x} \rightarrow \mathbb{R}$, another operator (denoted by $B$) that we are interested in is defined as follow:
\begin{align*}
	B[f] &: \mathbb{R}^{d_x} \times \mathbb{R}^{d_x} \rightarrow \r :=  \nabla_{(x,\qmo)}^* \left [
	Q \nabla_{(x,\qmo)}f \right ]
\end{align*}
where $Q:= \begin{pmatrix} 0 & - I_{d_x} \\
I_{d_x} & 0
\end{pmatrix}$.
By expanding out the adjoint, $B$ can be written equivalently as
\begin{align}
	B[f] =& -  \iprod{\nabla_{(x,\qmo)} \log \rho }{Q \nabla_{(x,\qmo)} f}
	- \nabla_{(x,\qmo)} \cdot \left (Q \nabla_{(x,\qmo)} f\right ) \nonumber \\
 =& - \iprod{\nabla_x f}{\nabla_\qmo \log \rho} + \iprod{\nabla_x \log \rho}{\nabla_\qmo f}. \label{eq:Bequalone}
\end{align}
One can generalize the operator $B$ to vector-valued functions $f: \mathbb{R}^{d_x} \times \mathbb{R}^{d_x} \rightarrow \mathbb{R}^{n}$ as follows (its equivalence when $n=1$ is easy to see):
\begin{align*}
	B[f] : \mathbb{R}^{d_x} \times \mathbb{R}^{d_x} \rightarrow \mathbb{R}^{n} :=& \nabla_{(x,\qmo)}^* \left [
	Q  \nabla_{(x,\qmo)}f \right ].
\end{align*}
This can be equivalently written as
\begin{align}
	B[f] =& -  \nabla_{(x,\qmo)}f^\top Q^\top\, \nabla_{(x,\qmo)} \log \rho
	- \nabla_{(x,\qmo)} \cdot \left (Q  \nabla_{(x,\qmo)}f\right ) \nonumber \\
=& - \nabla_{(x,u)} f^\top \, \begin{pmatrix}
	\nabla_\qmo \log \rho \\
	-  \nabla_x \log \rho
\end{pmatrix} \label{eq:Bequaltwo},
\end{align}
where we use the fact that 
\begin{align*}
\left (\nabla_{(x,\qmo)} \cdot \left (Q  \nabla_{(x,\qmo)}f \right ) \right )_i &=  \partial_{(x,u)_j} \left [  \begin{pmatrix}
	\nabla_\qmo f \\
	-  \nabla_x f
\end{pmatrix}_{ji} \right ]=0.
\end{align*}
In the following proposition, we show that the operator is anti-symmetric.
\begin{proposition}
	\label{prop:antisymmetric}
	For all $f,g :\mathbb{R}^{d_x} \times \mathbb{R}^{d_x} \rightarrow \mathbb{R}^{n}$, we have $\iprod{B[f]}{g}_\rho = -\iprod{f}{B[g]}_\rho$.
\end{proposition}
\begin{proof}
	From the definition, we have
	\begin{align*}
		\iprod{B[f]}{g}_\rho &= \iprod{\nabla_{(x,\qmo)}^* \left [
			Q \nabla_{(x,\qmo)}f \right ]}{g}_\rho \\
		&=  \iprod{Q  \nabla_{(x,\qmo)}f}{\nabla_{(x,\qmo)} g}_\rho,
	\end{align*}
	where we use the adjoint operator of $\nabla^*$ for the second equality. Continuing from the last line,
 	\begin{align*}
		\iprod{B[f]}{g}_\rho &=  \iprod{  \nabla_{(x,\qmo)}f}{Q^\top \nabla_{(x,\qmo)} g}_\rho = -\iprod{  f}{\nabla_{(x,\qmo)}^*Q \nabla_{(x,\qmo)} g}_\rho= -\iprod{f}{B[g]}_\rho.
	\end{align*}
\end{proof}

\begin{proposition}
	\label{prop:b}
	For $f: \mathbb{R}^{d_x} \times \mathbb{R}^{d_x} \rightarrow \mathbb{R}$, we have
	\begin{equation*}
		[\nabla_\qmo, B]f = \eta_x \nabla_x f, \quad \text{and} \quad  [\nabla_x, B]f = \nabla_x^2 \log \rho\, \nabla_\qmo f
  \end{equation*}
\end{proposition}
\begin{proof}
	For the first equality, we begin by expanding out the commutator
	$$
	[\nabla_\qmo, B]f = \nabla_\qmo B [f] - B \nabla_\qmo [f].
	$$
	From \Cref{eq:Bequalone}, we have
	$$
	 \nabla_\qmo B [f] = - \nabla_\qmo \nabla_x f \, \nabla_\qmo \log \rho - \nabla_\qmo^2 \log \rho \, \nabla_x f + \nabla_\qmo^2 f\, \nabla_x \log \rho,
	$$
	and, from \Cref{eq:Bequaltwo}, we have
\begin{equation*}
    B \nabla_\qmo [f] = - \nabla_\qmo \nabla_{(x,u)} f \,  \begin{pmatrix}
\nabla_\qmo \log \rho \\
-  \nabla_x \log \rho
\end{pmatrix} = -\nabla_u \nabla_x f\, \nabla_u \log \rho + \nabla_u^2 f \,\nabla_x \log \rho,
\end{equation*}
then, we must have
$$
[\nabla_\qmo, B]f = - \nabla_\qmo^2 \log \rho \, \nabla_x f = \eta_x \nabla_x f.
$$
For the second inequality, we begin similarly 
$$
[\nabla_x, B]f =  \nabla_x B [f] - B \nabla_x [f].
$$
For the first term, from \Cref{eq:Bequalone}, we have
$$
\nabla_x B [f] = -\nabla_x^2 f \, \nabla_\qmo \log \rho + \nabla^2_x \log \rho \, \nabla_\qmo f + \nabla_x \nabla_\qmo  f\, \nabla_x \log \rho.
$$
For the second term, from \Cref{eq:Bequaltwo}, we have
\begin{equation*}
    B \nabla_x [f] = - \nabla_x \nabla_{(x,u)} f \,  \begin{pmatrix}
    \nabla_\qmo \log \rho \\
    -  \nabla_x \log \rho
\end{pmatrix} = - \nabla_x^2f \, \nabla_u \log \rho + \nabla_x \nabla_u f\, \nabla_x \log \rho.
\end{equation*}
Hence, we have 
$$
[\nabla_x, B]f = [\nabla_x, B]f = \nabla_x^2 \log \rho\, \nabla_\qmo f.
$$
\end{proof}

\subsubsection{Simplifying \eqref{eq:dldqdt_t2}}
\label{sec:dldqdt_t2}

\begin{proposition}[Simplifying \eqref{eq:dldqdt_t2}]
	\label{prop:simplifying_dldqdt_t2}
    We have that
\begin{align*}
 \eqref{eq:dldqdt_t2}
= &- 2\gamma_x \left \langle \nabla_{(x,\qmo)}\nabla_{\qmo} \sbrac{\log \frac{q_t}{\rho_t}},
 T_{(x,\qmo)}\nabla_{(x,\qmo)}\nabla_{\qmo} \sbrac{\log \frac{q_t}{\rho_t}} \right \rangle_{q_t} \\
&- \left \langle \nabla_{(x,\qmo)}\log \frac{q_t}{\rho_t},\tilde{\Gamma}_x \nabla_{(x,\qmo)}\log \frac{q_t}{\rho_t} \right \rangle _{q_t},
\end{align*}
where $\rho_t := \rho_{\theta_t, \eta_x}$ and
$$
\tilde{\Gamma}_x = \frac{1}{2}\begin{pmatrix}
    2\tau_{x\qmo}\eta_x I_{d_x} &  \left ( \tau_{\qmo}\eta_x +  \tau_{x \qmo}\gamma_x \eta_x \right )  I_{d_x} + 2\tau_{x} \nabla^2_x [\ell]\\
     \left ( \tau_{\qmo}\eta_x +  \tau_{x \qmo} \gamma_x \eta_x \right )  I_{d_x} + 2\tau_{x} \nabla^2_x [\ell] &   4\gamma_x \tau_{\qmo}\eta_x I_{d_x} + 2\tau_{x \qmo} \nabla^2_x [\ell]
\end{pmatrix}.
$$

\end{proposition}
\begin{proof}
Recall that $\cal{H}_t:=\sqrt{\frac{q_t}{\rho_{\theta_t}}}$. We begin by applying the quotient rule to obtain the fact that
\begin{align*}
    &\nabla_{(x,\qmo)} \left ( \frac{\nabla_{(x,\qmo)} ^* T_{(x,\qmo)}\nabla_{(x,\qmo)} \cal{H}_t}{\cal{H}_t} \right ) \\
    = &\frac{\nabla_{(x,\qmo)} \nabla_{(x,\qmo)} ^* \sbrac{T_{(x,\qmo)}\nabla_{(x,\qmo)} \cal{H}_t}}{\cal{H}_t}
    - \frac{\nabla_{(x,\qmo)} \cal{H}_t \nabla_{(x,\qmo)} ^* \sbrac{ T_{(x,\qmo)}\nabla_{(x,\qmo)} \cal{H}_t }}{\cal{H}_t^2}.
\end{align*}
Then, we can write \eqref{eq:dldqdt_t2} as 
\begin{align}
\eqref{eq:dldqdt_t2}
=& - 8 \left \langle \frac{\nabla_{(x,\qmo)} \nabla_{(x,\qmo)} ^* \sbrac{ T_{(x,\qmo)}\nabla_{(x,\qmo)} \cal{H}_t }}{\cal{H}_t},
\begin{pmatrix}0_{d_x} & -I_{d_x}\\
I_{d_x} & \gamma_{x}I_{d_x}
\end{pmatrix}\nabla_{(x,\qmo)} \log \cal{H}_t \right \rangle_{q_t}  \label{eq:quotient_t1}\\
&+ 8 \left \langle \frac{\nabla_{(x,\qmo)} \cal{H}_t  \nabla_{(x,\qmo)} ^* \sbrac{ T_{(x,\qmo)}\nabla_{(x,\qmo)} \cal{H}_t}}{\cal{H}^2_t},
\begin{pmatrix}
	0_{d_x} & -I_{d_x}\\
	I_{d_x} & \gamma_{x}I_{d_x}
\end{pmatrix}\nabla_{(x,\qmo)} \log \cal{H}_t \right \rangle _{q_t}.
\label{eq:quotient_t2}
\end{align}
We can simplify the terms individually given in \Cref{prop:quotient_t1} and \Cref{prop:quotient_t2} for \eqref{eq:quotient_t1} and \eqref{eq:quotient_t2} Hence,  we can write \eqref{eq:dldqdt_t2} equivalently the following:
\begin{subequations}
	\begin{align}
		\eqref{eq:dldqdt_t2} &=  16\gamma_{x} \left\langle \nabla_{(x,\qmo)}\nabla_{\qmo}\sbrac{\cal{H}_t},T_{(x,\qmo)}\frac{\nabla_{(x,\qmo)} \cal{H}_t}{\cal{H}_t}\otimes\nabla_{\qmo} \cal{H}_t\right\rangle_{\rho_{t}}\label{eq:g_t1}\\
		&- 8\gamma_{x} \left\langle \frac{\nabla_{(x,\qmo)}\cal{H}_t}{\cal{H}_t}\otimes\nabla_{\qmo}\cal{H}_t,
		T_{(x,\qmo)}\frac{\nabla_{(x,\qmo)}\cal{H}_t}{\cal{H}_t}\otimes\nabla_{\qmo}\cal{H}_t\right\rangle_{\rho_{t}} \label{eq:g_t2}\\
		&- 8\gamma_x \left \langle \nabla_{(x,\qmo)}\nabla_{\qmo}\sbrac{\cal{H}_t} ,T_{(x,\qmo)}\nabla_{(x,\qmo)}\nabla_{\qmo}\sbrac{\cal{H}_t} \right \rangle_{\rho_{t}} \label{eq:g_t3}\\
		&-  4\tau_{x \qmo}\eta_x \|\nabla_x \cal{H}_t\|^2_{\rho_{t}} \label{eq:g_t4}\\
		&- 4 \left \langle \nabla_{x} \cal{H}_t , \left [ \left ( \tau_{\qmo}\eta_x +  \tau_{x \qmo} \gamma_x \eta_x \right )  I_{d_x} + 2\tau_{x} \nabla^2_x \sbrac{\ell} \right ] \nabla_{\qmo} \cal{H}_t \right \rangle_{\rho_{t}} \label{eq:g_t5}\\
		&- 8 \left \langle \nabla_{\qmo} \cal{H}_t , \left [ \tau_{\qmo} \eta_x\gamma_x I_{d_x} + \frac{\tau_{x \qmo}}{2}\nabla^2_x \sbrac{\ell} \right ] \nabla_{\qmo} \cal{H}_t \right \rangle_{\rho_{t}}.\label{eq:g_t6}
	\end{align}
\end{subequations}
We can rewrite the sum of \eqref{eq:g_t1}, \eqref{eq:g_t2}, \eqref{eq:g_t3} as
\begin{align*}
&\eqref{eq:g_t1}+ \eqref{eq:g_t2}+ \eqref{eq:g_t3} \\
 &= -8\gamma_x \left \langle \nabla_{(x,\qmo)}\nabla_{\qmo} \sbrac{\cal{H}_t} - \frac{\nabla_{(x,\qmo)}\cal{H}_t}{\cal{H}_t}\otimes\nabla_{\qmo}\cal{H}_t,
 T_{(x,\qmo)}\left (\nabla_{(x,\qmo)}\nabla_{\qmo} \sbrac{\cal{H}_t} -\frac{\nabla_{(x,\qmo)}\cal{H}_t}{\cal{H}_t}\otimes\nabla_{\qmo}\cal{H}_t\right) \right \rangle_{\rho_{t}} .
\end{align*}
Since we have that
$\partial_x \partial_{\qmo} \cal{H}_t = \partial_x (\cal{H}_t \partial_\qmo \log \cal{H}_t) = (\partial_x \cal{H}_t) (\partial_\qmo \log \cal{H}_t) + \cal{H}_t (\partial_x \partial_\qmo \log \cal{H}_t)$,
we can write the above as
\begin{align*}
\eqref{eq:g_t1}+ \eqref{eq:g_t2}+ \eqref{eq:g_t3}
 &= - 8\gamma_x \left \langle \nabla_{(x,\qmo)}\nabla_{\qmo} \sbrac{\log \cal{H}_t},
 T_{(x,\qmo)}\nabla_{(x,\qmo)}\nabla_{\qmo}\sbrac{\log \cal{H}_t} \right \rangle_{q_t} \\
 &= - 2\gamma_x \left \langle \nabla_{(x,\qmo)}\nabla_{\qmo} \sbrac{\log \frac{q_t}{\rho_{t}}},
 T_{(x,\qmo)}\nabla_{(x,\qmo)}\nabla_{\qmo}\sbrac{\log  \frac{q_t}{\rho_{t}}} \right \rangle_{q_t}.
\end{align*}

As for the other terms, we have that
\begin{align*}
\eqref{eq:g_t4} + \eqref{eq:g_t5} + \eqref{eq:g_t6} =&
    -4\left \langle \nabla_{(x,\qmo)}\log \cal{H}_t,\tilde\Gamma_x \nabla_{(x,\qmo)}\log \cal{H}_t\right \rangle_{q_t} \\
    =& -\left \langle \nabla_{(x,\qmo)}\log  \frac{q_t}{\rho_t},\tilde\Gamma_x \nabla_{(x,\qmo)}\log \frac{q_t}{\rho_t} \right \rangle_{q_t},
\end{align*}
where
$$
\tilde\Gamma_x := \frac{1}{2}\begin{pmatrix}
    2 \tau_{x \qmo} \eta_x I_{d_x} &  \left ( \tau_{\qmo}\eta_x +  \tau_{x \qmo} \gamma_x \eta_x \right )  I_{d_x} + 2\tau_{x} \nabla^2_x \ell \\
     \left ( \tau_{\qmo}\eta_x +  \tau_{x \qmo} \gamma_x \eta_x \right )  I_{d_x} + 2\tau_{x} \nabla^2_x \ell &   4\gamma_x \tau_{\qmo}\eta_x I_{d_x} + 2 \tau_{x \qmo} \nabla^2_x \ell
\end{pmatrix}.
$$
We have as desired.
\end{proof}

\begin{proposition}[Simplifying \eqref{eq:quotient_t2}]
\label{prop:quotient_t2}
We have
 \begin{align*}
\eqref{eq:quotient_t2} =&16\gamma_{x}\left\langle \nabla_{(x,\qmo)}\nabla_{\qmo}\sbrac{\cal{H}_t},
T_{(x,\qmo)}\frac{\nabla_{(x,\qmo)} \cal{H}_t}{\cal{H}_t}\otimes\nabla_{\qmo} \cal{H}_t\right\rangle_{\rho_{t}}
 \\
 &-8\gamma_{x} \left\langle \frac{\nabla_{(x,\qmo)}\cal{H}_t}{\cal{H}_t}\otimes\nabla_{\qmo}\cal{H}_t,
 T_{(x,\qmo)}\frac{\nabla_{(x,\qmo)}\cal{H}_t}{\cal{H}_t}\otimes\nabla_{\qmo}\cal{H}_t\right\rangle_{\rho_{t}}.
\end{align*}
\end{proposition}
\begin{proof}
Noting that $\frac{1}{\cal{H}_t}\nabla_{(x,\qmo)}^{*}\sbrac{T_{(x,\qmo)}\nabla_{(x,\qmo)} \cal{H}_t}$ is scalar-valued and using $\cal{H}^2_t$ to change the inner product from $\iprod{\cdot}{\cdot}_{q_t}$ to  $\iprod{\cdot}{\cdot}_{\rho_t}$, we obtain
\begin{align*}
\eqref{eq:quotient_t2}
=&8\left \langle \frac{\nabla_{(x,\qmo)}^{*}\sbrac{T_{(x,\qmo)}\nabla_{(x,\qmo)} \cal{H}_t}}{\cal{H}_t}, \left\langle \nabla_{(x,\qmo)} \cal{H}_t,\begin{pmatrix}0_{d_x} & -I_{d_x}\\
I_{d_x} & \gamma_{x}I_{d_x}
\end{pmatrix}\nabla_{(x,\qmo)} \cal{H}_t\right\rangle \right \rangle_{\rho_{t}} \\
	=&8\gamma_{x} \left \langle \frac{\nabla_{(x,\qmo)}^{*} \sbrac{T_{(x,\qmo)}\nabla_{(x,\qmo)} \cal{H}_t}
 }{\cal{H}_t},\left\Vert \nabla_{\qmo}\cal{H}_t\right\Vert _{2}^{2} \right \rangle_{\rho_{t}}.
\end{align*}
Then, using the adjoint of operator $\nabla_{(x,\qmo)}^*$ as was described in \Cref{sec:adjoint}, we obtain
\begin{align*}
	\eqref{eq:quotient_t2} =&8\gamma_{x} \left\langle \nabla_{(x,\qmo)} \frac{\left\Vert \nabla_{\qmo} \cal{H}_t\right\Vert^2_{2}}{\cal{H}_t},
	T_{(x,\qmo)}\nabla_{(x,\qmo)} \cal{H}_t\right\rangle_{\rho_{t}}.
 \end{align*}
 Then, using the quotient rule, we have that 
 \begin{align*}
	\eqref{eq:quotient_t2} =&8\gamma_{x} \left\langle \frac{\nabla_{(x,\qmo)}\left\Vert \nabla_{\qmo}\cal{H}_t\right\Vert^{2}}{\cal{H}_t}
    ,T_{(x,\qmo)}\nabla_{(x,\qmo)} \cal{H}_t \right\rangle_{\rho_{t}} 
        -8\gamma_{x} \left\langle \frac{\left\Vert \nabla_{\qmo} \cal{H}_t\right\Vert^2_2 \nabla_{(x,\qmo)}\cal{H}_t}{\cal{H}_t^2},T_{(x,\qmo)}\nabla_{(x,\qmo)} \cal{H}_t\right\rangle_{\rho_{t}}  \\
	=&16\gamma_{x} \left\langle \frac{\nabla_{(x,\qmo)}\nabla_{\qmo}\sbrac{\cal{H}_t} \, \nabla_{\qmo} \cal{H}_t}{\cal{H}_t},T_{(x,\qmo)}\nabla_{(x,\qmo)} \cal{H}_t\right\rangle_{\rho_{t}}
         -8\gamma_{x} \left\langle \frac{\left\Vert \nabla_{\qmo}\cal{H}_t\right\Vert^2 \nabla_{(x,\qmo)}\cal{H}_t}{\cal{H}_t^{2}},T_{(x,\qmo)}\nabla_{(x,\qmo)} \cal{H}_t\right\rangle_{\rho_{t}}.
 \end{align*}
 Finally, we obtain
 \begin{align*}
\eqref{eq:quotient_t2} =&16\gamma_{x}\left\langle \nabla_{(x,\qmo)}\nabla_{\qmo}\sbrac{\cal{H}_t},T_{(x,\qmo)}\frac{\nabla_{(x,\qmo)} \cal{H}_t}{\cal{H}_t}\otimes\nabla_{\qmo} \cal{H}_t\right\rangle_{\rho_{t}}
 \\
 &-8\gamma_{x} \left\langle \frac{\nabla_{(x,\qmo)}\cal{H}_t}{\cal{H}_t}\otimes\nabla_{\qmo}\cal{H}_t,T_{(x,\qmo)}\frac{\nabla_{(x,\qmo)}\cal{H}_t}{\cal{H}_t}\otimes\nabla_{\qmo}\cal{H}_t\right\rangle_{\rho_{t}},
\end{align*}
where $\otimes$ denotes the outer product, and we use the fact that $\iprod{Mv}{u} = \iprod{M}{u \otimes v}_F$ and 
$$
 \left\Vert \nabla_{\qmo}\cal{H}_t\right\Vert^2 \nabla_{(x,\qmo)}\cal{H}_t=\left(\nabla_{(x,\qmo)}\cal{H}_t\otimes\nabla_{\qmo}\cal{H}_t\right)\, \nabla_{\qmo}\cal{H}_t.
$$
\end{proof}
\begin{proposition}[Simplifying \eqref{eq:quotient_t1}]
\label{prop:quotient_t1}
 We have
\begin{align*}
\eqref{eq:quotient_t1}
=&-8 \gamma_x\left \langle \nabla_{(x,\qmo)}\nabla_{\qmo}\sbrac{\cal{H}_t},T_{(x,\qmo)}\nabla_{(x,\qmo)}\nabla_{\qmo}\sbrac{\cal{H}_t}\right \rangle_{\rho_t} \\
&-4  \tau_{x \qmo}\eta_x \|\nabla_x \cal{H}_t\|^2_{\rho_t} \\
&- 8 \left \langle \nabla_{\qmo} \cal{H}_t, \left ( \tau_x \eta_x\gamma_x I_{d_x} + \frac{\tau_{x \qmo}}{2}\nabla^2_x \log \rho_t \right ) \nabla_{\qmo} \cal{H}_t \right \rangle_{\rho_t} \\
&-4 \left \langle \nabla_{x} \cal{H}_t , \left ( \left [ \tau_{x}\eta_x +  \tau_{x \qmo}\gamma_x \eta_x \right ]  I_{d_x} + 2 \tau_{x} \nabla^2_x \log \rho_t  \right ) \nabla_{\qmo} \cal{H}_t \right \rangle_{\rho_t},
\end{align*}
where $\rho_t := \rho_{\theta_t, \eta_x}$.
\end{proposition}
\begin{proof}
We begin by changing the inner product from $\iprod{\cdot}{\cdot}_{q_t}$ to $\iprod{\cdot}{\cdot}_{\rho_t}$
\begin{align*}
\eqref{eq:quotient_t1}
= - 8\left \langle \nabla_{(x,\qmo)} \nabla_{(x,\qmo)} ^* \sbrac{T_{(x,\qmo)}\nabla_{(x,\qmo)} \cal{H}_t}, \begin{pmatrix}0_{d_x} & -I_{d_x}\\
I_{d_x} & \gamma_{x}I_{d_x}
\end{pmatrix}
\nabla_{(x,\qmo)} \cal{H}_t \right \rangle_{\rho_{t}}.
\end{align*}
Since we have
\begin{align*}
    \nabla_{(x,\qmo)} ^* \sbrac{ T_{(x,\qmo)}\nabla_{(x,\qmo)} \cal{H}_t} &= \nabla_{x}^*  \sbrac{ \tau_{x}\nabla_{x} \cal{H}_t + \frac{\tau_{x \qmo}}{2} \nabla_{\qmo} \cal{H}_t  } +\nabla_{\qmo}^* \sbrac{ \frac{\tau_{x \qmo}}{2}\nabla_{x} \cal{H}_t + \tau_{\qmo}\nabla_{\qmo} \cal{H}_t},
\end{align*}
we obtain
\begin{subequations}
	\begin{align}
		\eqref{eq:quotient_t1}
		= &- 4 \tau_{x \qmo} \left \langle \nabla_{(x,\qmo)} \left [ \nabla_{x}^* \nabla_{\qmo} \cal{H}_t + \nabla_{\qmo}^* \nabla_{x} \cal{H}_t \right ] ,
		\begin{pmatrix}
			0_{d_x} & -I_{d_x} \\
			I_{d_x} & \gamma_{x}I_{d_x}
		\end{pmatrix}\nabla_{(x,\qmo)} \cal{H}_t\right \rangle_{\rho_{t}} \label{eq:nested_t1}\\
		&- 8 \tau_{x} \left \langle \nabla_{(x,\qmo)} \left [ \nabla_{x}^* \left(\nabla_{x} \cal{H}_t   \right ) \right ], \begin{pmatrix}
			0_{d_x} & -I_{d_x} \\
			I_{d_x} & \gamma_{x}I_{d_x}
		\end{pmatrix}\nabla_{(x,\qmo)} \cal{H}_t \right \rangle_{\rho_{t}} \label{eq:nested_t2}\\
		&- 8 \tau_{\qmo} \left \langle \nabla_{(x,\qmo)} \left [ \nabla_{\qmo}^* \left(\nabla_{\qmo} \cal{H}_t \right ) \right ] , \begin{pmatrix}
			0_{d_x} & -I_{d_x}\\
			I_{d_x} & \gamma_{x}I_{d_x}
		\end{pmatrix}\nabla_{(x,\qmo)} \cal{H}_t\right \rangle_{\rho_{t}}.
		\label{eq:nested_t3}
	\end{align}
\end{subequations}
In Proposition \ref{prop:nested_t1}, \ref{prop:nested_t2} and \ref{prop:nested_t3} we deal with these respective terms:
\begin{proposition}[Simplifying \eqref{eq:nested_t1}]
We have
    \begin{align*}
    \eqref{eq:nested_t1}
    = &- 4 \tau_{x \qmo} {\eta_x}
    \|\nabla_x \cal{H}_t\|^2_{\rho_{t}}
    \\
    &-4 \tau_{x \qmo} \left \langle \nabla_{\qmo} \cal{H}_t,
    \nabla^2_x \sbrac{\ell} \, \nabla_{\qmo} \cal{H}_t \right \rangle_{\rho_{t}} \\
    &-8 \gamma_{x} \tau_{x \qmo}
    \left \langle   \nabla_{\qmo}^2 \sbrac{\cal{H}_t},  \nabla_{\qmo} \nabla_{x} \sbrac{\cal{H}_t} \right \rangle_{\rho_{t}}\\
    &-4 \gamma_{x}\eta_x \tau_{x \qmo} \left \langle \nabla_{x} \cal{H}_t , \nabla_{\qmo} \cal{H}_t \right \rangle_{\rho_{t}}.
    \end{align*}
    \label{prop:nested_t1}
\end{proposition}

\begin{proposition}[Simplifying \eqref{eq:nested_t2}]
   We have
   \begin{align*}
       \eqref{eq:nested_t2} = -8 \tau_{x} \gamma_x \left \| \nabla_x\nabla_{\qmo} \cal{H}_t \right \|^2_{\rho_{t}}
    -8 \tau_{x} \left \langle \nabla_{x} \cal{H}_t , \nabla^2_x \sbrac{\ell} \nabla_{\qmo} \cal{H}_t  \right \rangle_{\rho_{t}},
   \end{align*}
   \label{prop:nested_t2}
\end{proposition}

\begin{proposition}[Simplifying \eqref{eq:nested_t3}]
    \begin{align*}
        \eqref{eq:nested_t3}= - 8 \tau_{\qmo} \gamma_x \left \| \nabla_{\qmo}^2 \cal{H}_t \right \|^2_{\rho_{t}}
        - 8 \tau_{\qmo} \eta_x \gamma_{x} \left \| \nabla_{\qmo} \cal{H}_t\right \|_{\rho_{t}}^2
        - 4 \tau_{\qmo} \eta_x \langle \nabla_x \cal{H}_t, \nabla_{\qmo} \cal{H}_t\rangle_{\rho_{t}}.
    \end{align*}
    \label{prop:nested_t3}
\end{proposition}
Their proofs can be found below.

Thus, summing up the results of \Cref{prop:nested_t1}, \Cref{prop:nested_t2} and \Cref{prop:nested_t3}, we obtain
 \begin{subequations}
    \begin{align}
	\eqref{eq:quotient_t1} =&  \eqref{eq:nested_t1}+ \eqref{eq:nested_t2} + \eqref{eq:nested_t3} \nonumber \\
	=&- 8 \tau_{x \qmo} \gamma_{x} \left \langle \nabla_{\qmo}^2 \cal{H}_t,  \nabla_{\qmo} \nabla_{x} \sbrac{\cal{H}_t} \right \rangle_{\rho_{t}} \label{eq:f_t1}\\
	&- 8 \tau_{\qmo} \gamma_x \left \|  \nabla_{\qmo}^2 \cal{H}_t \right \|_{\rho_{t}}^2 \label{eq:f_t2}\\
	&- 8 \tau_{x} \gamma_x \left \langle \nabla_x\nabla_{\qmo} \cal{H}_t,\nabla_x\nabla_{\qmo} \cal{H}_t \right \rangle_{\rho_{t}} \label{eq:f_t3}\\
	&-4 \tau_{x \qmo} {\eta_x} \|\nabla_x \cal{H}_t\|^2_{\rho_{t}} \nonumber \\
	&- 8 \left \langle \nabla_{\qmo} \cal{H}_t, \left ( \tau_{\qmo} \eta_x\gamma_x I_{d_x} + \frac{\tau_{x \qmo}}{2}\nabla^2_x \sbrac{\ell} \right ) \nabla_{\qmo} \cal{H}_t \right \rangle_{\rho_{t}} \nonumber  \\
	&-4 \left \langle \nabla_{x} \cal{H}_t, \left [ \left ( \tau_{\qmo} \eta_x +  \tau_{x \qmo}\gamma_x \eta_x \right ) I_{d_x} + 2 \tau_{x} \nabla^2_x \sbrac{\ell} \right ] \nabla_{\qmo} \cal{H}_t \right \rangle_{\rho_{t}}. \nonumber
\end{align}
 \end{subequations}
Since we can write the following sum equivalently as:
    \begin{align*}
        \eqref{eq:f_t1} + \eqref{eq:f_t2} + \eqref{eq:f_t3}
            = &- 8 \gamma_x \left \langle \nabla_{(x,\qmo)}\nabla_{\qmo}\sbrac{\cal{H}_t},T_{(x,\qmo)}\nabla_{(x,\qmo)}\nabla_{\qmo}\sbrac{\cal{H}_t} \right \rangle_{\rho_{t}},
    \end{align*}
    we obtain as desired.
\end{proof}
\begin{proof}[Proof of \Cref{prop:nested_t1}]
    \label{proof:nested_t1}
    Beginning from the definition, we have
    \begin{subequations}
    	 \begin{align}
    		\eqref{eq:nested_t1} = -&4 \tau_{x \qmo} \left \langle \nabla_{(x,\qmo)} \left [ \nabla_{x}^* \nabla_{\qmo} \cal{H}_t + \nabla_{\qmo}^* \nabla_{x} \cal{H}_t \right ],
    		\begin{pmatrix}
    			0_{d_x} & 0_{d_x} \\
    			0_{d_x} & \gamma_{x}I_{d_x}
    		\end{pmatrix}\nabla_{(x,\qmo)} \cal{H}_t \right \rangle_{\rho_{t}} \nonumber \\
    		- &4 \tau_{x \qmo} \left \langle \nabla_{(x,\qmo)} \left [ \nabla_{x}^* \nabla_{\qmo} \cal{H}_t + \nabla_{\qmo}^* \nabla_{x} \cal{H}_t \right ], \begin{pmatrix}
    			0_{d_x} & -I_{d_x} \\
    			I_{d_x} & 0_{d_x}
    		\end{pmatrix}\nabla_{(x,\qmo)} \cal{H}_t \right \rangle_{\rho_{t}} \nonumber \\
    		= -  & 4 \gamma_{x} \tau_{x \qmo} \left \langle \nabla_{\qmo} \left [ \nabla_{x}^* \nabla_{\qmo} \cal{H}_t + \nabla_{\qmo}^* \nabla_{x} \cal{H}_t \right ], \nabla_{\qmo} \cal{H}_t \right \rangle_{\rho_{t}} \label{eq:dt_t3}\\
    		- & 4\tau_{x \qmo} \left \langle \nabla_{(x,\qmo)} \left [ \nabla_{x}^* \nabla_{\qmo} \cal{H}_t + \nabla_{\qmo}^* \nabla_{x} \cal{H}_t \right ],
    		\begin{pmatrix}
    			0_{d_x} & -I_{d_x} \\
    			I_{d_x} & 0_{d_x}
    		\end{pmatrix}\nabla_{(x,\qmo)} \cal{H}_t \right \rangle_{\rho_{t}}. \label{eq:dt_t4}
    	\end{align}
    \end{subequations}
    We will deal with \eqref{eq:dt_t3} and \eqref{eq:dt_t4} separately.
    
    Using the adjoints of $\nabla_{\qmo}^*, \nabla_{\qmo}, \nabla_{x}$, we obtain
    \begin{align*}
        \eqref{eq:dt_t3}
        = &- 4\gamma_{x} \tau_{x \qmo} \left \langle   \nabla_{\qmo} \cal{H}_t ,  \nabla_{x}\nabla_{\qmo}^*\nabla_{\qmo} \cal{H}_t \right \rangle_{\rho_{t}} 
    - 4\gamma_{x} \tau_{x \qmo} \left \langle \nabla_{x} \cal{H}_t , \nabla_{\qmo}\nabla_{\qmo}^* \nabla_{\qmo} \cal{H}_t \right \rangle_{\rho_{t}},
    \end{align*}
    From \Cref{prop:commute3},  we have that  $\nabla_x$ commutes with $\nabla^*_{\qmo}\nabla_{\qmo}$
    \begin{align*}
        \eqref{eq:dt_t3} =- &4\gamma_{x} \tau_{x \qmo} \left \langle   \nabla_{\qmo} \cal{H}_t,  \nabla_{\qmo}^*\nabla_{\qmo} \nabla_{x} \sbrac{\cal{H}_t} \right \rangle_{\rho_{t}} 
        - 4 \gamma_{x} \tau_{x \qmo} \left \langle \nabla_{x} \cal{H}_t , \nabla_{\qmo}\nabla_{\qmo}^* \nabla_{\qmo} \cal{H}_t \right \rangle_{\rho_{t}},
    \end{align*}
    We can write this in terms of the commutator $[\cdot,\cdot]$ (defined in \Cref{sec:commutator}) using the fact that $[\nabla_{\qmo}, \nabla_{\qmo}^*] \nabla_\qmo f= \eta_x \nabla_\qmo f$ (as shown in \Cref{prop:commute1}).
    Thus, we have 
    \begin{align}
         \eqref{eq:dt_t3} 
         = &- 8 \gamma_{x} \tau_{x \qmo} \left \langle   \nabla_{\qmo} \cal{H}_t,  \nabla_{\qmo}^*\nabla_{\qmo} \nabla_{x} \sbrac{\cal{H}_t} \right \rangle_{\rho_{t}} 
         - 4 \gamma_{x} \tau_{x \qmo} \left \langle \nabla_{x} \cal{H}_t , [\nabla_{\qmo}, \nabla_{\qmo}^*] \nabla_{\qmo} \cal{H}_t \right \rangle_{\rho_{t}} \nonumber \\
    =&- 8 \gamma_{x} \tau_{x \qmo} \left \langle \nabla_{\qmo}^2 \sbrac{\cal{H}_t},  \nabla_{\qmo}\nabla_{x} \sbrac{\cal{H}_t} \right \rangle_{\rho_{t}}
    - 4 \gamma_{x}\eta_x \tau_{x \qmo} \left \langle \nabla_{x} \cal{H}_t , \nabla_{\qmo} \cal{H}_t \right \rangle_{\rho_{t}}. \label{eq:t2_first}
    \end{align}

    As for \eqref{eq:dt_t4}, we start by
    writing it in terms of $B$ (see \Cref{sec:b}) $B[\cal{H}_t] := \nabla_{(x,u)}^* Q \nabla_{(x,u)} \cal{H}_t$, we can write 
    \begin{align*}
    	\eqref{eq:dt_t4} = &-4 \tau_{x \qmo} \left \langle  \nabla_{(x,\qmo)}^* \nabla_{(\qmo, x)} \cal{H}_t,  B[\cal{H}_t] \right \rangle_{\rho_{t}}.
    \end{align*}
    Using the fact that the adjoint of $B$ denoted by $B^*$ w.r.t.\ the inner product $\langle \cdot ,\cdot \rangle_{\rho}$ is $B^*= -B$ (see \Cref{prop:antisymmetric})
    we obtain
    \begin{align*}
        \eqref{eq:dt_t4} &= -4 \tau_{x \qmo} \left \langle \nabla_{(\qmo, x)} \cal{H}_t , \nabla_{(x, \qmo)}  B[\cal{H}_t] 
        \right \rangle_{\rho_{t}} \\
        &= -4 \tau_{x \qmo} \left ( \left \langle \nabla_{x} \cal{H}_t , \nabla_{\qmo}  B[\cal{H}_t]\right \rangle_{\rho_{t}} 
        + \left \langle \nabla_{\qmo} \cal{H}_t , \nabla_{x}  B[\cal{H}_t] 
        \right \rangle_{\rho_{t}} \right ) \\
        &= -4 \tau_{x \qmo} \left ( \left \langle \nabla_{x} \cal{H}_t , \nabla_{\qmo}  B[\cal{H}_t] 
        \right \rangle_{\rho_{t}}
        + \left \langle \nabla_{\qmo} \cal{H}_t , B \nabla_{x}  [\cal{H}_t] \right \rangle_{\rho_{t}}
        + \left \langle \nabla_{\qmo} \cal{H}_t  , [\nabla_{x}, B][\cal{H}_t] \right \rangle_{\rho_{t}}  \right ) \\
        &= -4 \tau_{x \qmo} \left (  \left \langle \nabla_{x} \cal{H}_t , \nabla_{\qmo}  B[\cal{H}_t] 
        \right \rangle_{\rho_{t}}
        - \left \langle B\nabla_{\qmo} \cal{H}_t  , \nabla_{x}  [\cal{H}_t] \right \rangle_{\rho_{t}}
        + \left \langle \nabla_{\qmo} \cal{H}_t  , [\nabla_{x}, B][\cal{H}_t] \right \rangle_{\rho_{t}} \right ) \\
        &= -4\tau_{x \qmo} \left ( \left \langle \nabla_{x} \cal{H}_t , [\nabla_{\qmo},B][\cal{H}_t] 
        \right \rangle_{\rho_{t}}
        + \left \langle \nabla_{\qmo} \cal{H}_t  , [\nabla_{x}, B][\cal{H}_t] \right \rangle_{\rho_{t}} \right ).
    \end{align*}
    From \Cref{prop:b}, we have $[\nabla_{\qmo}, B][\cal{H}_t]=\eta_x \nabla_x \cal{H}_t$ and $[\nabla_{x}, B][\cal{H}_t]= \nabla_x^2[ \ell]\nabla_{\qmo}\cal{H}_t$, and so
    \begin{align}
        \eqref{eq:dt_t4}=&
        -4{\tau_{x \qmo}\eta_x} \|\nabla_x \cal{H}_t\|^2_{\rho_{t}} -4 \tau_{x \qmo} \left \langle \nabla_{\qmo} \cal{H}_t , \nabla^2_x \sbrac{\ell} \nabla_{\qmo} \cal{H}_t \right \rangle_{\rho_{t}}. \label{eq:t2_last}
    \end{align}
    Thus, summing up \eqref{eq:t2_first} and \eqref{eq:t2_last}, we have as desired
    \begin{align*}
    \eqref{eq:nested_t1} =  \eqref{eq:t2_first} + \eqref{eq:t2_last}
    = &- 4 \tau_{x \qmo} {\eta_x}
    \|\nabla_x \cal{H}_t\|^2_{\rho_{t}}
    \\
    &-4 \tau_{x \qmo} \left \langle \nabla_{\qmo} \cal{H}_t,
    \nabla^2_x \sbrac{\ell} \, \nabla_{\qmo} \cal{H}_t \right \rangle_{\rho_{t}} \\
    &-8 \gamma_{x} \tau_{x \qmo}
    \left \langle \nabla_{\qmo}^2 \sbrac{\cal{H}_t} ,  \nabla_{\qmo} \nabla_{x} \sbrac{\cal{H}_t} \right \rangle_{\rho_{t}}\\
    &-4 \gamma_{x}\eta_x \tau_{x \qmo} \left \langle \nabla_{x} \cal{H}_t , \nabla_{\qmo} \cal{H}_t \right \rangle_{\rho_{t}}.
    \end{align*}
\end{proof}

\begin{proof}[Proof of \Cref{prop:nested_t2}]
    \label{proof:nested_t2}
     Beginning from the definition and following in the same expansion as in \Cref{prop:nested_t1}, and using the adjoint of $\nabla_{(x.u)}$ ( i.e., $\nabla^*_{(x.u)}$) and adjoint of $\nabla^*_{x}$ (i.e., $\nabla^*_{x}$)
     we obtain the following:
    \begin{align*}
    \eqref{eq:nested_t2}
    =& -8 \tau_{x} \gamma_x  \left \langle \nabla_{\qmo} \left [ \nabla_{x}^* \left(\nabla_{x} \cal{H}_t   \right ) \right ] , \nabla_{\qmo} \cal{H}_t \right \rangle_{\rho_{t}}
    -8 \tau_{x} \left \langle \nabla_{x} \cal{H}_t   , \nabla_x B [\cal{H}_t] \right \rangle_{\rho_{t}}.
    \end{align*}
    Writing this in terms of the commutator, we get 
    \begin{align*}
    \eqref{eq:nested_t2}
    =& -8 \tau_{x} \gamma_x \left \langle \nabla_{\qmo} \left [ \nabla_{x}^* \left(\nabla_{x} \cal{H}_t   \right ) \right ], \nabla_{\qmo} \cal{H}_t \right \rangle_{\rho_{t}}
    -8 \tau_{x} \left \langle \nabla_{x} \cal{H}_t, [\nabla_x, B] \cal{H}_t +B\nabla_x \cal{H}_t \right \rangle_{\rho_{t}} \\
    =& -8 \tau_{x} \gamma_x \left \langle \nabla_x\nabla_{\qmo}\cal{H}_t,
    \nabla_x\nabla_{\qmo}\cal{H}_t \right \rangle_{\rho_{t}}
    -8 \tau_{x} \left \langle \nabla_{x} \cal{H}_t , \nabla^2_x \sbrac{\ell} \nabla_{\qmo} \cal{H}_t  \right \rangle_{\rho_{t}},
    \end{align*}
    where, for the first term, we use the fact that $\nabla_u$ commutes with $\nabla_{x}^*\nabla_{x}$ (see \Cref{prop:commute2}); and, for the second term, we use \Cref{prop:b} and $\left \langle \nabla_{x} \cal{H}_t , B\nabla_x \cal{H}_t  \right \rangle_{\rho_{t}}=0$ from antisymmetry (see \Cref{prop:antisymmetric}).
\end{proof}
\begin{proof}[Proof of \Cref{prop:nested_t3}]
\label{proof:nested_t3}
Beginning with the definition and following in the same expansion as in \Cref{prop:nested_t1}, we obtain
    \begin{align*}
        \eqref{eq:nested_t3}
        = &- 8\tau_{\qmo} \gamma_{x} \left \langle \nabla_{\qmo} \nabla_{\qmo}^* \left [\nabla_{\qmo} \cal{H}_t \right ] , \nabla_{\qmo} \cal{H}_t \right \rangle_{\rho_{t}} \\
        &-8 \tau_{\qmo} \left \langle \nabla_{\qmo} \cal{H}_t , \nabla_{\qmo} \nabla_{(x,\qmo)}^* \sbrac{
        	\begin{pmatrix}
        		0_{d_x} & -I_{d_x}\\
		        I_{d_x} & 0_{d_x}
        \end{pmatrix}\nabla_{(x,\qmo)} \cal{H}_t} \right \rangle_{\rho_{t}} \\
        = &- 8\tau_{\qmo} \gamma_{x} \left \langle \nabla_{\qmo} \nabla_{\qmo}^* \left[\nabla_{\qmo} \cal{H}_t \right ], \nabla_{\qmo} \cal{H}_t \right \rangle_{\rho_{t}} - 4 \tau_{\qmo} \left \langle \nabla_{\qmo} \cal{H}_t , \nabla_{\qmo} B [\cal{H}_t] \right \rangle_{\rho_{t}}.
    \end{align*}
    Writing in terms of the commutator $[\cdot, \cdot]$ and $B$, we have equivalently
    \begin{align*}
        \eqref{eq:nested_t3}
        = &- 8 \tau_{\qmo} \gamma_x \left \langle \nabla_{\qmo} \nabla_{\qmo} \sbrac{\cal{H}_t}  , \nabla_{\qmo}\nabla_{\qmo} \sbrac{\cal{H}_t} \right \rangle_{\rho_{t}}
        - 8 \tau_{\qmo} \gamma_{x} \left \langle \left [ \nabla_{\qmo}, \nabla_{\qmo}^* \right ] \nabla_{\qmo} \cal{H}_t  , \nabla_{\qmo} \cal{H}_t \right \rangle_{\rho_{t}} \\
        & - 4 \tau_{\qmo} \left \langle \nabla_{\qmo} \cal{H}_t, B\nabla_{\qmo} \cal{H}_t \right \rangle_{\rho_{t}}
        - 4 \tau_{\qmo} \eta_x  \langle \nabla_x \cal{H}_t, \nabla_{\qmo} \cal{H}_t \rangle_{\rho_{t}},
    \end{align*}
    where, for the last term, we use the fact of \Cref{prop:b} to obtain
    \begin{align*}
        \iprod{\nabla_u \cal{H}_t}{[\nabla_{\qmo}, B] \cal{H}_t}_{\rho_{t}} = \eta_x \iprod{\nabla_x \cal{H}_t}{\nabla_{\qmo} \cal{H}_t}_{\rho_{t}}
    \end{align*}
    Using the antisymmetric property of $B$, we have $\left \langle \nabla_{\qmo} \cal{H}_t, B\nabla_{\qmo} \cal{H}_t \right \rangle_{\rho_{t}} = 0$; and, from \Cref{prop:commute1}, the fact that $[\nabla_\qmo, \nabla_\qmo^*] \nabla_\qmo f = \eta_x \nabla_\qmo f$, we have
    \begin{align*}
        \eqref{eq:nested_t3}=& - 8 \tau_{\qmo} \gamma_x \left \langle \nabla_{\qmo}^2 \sbrac{\cal{H}_t} , \nabla_{\qmo}^2 \sbrac{\cal{H}_t}  \right \rangle_{\rho_{t}} \\
        &- 8 \tau_{\qmo} \eta_x \gamma_{x} \left \langle \nabla_{\qmo} \cal{H}_t , \nabla_{\qmo} \cal{H}_t \right \rangle_{\rho_{t}} \\
        &- 4 \tau_{\qmo} \eta_x \langle \nabla_x \cal{H}_t, \nabla_{\qmo} \cal{H}_t\rangle_{\rho_{t}}.
    \end{align*}
\end{proof}

\subsubsection{Bounding the Cross terms \eqref{eq:dldthetadt_t2} and \eqref{eq:dldqdt_t3}.} 
\label{sec:xterm}
\begin{proposition}
	\label{prop:xterm_bound}
	For all $\varepsilon > 0$, we have that
	\begin{align*}
		\eqref{eq:dldthetadt_t2} + \eqref{eq:dldqdt_t3}
		\le& \varepsilon\iprod{\nabla_{(\theta, m)}\cal{F}_t}{  \Gamma_{\times} \nabla_{(\theta, m)}\cal{F}_t} + \frac{1}{\varepsilon}\| \nabla_{(x,\qmo)} \delta_{q}\cal{F}_t \|^2_{q_t},
	\end{align*}
	where $\Gamma_{\times} = K^2 \begin{pmatrix}
		\tau_\theta^2 I_{d_\theta} & \tau_\theta \left ( \frac{\tau_{xu}+\tau_{\theta m }}{2}\right) I_{d_\theta} \\
		\tau_\theta \left ( \frac{\tau_{xu}+\tau_{\theta m }}{2}\right) I_{d_\theta} & \left ( \left ( \frac{\tau_{xu}+\tau_{\theta m }}{2}\right)^2 + \tau_x^2 \right ) I_{d_\theta}
	\end{pmatrix}$.
\end{proposition}
\begin{proof}
	Because $T_{(x,u)}$ is symmetric,
	\begin{align*}
		\eqref{eq:dldthetadt_t2}
		&=-2\iprod{ A_t \nabla_{(\theta, m)}
		\cal{F}_t}{ \nabla_{(x,\qmo)}\delta_{q}\cal{F}_t }_{q_t}, \quad \eqref{eq:dldqdt_t3}
		= 2 \left \langle  A^{'}_t \nabla_{(\theta,m)}\cal{F}_t, \nabla_{(x,\qmo)} \delta_{q}\cal{F}_t \right \rangle_{q_t} ,
	\end{align*}
	where
	\begin{align*}
		A_t &:= T_{(x,\qmo)}
		\nabla_{(x,\qmo)} \nabla_{(\theta,m)} \sbrac{\delta_{q}\cal{F}_t}\,
		\begin{pmatrix}
			0_{d_\theta} & -I_{d_\theta}\\
			I_{d_\theta} & \gamma_{\theta}I_{d_\theta}
		\end{pmatrix}, \\
		A_t^{'} &:= \begin{pmatrix}0_{d_x} & I_{d_x}\\
			-I_{d_x} & \gamma_{x}I_{d_x}
		\end{pmatrix}
		\nabla_{(x,\qmo)} \nabla_{(\theta,m)} \sbrac{\log \rho_{\theta_t}} \,
		T_{(\theta,\tmo)}.
	\end{align*}
Given that
\begin{align*}
	\nabla_{(x,\qmo)}\nabla_{(\theta,m)} \sbrac{\delta_{q}\cal{F}_t}
	=-\nabla_{(x,\qmo)}\nabla_{(\theta,m)}\sbrac{\log \rho_{\theta_t, \eta_x}}
 = -\begin{pmatrix}
		 \nabla_x \nabla_\theta \sbrac{\log \rho_{\theta_t}}  & 0_{d_x\times d_\theta} \\
		0_{d_x\times d_\theta}  & 0_{d_x\times d_\theta} \\ 
	\end{pmatrix},
\end{align*}
we can re-write $A_t$ and $A_t^{'}$ as
	\begin{align*}
		A_t =
		\begin{pmatrix}
			0_{d_x \times d_\theta} & \tau_{x} \nabla_x \nabla_\theta \sbrac{\log \rho_{\theta_t}} \\
			0_{d_x \times d_\theta}  & \frac{\tau_{x\qmo}}{2} \nabla_x \nabla_\theta  \sbrac{\log \rho_{\theta_t}}
		\end{pmatrix}, \quad
		A_t^{'}=  -
		\begin{pmatrix}
			0_{d_x \times d_\theta} 
			& 0_{d_x \times d_\theta}  \\
			\tau_{\theta} \nabla_x \nabla_\theta \sbrac{\log \rho_{\theta_t}}
			& \frac{\tau_{\theta\tmo}} {2} \nabla_x \nabla_\theta \sbrac{\log \rho_{\theta_t}} 
		\end{pmatrix}.
	\end{align*}
Hence,
	\begin{align}
		\eqref{eq:dldthetadt_t2} + \eqref{eq:dldqdt_t3} = & 2\iprod{ A_t \nabla_{(\theta,m)}\cal{F}_t}{
		\nabla_{(x,\qmo)} \delta_{q}\cal{F}_t}_{q_t}\nonumber
	\end{align}
	where
	$$
	\tilde{A}_t :=  A_t^{'} - A_t =
	- \begin{pmatrix}
		0 
		&
		\tau_x \nabla_{x} \nabla_{\theta} \sbrac{\log \rho_{\theta_t}}\\
		\tau_\theta \nabla_{x} \nabla_{\theta} \sbrac{\log \rho_{\theta_t}}
		&
		\frac{(\tau_{x\qmo}+\tau_{\theta \tmo})}{2}\nabla_{x}\nabla_{\theta} \sbrac{\log \rho_{\theta_t}}
	\end{pmatrix}.
	$$
	Fix any $\varepsilon>0$. Applying the Cauchy-Schwarz inequality and Young's inequality, we obtain that
	\begin{align}
		\eqref{eq:dldthetadt_t2} + \eqref{eq:dldqdt_t3} 
		\le& 2 \norm{\tilde{A}_t \nabla_{(\theta,m)}\cal{F}_t}_{q_t} \norm{
		\nabla_{(x,\qmo)} \delta_{q}\cal{F}_t}_{q_t} \nonumber\\
		\le&  \varepsilon \norm{\tilde{A}_t \nabla_{(\theta,m)}\cal{F}_t}^2_{q_t}
		+  \frac{1}{\varepsilon}\norm{
		\nabla_{(x,\qmo)} \delta_{q}\cal{F}_t}^2_{q_t} \label{eq:xterm_youngs}.
	\end{align}
	But,
	\begin{align*}
		\| \tilde{A}_t \nabla_{(\theta,m)}\cal{F}_t\|^2_{q_t}
		=& \| \tau_x \nabla_x \nabla_\theta \sbrac{\log \rho_{\theta_t}}\,  \nabla_{m}\cal{F}_t\|^2_{q_t} \\
		&+ \left \|\nabla_x \nabla_\theta \sbrac{\log \rho_{\theta_t}} \, \left [\tau_\theta \nabla_\theta\cal{F}_t + \frac{\tau_{xu} + \tau_{\theta m}}{2}\nabla_{m}\cal{F}_t\right ] \right \|_{q_t}^2.\end{align*}
        Given that $\nabla_\theta\cal{F}(z)$ and $\nabla_m\cal{F}(z)$ are constant in $x$, the above reads
  \begin{align}
		\| \tilde{A}_t \nabla_{(\theta,m)}\cal{F}_t\|^2_{q_t}=& \|  \nabla_x \iprod{\nabla_\theta \log \rho_{\theta_t}}{\tau_x \nabla_{m}\cal{F}_t } \|^2_{q_t}\nonumber
		\\
		&+ \left \|\nabla_x \left \langle \nabla_\theta \log \rho_{\theta_t} , \left (\tau_\theta \nabla_\theta\cal{F}_t + \frac{\tau_{xu} + \tau_{\theta m}}{2}\nabla_{m}\cal{F}_t\right ) \right \rangle \right \|^2_{q_t}.\label{eq:dnsau9ndsuiadnasos}
	\end{align}
	Fix any $\theta,v$ in $\r^{d_\theta}$. Because $\nabla_{(\theta,x)}\ell$ is $K$-Lipschitz (\Cref{ass:gradlip}), the function 
 $$f_{v,\theta}(x) := \left \langle \nabla_\theta  \log \rho_{\theta} (x), v\right \rangle=\left \langle \nabla_\theta \ell(\theta,x), v\right \rangle\quad\forall x\in\r^{d_x},$$
 is $(K\|v\|)$-Lipschitz: by the Cauchy-Schwarz inequality,
	\begin{align*}
		f_{v,\theta}(x) - f_{v,\theta}(x') =&
		\left \langle \nabla_\theta  \log\ell(\theta,x) - \nabla_\theta \ell(\theta,x'), v\right \rangle \leq \|\nabla_\theta  \ell(\theta,x) - \nabla_\theta  \ell(\theta,x')\|\|v\|, \\
		\le& \|\nabla_{(\theta,\tmo)}  \ell(\theta,x) - \nabla_{(\theta,\tmo)} \ell(\theta,x')\|\|v\|\leq K\|v\|\|x-x'\|\quad\forall x,x'\in\r^{d_x}.
	\end{align*}
Recalling that Lipschitz seminorms can be estimated by suprema of the norm of the gradient (e.g., see \citet[Lemma 4.3]{van2014probability}), we then see that
	$$
	\|\nabla_x f_{v,\theta}\|^2_{q_t} \le \sup_{x\in\r^{d_x}} \|\nabla_x f_{v,\theta}(x)\|^2\le K^2 \|v\|^2.
 $$
Applying the above inequality to~\eqref{eq:dnsau9ndsuiadnasos}, we find that
	\begin{align*}
		\| \tilde{A}_t \nabla_{(\theta,m)}\cal{F}_t\|^2_{q_t}&\leq 
		K^2\left[ \|\tau _x\nabla_m\cal{F}_t\|^2+\left \|\tau_\theta \nabla_\theta\cal{F}_t + \left ( \frac{\tau_{xu} + \tau_{\theta m}}{2}\right ) \nabla_{m}\cal{F}_t\right \|^2\right]\\
  &=\left \langle \nabla_{(\theta, m)}\cal{F}_t , \Gamma_{\times} \nabla_{(\theta, m)}\cal{F}_t\right \rangle.
	\end{align*}
 Plugging the above into \eqref{eq:xterm_youngs} yields the desired inequality.
\end{proof}
\subsubsection{Positive Semi-definiteness Conditions}
\label{sec:psd_conditions}
In this section, we establish sufficient conditions to ensure that a matrix with the form described in \eqref{eq:psd_cond_form}.

\begin{proposition}
	\label{prop:psd_conds}
	Given  $\alpha, \beta, \kappa, a, c \in \mathbb{R}$, let $A$ be a symmetric matrix with the following form:
	\begin{equation}
	A = 
	\begin{pmatrix}
		\alpha I & \beta I + a H \\
		\beta I + a H & \kappa I + c H
	\end{pmatrix},
	\label{eq:psd_cond_form}
	\end{equation}
	where $H$ is a symmetric matrix that satisfies $-KI \preceq H \preceq K I$ for some $K > 0$.
	If $A$ satisfies the following conditions:
	\begin{align*}
	\begin{cases}
		- cK - \alpha - \kappa \le 0 & {\text{if } c \le 0,} \\
		cK - \alpha - \kappa \le 0 & \text{otherwise}. \\
	\end{cases}\\
	a^2 K^2 - (c\alpha - 2 \beta a)K - \alpha \kappa + \beta^2  \le 0, \\
	a^2 K^2 + (c\alpha - 2 \beta a)K - \alpha \kappa + \beta^2  \le 0,
	\end{align*}
	then, we have that $A$ is positive semi-definite.
\end{proposition}

\begin{proof}
	We prove this by showing that if the conditions are satisfied, then the eigenvalues of $A$ are non-negative.
	
	The eigenvalues $l$ of $A$ satisfy its characteristic equation
	$$
	\mathrm{det}(A-lI) = 0.
	$$
	We have that 
	$$
	\mathrm{det}(A-lI) = \mathrm{det}((\alpha - l)(\kappa -l)I + c (\alpha - l ) H - (\beta I  + a H)^2),
	$$
    where we use the fact that $(\kappa I + c H)$ and $(\beta I + a H)$ are symmetric matrices and so their multiplication commutes (e.g., see \citet{silvester2000determinants}).

	Let $H=U\Lambda U^\top$ be the eigenvalue decomposition of $H$, then observe that
	$$
	\mathrm{det}(A-lI) = \mathrm{det}(U (((\alpha - l)(\kappa - l) - \beta^2)I + ( c ( \alpha - l) - 2 \beta a)\Lambda - a^2\Lambda^2) U^\top).
	$$
	Hence, we obtain for each eigenvalue $\Lambda_i$ of $H$ the following constraints:
	$$
	((\alpha - l)(\kappa - l) - \beta^2) + ( c ( \alpha - l) - 2 \beta a)\Lambda_i - a^2\Lambda_i^2 = 0.
	$$
	These constraints can be written equivalently as
	$$
	l^2 + \left (- c\Lambda_i - \alpha - \kappa \right )l + \left (\alpha \kappa - \beta^2 - a^2 \Lambda_i^2 +(c\alpha - 2 \beta a)\Lambda_i \right ) = 0.
	$$
	Since the equality constraint is a quadratic function in $l$, we can utilize the quadratic formula to solve for $l$. It can be verified that the discriminant to the above quadratic equation is non-negative (hence has real roots). To ensure that $l$ is positive, we require that 
	\begin{align}
			- c\Lambda_i - \alpha - \kappa &\le 0, \label{eq:psd_cond1}\\
			 a^2 \Lambda_i^2  - (c\alpha - 2 \beta a)\Lambda_i - \alpha \kappa + \beta^2  &\le 0 \label{eq:psd_cond2},
	\end{align}
	for all $\Lambda_i \in [-K, K]$. The constaints on $\Lambda_i$ come from the fact that $-KI \preceq H \preceq K I$ and the min-max theorem.
	
	 Since \eqref{eq:psd_cond1} is a linear function and \eqref{eq:psd_cond2} is a quadratic function in $\Lambda_i$, we only require that the inequalities are satisfied at the end points. Namely, we obtain the following conditions
	\begin{align*}
		\begin{cases}
			- cK - \alpha - \kappa \le 0 & \text{if }c \le 0, \\
			cK - \alpha - \kappa \le 0 & \text{otherwise.} \\
		\end{cases}\\
		a^2 K^2 - (c\alpha - 2 \beta a)K - \alpha \kappa + \beta^2  \le 0, \\
		a^2 K^2 + (c\alpha - 2 \beta a)K - \alpha \kappa + \beta^2  \le 0.
	\end{align*}
	This concludes the proof.
\end{proof} 

\section{Existence and Uniqueness of Strong solutions to \eqref{eq:mpd_flow}}
In this section, we show the existence and uniqueness of the McKean--Vlasov SDE \eqref{eq:mpd_flow} under Lipschitz \cref{ass:gradlip}. The structure is as follows:
\begin{enumerate}
    \item We begin by showing that $\cal{F}$ has Lipschitz $\theta$-gradients (\Cref{prop:f_theta_lipschitz}). 
    \item Then, we show that the drift, defined in \Cref{eq:mpd_flow_f2}, is Lipschitz (\Cref{prop:lipschitz_drift}).
    \item Finally, we prove the existence and uniqueness (\Cref{prop:existence_uniqueness}).
\end{enumerate}
In this section, we write the SDE \eqref{eq:mpd_flow} equivalently as
\begin{equation}
    \rm{d}(\vartheta, \Upsilon)_t = b(\Upsilon_t, \vartheta_t, \rm{Law}(\Upsilon_t)) \rm{d}t + \bm\beta \rm{d}W_t,
    \label{eq:mpd_flow_f2}
\end{equation}
where
$
\vartheta_t = ({\theta}_t, \tmo_t) \in \bb{R}^{2d_\theta}$,
$\Upsilon_t = ({X}_t, {\Qmo}_t) \in \bb{R}^{2d_x}$, $\bm{\beta} =\sqrt{2} \rm{Diag}([0_{2d_\theta}, 0_{d_{x}}, 1_{d_{x}}])$ (where $\rm{Diag}(x)$ returns a Diagonal matrix whose elements are given by $x$), and $
b: \bb{R}^{2d_x} \times \bb{R}^{2d_\theta} \times \cal{P}(\bb{R}^{2d_x}) \times \bb{R}^{2d_\theta} \rightarrow \bb{R}^{2d_\theta + 2d_x}$ defined as
$$
b((x, \qmo), (\theta, \tmo), q) =
\begin{pmatrix}
    \eta_\theta m\\
    -\gamma_\theta \eta_\theta m - \nabla_\theta \cal{F}(\theta, q) \\
    \eta_x u \\
    -\gamma_x \eta_x u + \nabla_x \ell (\theta, x)
\end{pmatrix}.
$$

We now prove that $\nabla_\theta \cal{F}$ is Lipschitz. Since the $\theta$-gradients do not depend on the momentum parameter $m$, we will drop the dependence on $m$ for brevity.
\begin{proposition}[$\cal{F}$ has Lipschitz $\theta$-gradient] Under \Cref{ass:gradlip}, we have that $\cal{F}$ is Lipschitz,i.e., there exist a constant $K_{\cal{F}} >0$ such that the following inequality holds:
    $$
    \left\Vert \nabla_{\theta}{\cal {\cal F}}\left(\theta,q\right)-\nabla_{\theta}{\cal {\cal F}}\left(\theta',q'\right)\right\Vert \le K_{\cal{F}}(\|\theta-\theta'\| + \sf{W}_1(q, q')),
    $$
    for all $\theta, \theta' \in \bb{R}^{d_\theta}$ and $q, q' \in \cal{P}(\mathbb{R}^{2d_x})$.
    \label{prop:f_theta_lipschitz}
\end{proposition}
\begin{proof}
From the definition, and adding and subtracting the same quantities and triangle inequality, we obtain
    \begin{align*}
    \left\Vert \nabla_{\theta}{\cal F}\left(\theta,q\right)-\nabla_{\theta}{\cal F}\left(\theta',q'\right)\right\Vert \le \|\nabla_{\theta} {\cal F}\left(\theta,q\right) - \nabla_{\theta} {\cal F}\left(\theta',q\right)\| +\|\nabla_{\theta} {\cal F}\left(\theta',q\right) - \nabla_{\theta} {\cal F}\left(\theta',q'\right)\|.
\end{align*}
We treat the terms on the RHS separately. For the first, from Jensen's inequality, we obtain
\begin{align*}
    \|\nabla_{\theta} {\cal F}\left(\theta,q\right) - \nabla_{\theta} {\cal F}\left(\theta',q\right)\| &= \left \|\int \nabla_{\theta} \ell (\theta,x) - \nabla_{\theta} \ell (\theta',x) \, q_X(\rm{d}x) \right \| \\
    &\le \int \left \|\nabla_{\theta} \ell (\theta,x) - \nabla_{\theta} \ell (\theta',x)\right \| \, q_X(\rm{d}x) \\
    &\le K\|\theta - \theta'\|.
\end{align*}
As for the other term, we have
\begin{align*}
    \|\nabla_{\theta} {\cal F}\left(\theta',q\right) - \nabla_{\theta} {\cal F}\left(\theta',q'\right)\| &\le \left \| \int \nabla_\theta \ell(\theta', x) (q -q')(x,u) \rm{d}x\rm{d}\qmo \right \| \\
    &{\le} K \int \frac{1}{K} \left \|\nabla_\theta \ell(\theta', x)\right \| |q -q'|(x,u) \rm{d}x\rm{d}\qmo \\
    &\overset{(a)}{\le} 2K \sf{W}_1(q, q').
\end{align*}
For $(a)$, we use the fact that $(x, u) \mapsto \|\nabla_\theta \ell (\theta', x)\|$ is $K$-Lipschitz (under \cref{ass:gradlip})
(and so the map $x \mapsto \frac{1}{K}\|\nabla_\theta \ell (\theta', x)\|$ is $1$-Lipschitz), from the dual representation of $\sf{W}_1$ we have as desired.

To conclude, by combining the two bounds, we have shown that $\nabla_\theta \cal{F}$ is Lipschitz with constant $K_{\cal{F}} = 2K$.
\end{proof}
We now prove that the drift of the SDE \eqref{eq:mpd_flow_f2} is Lipschitz.
\begin{proposition}[Lipschitz Drift]
\label{prop:lipschitz_drift}
    Under \Cref{ass:gradlip}, the drift $b$ is Lipschitz, i.e., there is some constant $K_b>0$ such that the following inequality holds:
    $$
    \|b(\Upsilon, \vartheta, q) - b(\Upsilon', \vartheta',q') \| \le K_b (\|\Upsilon - \Upsilon'\| + \|\vartheta-\vartheta'\| + \sf{W}_1(q,q')),
    $$
    for all $\vartheta, \vartheta' \in \mathbb{R}^{d_\theta}$, $\Upsilon, \Upsilon' \in \bb{R}^{2d_\theta} $, and $q, q' \in \cal{P}(\bb{R}^{2d_x})$.
\end{proposition}

\begin{proof}
We begin with the definition, apply the triangle inequality, and use the fact that from the concavity of $\sqrt{\cdot}$ we have $\sqrt{a+b} \le \sqrt{a} + \sqrt{b}$ for $a,b>0$ to obtain
\begin{align*}
    \|b(\Upsilon, \vartheta, q) - b(\Upsilon', \vartheta', q') \|
    &\le  \eta_{\theta} (1 + \gamma_\theta) \left\Vert \tmo - \tmo'\right\Vert + \|\nabla_\theta \cal{F}(\theta, q) - \nabla_\theta \cal{F}(\theta', q')\| \\
     &+ \eta_x(1+\gamma_x)\|\qmo-\qmo'\| + \|\nabla_x \ell(\theta, x) -\nabla_x \ell(\theta', x')\| \\
     &\le  \eta_{\theta} (1 + \gamma_\theta) \left\Vert \tmo - \tmo'\right\Vert + K_{\cal{F}}(\|\theta-\theta'\| + \sf{W}_1(q, q'))  \\
     &+ \eta_x(1+\gamma_x)\|\qmo-\qmo'\| + K(\|\theta - \theta'\| + \|x-x'\|) \\
     &\le \sqrt{2}\max\{\eta_{\theta} (1 + \gamma_\theta) ,K_\cal{F} + K\}\|\vartheta - \vartheta'\| \\
     &+\sqrt{2}\max\{\eta_x(1+\gamma_x), K, K_\cal{F} \}(\| \Upsilon - \Upsilon' \| + \sf{W}_1(q,q')),
\end{align*}
where we use the Lipschitz \cref{ass:gradlip} on $\ell$, and \Cref{prop:f_theta_lipschitz}. Hence we obtained as desired.
\end{proof}

\begin{proof}[Proof of \Cref{prop:existence_uniqueness}]
\label{proof:existence_uniqueness}
The proof is similar to \citet[Theorem 1.7]{carmona2016lectures} but with key generalizations to the product space.

Fix some $\nu \in C([0,T], \bb{R}^{2d_\theta} \times \cal{P}(\bb{R}^{2d_x}))$.
We denoted by $\nu^\vartheta$ and $\nu^\Upsilon$ the projection to the $\bb{R}^{2d_\theta}$ and $\cal{P}(\bb{R}^{2d_x})$ components respectively.
Consider substituting $\nu_t$ into \eqref{eq:mpd_flow_f2} in place of the $\rm{Law}(\Upsilon_t)$ and $\vartheta_t$, from \citet[Theorem 1.2]{carmona2016lectures}, we have existence and uniqueness of the strong solution for some initial point $(\vartheta,\Upsilon)_0$. More explicitly, we have
$$
(\vartheta, \Upsilon)^\nu_t = (\vartheta, \Upsilon)_0 + \int_0^tb(\Upsilon_t^\nu, \nu^\vartheta_s, \nu^\Upsilon_s)\rm{d}s + \int_0^t\beta\rm{d}W_t.
$$
for $t\in [0,T]$. 
We define the operator $F_T:C([0,T], \bb{R}^{2d_\theta}\times \cal{P}(\mathbb{R}^{2d_x})) \rightarrow C([0,T], \bb{R}^{2d_\theta} \times \cal{P}(\mathbb{R}^{2d_x}))$ as
$$
 F_T: \nu \rightarrow (t \mapsto (\vartheta^\nu_t, \rm{Law}(\Upsilon^\nu_t))).
$$
Clearly, if the process $(\vartheta, \Upsilon)_t$ is a solution to  \eqref{eq:mpd_flow_f2} then the function $t \mapsto (\vartheta_t, \rm{Law}(\Upsilon_t))$ is a fixed point to the operator $F_T$, and vice versa.
Now we establish the existence and uniqueness of the fixpoint of the operator $F_T$.

We begin by endowing the space $\bb{R}^{2d_\theta} \times \cal{P}(\bb{R}^{2d_x})$ with the metric:
\begin{equation*}
    \sf{d}((\vartheta, q),(\vartheta', q')) =  \sqrt{\|\vartheta - \vartheta'\|^2 + \sf{W}_2(q,q')^2}.
\end{equation*}
Note that the metric space $(\bb{R}^{2d_\theta} \times \cal{P}(\bb{R}^{2d_x}), d)$ is complete \citep{villani2009optimal}.

First note that using Jensen's inequality, $b$ is Lipschitz, and the fact that $(a+b+c+d)^2 \le 4(a^2 + b^2 + c^2+d^2)$ we obtain
\begin{align*}
    \|\vartheta_t^\nu-\vartheta_t^{\nu'}\|^2+\mathbb{E}[\|\Upsilon^\nu_t - \Upsilon^{\nu'}_t\|^2] &=  \mathbb{E} \left [\left \| \int_0^tb(\Upsilon^\nu_s,  \nu^\vartheta_s,\nu^\Upsilon_s) - b(\Upsilon^{\nu'}_s, {\nu'}^\vartheta_s, {\nu'}^\Upsilon_s) \rm{d}s \right\|^2 \right ] \\
    &\le t \int_0^t  \mathbb{E} \left [\left \|b(\Upsilon^\nu_s,  \nu^\vartheta_s, \nu^\Upsilon_s) - b(\Upsilon^{\nu'}_s, {\nu'}^\vartheta_s,  {\nu'}^\Upsilon_s)\right\|^2 \right ] \rm{d}s \\
    &\le tC\int_0^t \mathbb{E} \left [\left \|\Upsilon^{\nu}_s - \Upsilon^{\nu'}_s\right \|^2 \right ] +  \|\vartheta^\nu_s- \vartheta^{\nu'}_s\|^2+ \|{\nu}^\vartheta_s-{\nu'}^\vartheta_s\|^2 + \sf{W}^2_1({\nu}^\Upsilon_s, {\nu'}^\Upsilon_s)\rm{d}s,
\end{align*}
where $C=4K_b^2$. Applying Gr\"onwall's inequality, we obtain that
\begin{align*}
    \|\vartheta_t^\nu-\vartheta_t^{\nu'}\|^2+\mathbb{E}[\|\Upsilon^\nu_t - \Upsilon^{\nu'}_t\|^2] &\le C(t)\int_0^t \left[ \sf{W}^2_1({\nu}^\Upsilon_s, {\nu'}^\Upsilon_s) +  \|{\nu}^\vartheta_s-{\nu'}^\vartheta_s\|^2 
 \right ]\rm{d}s,
\end{align*}
where $C(t)=Ct\exp(Ct^2)$. Then, using the fact that the LHS is an upper bound for the squared distance $d$ and $\sf{W}_1 \le \sf{W}_2$ \citep[Remark 6.6]{villani2009optimal}, we have
\begin{align}
    \sf{d}^2(F_T(\nu)_t,F_T(\nu')_t)
 &\le C(t)\int_0^t \left[ \sf{W}^2_2({\nu}^\Upsilon_s, {\nu'}^\Upsilon_s) +  \|{\nu}^\vartheta_s-{\nu'}^\vartheta_s\|^2
 \right ]\rm{d}s \nonumber\\
 &\le C(T)\int_0^t \sf{d}^2(\nu_s, \nu'_s)\rm{d}s \le tC(T) \sup_{s \in [0,T]}\sf{d}^2(\nu_s, \nu'_s) \label{eq:induction_base_case},
\end{align}
where we use the fact that $C(t) \le C(T)$ since $t \in [0,T]$.
We show that for $k\ge 1$ successive compositions of the map $F_T$ denoted by $F_{T}^k$, we have the following inequality:
\begin{equation}
\sf{d}^2(F^k_T(\nu)_t,F^k_T(\nu')_t) \le \frac{(tC(T))^k}{k!} \sup_{u\in [0,T]}\sf{d}^2(\nu_u, \nu'_u).
\label{eq:induction_inequality}
\end{equation}
This can be proved inductively. The base case $k=1$ follows immediately from \eqref{eq:induction_base_case}, assume the inequality holds for $k-1$, then for $k$ we have
\begin{align*}
\sf{d}^2(F^k_T(\nu)_t,F^k_T(\nu')_t) &\le C(T) \int_0^t \sf{d}^2(F^{k-1}_T(\nu)_s,F^{k-1}_T(\nu')_s)\rm{d}s \\
&\le \frac{C(T)^k}{{(k-1)!}}\sup_{u\in [0,T]}\sf{d}^2(\nu_u, \nu'_u) \int_0^t {t^{k-1}} \rm{d}s \\
&\le \frac{(tC(T))^k}{{k!}}\sup_{u\in [0,T]}\sf{d}^2(\nu_u, \nu'_u).
\end{align*}
Hence, we have shown that the inequality \eqref{eq:induction_inequality} holds. Taking the supremum, we obtain
$$
\sup_{s \in [0,T]}\sf{d}^2(F^k_T(\nu)_s,F^k_T(\nu')_s)\le \frac{(TC(T))^k}{{k!}} \sup_{s \in [0,T]}\sf{d}^2(\nu_s, \nu'_s).
$$
Since $k \mapsto (TC(T))^k \in o(k!)$, there exists a large enough $k$ such that there is a constant $0 < \alpha < 1$ for the following inequality holding: 
$$
\sup_{s\in[0,T]}\sf{d}^2(F_T^k(\nu)_s, F_T^k(\nu')_s) \le \alpha \sup_{s\in[0,T]}\sf{d}^2(\nu_s, \nu_s').
$$
Hence, we have shown that the operator $F^k_T$ is a contraction and from the Banach Fixed Point theorem and completeness of the space $(C([0,T], \bb{R}^{d_\theta} \times \cal{P}(\bb{R}^{d_x})), \sup \sf{d})$, we have existence and uniqueness.
\end{proof}

\section{Space Discretization}

In this section, we establish asymptotic pointwise propagation of chaos results. We are interested in justifying the use of a particle approximation in the flow \eqref{eq:mpd_flow}. The flow \eqref{eq:mpd_flow}  can be rewritten equivalently as follows. Let $(\Upsilon_0^i, W^i)_{i\in[M]}$ be i.i.d.\ copies of $(\Upsilon_0, W)$. We write $(\vartheta, \{\Upsilon^{i}\}_{i=1}^M)$ as solutions of \eqref{eq:mpd_flow} starting at $(\vartheta_0, \{\Upsilon_0^i\}_{i=1}^M)$ driven by the $\{W^i\}_{i=1}^M$. In other words, \eqref{eq:mpd_flow} satisfies
\begin{subequations}
    \begin{align*}
        \rm{d}\vartheta_t &= b_\vartheta\left(\vartheta_t, \rm{Law}(\Upsilon^1_t)\right)\rm{d}t, \\
    \forall i \in [M]: \rm{d}\Upsilon^{i}_t &= b_\Upsilon(\Upsilon^{i}_t, \vartheta_t)\rm{d}t + \bm{\beta}_\Upsilon\,\rm{d}W^i_t.
    \end{align*}
\end{subequations}
where
$$
b_\vartheta(\vartheta_t, q ) = \begin{bmatrix}
\eta_\theta m \\
-\gamma_\theta \eta_\theta m - \nabla_\theta \cal{F}(\theta', q)
\end{bmatrix}, \quad 
b_\Upsilon(\Upsilon, \vartheta ) = \begin{bmatrix}
    \eta_x u \\
    -\gamma_x \eta_x u + \nabla_x \ell (\theta, x)
\end{bmatrix}, \quad
\bm{\beta}_\Upsilon := \sqrt{2} \rm{Diag}([0_{d_x}, 1_{d_x}]).
$$
Clearly, for all $i,j \in [M]$,
we have $\rm{Law}(\Upsilon^i_t) =\rm{Law}(\Upsilon^j_t)$.

We will justify that we can replace the $\rm{Law}(\Upsilon^1_t)$ with a particle approximation to obtain the approximate process  \eqref{eq:space_discretization}, or equivalently as:
\begin{subequations}
\begin{align}
    \rm{d}\vartheta^M_t &= b_\vartheta\left(\vartheta^M_t, \frac{1}{M}\sum_{i=1}^M \delta_{\Upsilon^{i,M}_t}\right)\rm{d}t \\
    \forall i \in [M]: \rm{d}\Upsilon^{i,M}_t &= b_\Upsilon(\Upsilon^{i,M}_t, \vartheta^M_t)\rm{d}t + \beta_\Upsilon \rm{d}W^i_t.
\end{align}
\end{subequations}
Similarly to \Cref{prop:lipschitz_drift}, we can show that $b_\vartheta$ and $b_\Upsilon$ are both Lipschitz.
We are now ready to prove \Cref{prop:chaos} justifying that \eqref{eq:space_discretization} is a good approximation to \eqref{eq:mpd_flow}.

\begin{proof}[Proof of \Cref{prop:chaos}]
    \label{proof:chaos}
    This is equivalent to showing that
    \begin{align*}
       \lim_{M\rightarrow \infty} \bb{E} \left [ \sup_{t \in [0,T]} \left  \{\|\vartheta_t - \vartheta^M_t \|^2 + \frac{1}{M} \sum_{i=1}^M\|\Upsilon^i_t - \Upsilon^{i,M}_t\|^2 \right \} \right ] = 0.
    \end{align*}
    More specifically, we will show that 
    \begin{align*}
       \underbrace{\bb{E} \left [ \sup_{t \in [0,T]} \|\vartheta_t - \vartheta^M_t \|^2 \right ]}_{(a)} + \underbrace{\bb{E} \left [\sup_{t \in [0,T]} \frac{1}{M} \sum_{i=1}^M\|\Upsilon^i_t - \Upsilon^{i,M}_t\|^2  \right ]}_{(b)} = o(1).
    \end{align*}
    We begin with (a). From Jensen's inequality, we have
    \begin{align*}
        (a)
        &\le T \mathbb{E} \int_0^T \left \|b_\vartheta(\vartheta_s, \rm{Law}(\Upsilon^1_s)) 
        - b_\vartheta(\vartheta^M_s, \frac{1}{M}\sum_{i=1}^M \delta_{\Upsilon^{i,M}_s})\right  \|^2 \rm{d}s\\
        &\le C_\vartheta \int_0^T \mathbb{E}\|\vartheta_s - \vartheta^M_s\|^2 + \mathbb{E}\sf{W}_2^2\left(\rm{Law}(\Upsilon^{0}_s), \frac{1}{M}\sum_{i=1}^M \delta_{\Upsilon^{i,M}_s}\right) \rm{d}s,
    \end{align*}
    where we use the fact that $b_\vartheta$ is $K_{b_\vartheta}$--Lipschitz, $C_\vartheta = 2 T K_{b_\vartheta}^2$, and $\mathbb{E}(a+b)^2 \le 2(\mathbb{E}a^2 + \mathbb{E}b^2)$ (known as the $C_r$--inequality \citep[p157]{Loève1977}).
    
    Note that by the triangle inequality
    \begin{align*}
        \mathbb{E}\sf{W}_2^2\left(\rm{Law}(\Upsilon^{0}_s), \frac{1}{M}\sum_{i=1}^M \delta_{\Upsilon^{i,M}_s}\right) &\le 2\mathbb{E} \sf{W}_2^2\left(\rm{Law}(\Upsilon^{0}_s), \frac{1}{M}\sum_{i=1}^M \delta_{\Upsilon^{i}_s}\right)\\
        &+ 2\mathbb{E}\sf{W}_2^2\left(\frac{1}{M}\sum_{i=1}^M \delta_{\Upsilon^{i}_s}, \frac{1}{M}\sum_{i=1}^M \delta_{\Upsilon^{i,M}_s}\right).
    \end{align*}
    The two terms can be bounded using \citet[Eq.\ (1.24) and Lemma 1.9]{carmona2016lectures} to produce the inequality
    \begin{align*}
        \mathbb{E}\sf{W}_2^2\left(\rm{Law}(\Upsilon^{0}_s), \frac{1}{M}\sum_{i=1}^M \delta_{\Upsilon^{i,M}_s}\right) &\le o(1) + \frac{2}{M}\sum_{i=1}^M \mathbb{E} \|\Upsilon^i_s - \Upsilon^{i,M}_s\|^2.
    \end{align*}
    Hence, we have that
    $$
    (a) \le C' \int_0^T \left \{ \mathbb{E}\|\vartheta_s - \vartheta^M_s\|^2 + o(1) + \frac{1}{M}\sum_{i=1}^M \mathbb{E} \|\Upsilon^i_s - \Upsilon^{i,M}_s\|^2 \right \} \rm{d}{s}.
    $$
    where $C'_\vartheta = 2C_\vartheta$. We use the fact that 
    \begin{equation}
        \|\vartheta_s - \vartheta^M_s\|^2 \le \sup_{s'\in [0,s]}\|\vartheta_{s'} - \vartheta^M_{s'}\|^2, \quad  \frac{1}{M}\sum_{i=1}^M\|\Upsilon^i_s - \Upsilon^{i,M}_s\|^2 \le \sup_{s'\in[0,s]}\frac{1}{M}\sum_{i=1}^M\|\Upsilon^i_{s'} - \Upsilon^{i,M}_{s'}\|^2,
        \label{eq:trick}
    \end{equation}
    to obtain
    \begin{equation}
        (a) \le C'\int_0^T \left \{ \mathbb{E}\sup_{s'\in [0,s]}\|\vartheta_{s'} - \vartheta^M_{s'}\|^2 + o(1) + \mathbb{E}\sup_{s'\in[0,s]}\frac{1}{M}\sum_{i=1}^M\|\Upsilon^i_{s'} - \Upsilon^{i,M}_{s'}\|^2 \right \} \rm{d}{s}
    \end{equation}
    Similarly, for (b), we have
    \begin{align*}
        (b) &\le\frac{T}{M} \sum_{i=1}^M \bb{E}  \int_0^T \|b_\Upsilon(\Upsilon^{i}_s, \vartheta_{s})-b_\Upsilon(\Upsilon^{i,M}_s, \vartheta_{s}^{M})\|^2\rm{d}s \\
        &\le C_\Upsilon\mathbb{E} \int^T_0 \|\vartheta_s - \vartheta_s^M\|^2 + \frac{1}{M} \sum_{i=1}^M \|\Upsilon_s^{i} - \Upsilon_s^{i,M}\|^2 
        \rm{d}s \\
        &\le C_\Upsilon\int^T_0 \mathbb{E}\sup_{s'\in [0,s]}\|\vartheta_{s'} - \vartheta^M_{s'}\|^2 + \bb{E}\sup_{s'\in[0,s]}\frac{1}{M}\sum_{i=1}^M\|\Upsilon^i_{s'} - \Upsilon^{i,M}_{s'}\|^2, 
        \rm{d}s,
    \end{align*}
    where $C_\Upsilon := 2TK_{b_\Upsilon}^2$ for the last line we use the trick in \eqref{eq:trick}.

    Combining the bounds for $(a)$ and $(b)$, we obtain
    $$
    (a) + (b) \le C \int_0^T \left \{ \mathbb{E}\sup_{s'\in [0,s]}\|\vartheta_{s'} - \vartheta^M_{s'}\|^2 + o(1) + \mathbb{E}\sup_{s'\in[0,s]}\frac{1}{M}\sum_{i=1}^M\|\Upsilon^i_{s'} - \Upsilon^{i,M}_{s'}\|^2 \right \} \rm{d}s,
    $$
    where $C=C_\Upsilon+C'_\vartheta$.
    Applying Gronwall's inequality, we obtain 
    $$
    \bb{E} \left [ \sup_{t \in [0,T]} \|\vartheta_t - \vartheta^M_t \| \right ] + \bb{E} \left [\sup_{t \in [0,T]} \frac{1}{M} \sum_{i=1}^M\|\Upsilon^i_t - \Upsilon^{i,M}_t\|  \right ] \le o(1).
    $$
    Taking the limit, we have as desired.
\end{proof}

\section{Time Discretization}

In this section, we are concerned with discretization schemes of various ODE/SDEs. The structure is as follows:
\begin{itemize}
	\item \Cref{sec:pgd_discetization}. We describe the Euler-Marayama discretization of PGD as described in \citet{Kuntz2022}.
	\item \Cref{sec:nag_discretization}. We show we can obtain NAG as a discretization of the damped Hamiltonian.
	\item \Cref{sec:cheng_nag_mpd}. We show a discretization of MPD (described in \Cref{eq:mpd_flow}) using a scheme replicating NAG as described in \Cref{sec:nag_discretization} for the $(\theta, \tmo)$-components, and \citet{cheng2018underdamped}'s for $(x, \qmo)$-components.
	\item  \Cref{sec:our_discretization}. We derive the transition using a scheme inspired by \citet{cheng2018underdamped} while incorporating NAG-style gradient correction as described in \citet{sutskever2013importance}.
\end{itemize}
\subsection{PGD discretization}
\label{sec:pgd_discetization}
In order to obtain an implementable system, it is standard to then discretise the distribution $q_t$ by representing it with a finite particle system, i.e. $q_t \left( \mathrm{d} x \right) \approx \frac{1}{M} \sum_{i \in \left[ M \right]} \delta \left( X_t^i, \mathrm{d} x \right)$. Upon making this approximation, one obtains the system
\begin{align*}
    \dot{\theta}_t &= \frac{1}{M} \sum_{i \in \left[ M \right]} \nabla_{\theta} \ell \left( \theta_t, X_t^i \right), \\
    \text{for } i \in \left[ M \right], \quad \mathrm{d} X_t^i &= \nabla_x \ell \left( \theta_t, X_t^i \right) \, \mathrm{d} t + \sqrt{2} \, \mathrm{d} W_t^i,
\end{align*}
in which all terms are readily available. Discretising this process in time then yields the Particle Gradient Descent (PGD) algorithm of \cite{Kuntz2022}, i.e. for $k \geq 1$, iterate
\begin{align*}
    \theta_k &= \theta_{k - 1} + h  \left( \frac{1}{M} \sum_{i=1}^M \nabla_\theta \ell \left( \theta_{k - 1}, X_{k - 1}^i \right) \right), \\
    \text{for } i \in [M], \quad X_k^i &=  X_{k - 1}^i + h \nabla_x \ell \left( \theta_{k - 1}, X_{k - 1}^i \right) + \sqrt{2 h} \epsilon^i_k,
\end{align*}
where $\epsilon^i_k \overset{\textrm{i.i.d.}}{\sim} \mathcal{N}(0_{d_x}, I_{d_x})$, for some initialization $(\theta_0, \{X_0^i\}_{i=1}^M)$.

\subsection{NAG as a discretization}
\label{sec:nag_discretization}

Recall the momentum-enriched ODE is given by
$$
\ddot{ \theta}_t + \gamma   \eta \dot{\theta}_t + \eta  \nabla_\theta  f(\theta_t) = 0.
$$
Let $m_t = \frac{1}{\eta} \dot{\theta}_t$, then we can write the above equivalently as the following coupled first-order ODEs:
\begin{align*}
	\dot{m}_t &= - \gamma  \eta m_t - \nabla f(\theta_t), \quad \dot{\theta_t} = \eta  m_t.
\end{align*}
When $\gamma= \frac{3}{t}$ and $\eta = 1$, we show that a particular discretization of the above 
is equivalent to Nesterov Accelerated Gradient (NAG) method \citep{nesterov1983method}. The argument is inspired by reversing the one of \citet{su2014differential} who obtained the continuous limit of NAG.

Recall that NAG \citep{nesterov1983method}, for convex $f$ (but not strongly convex), is defined by the following iteration:
\begin{align*}
	\theta_{k}&=y_{k-1}-h\nabla f\left(y_{k-1}\right),\quad y_{k}=\theta_{k}+\left (\frac{k-1}{k+2}\right )\left(\theta_{k}-\theta_{k-1}\right).
\end{align*}
Since $\frac{k-1}{k+2} \approx (1 - \frac{3}{k})$, we have
\begin{align}
	\theta_{k}&\approx y_{k-1}-h\nabla_\theta f\left(y_{k-1}\right),\quad y_{k}\approx \theta_{k}+\left ( 1-\frac{3}{k}\right )\left(\theta_{k}-\theta_{k-1}\right).
	\label{eq:approx_nag}
\end{align}
We will now show how a particular discretization scheme will produce \eqref{eq:approx_nag}.
It uses a combination of implicit Euler for $\dot{\theta}_t$, and use a semi-implicit Euler scheme for $\dot{m}_t$. More specifically, the semi-implicit scheme for $\dot{m_t} = -\frac{3}{t} m_t - \nabla_\theta f(\theta_t)$ uses an explicit approximation for the momentum $\frac{3}{t} m_t$, and implicit approximation for the gradient $\nabla_\theta f(\theta_t)$. In summary, we obtain the following discretization
\begin{align*}
	{m_{t}}
	& \approx m_{t-\sqrt{h}} - \sqrt{h} \left ( \frac{3}{t - \sqrt{h}} m_{t-\sqrt{h}} + \nabla_\theta f(\theta_{t})  \right ), 
	\quad {\theta_{t}} \approx \theta_{t-\sqrt{h}} + \sqrt{h} m_{t}.
\end{align*}
We can write the above equivalently through the map $t \mapsto \frac{t}{\sqrt{h}}$ and  in terms of $k:=\frac{t}{\sqrt{h}}$, to obtain
\begin{align*}
	{m_{k}} &\approx \left ( 1 - \frac{3}{k - 1} \right ) m_{k-1} -  \sqrt{h}\nabla_\theta f(\theta_{k}) , \\
	{\theta_{k}} &\approx \theta_{k-1} + \sqrt{h} m_{k}.
\end{align*}
Hence,
\begin{align*}
	{\theta_{k}} &\approx \theta_{k-1} + \sqrt{h} \left ( 1 - \frac{3}{k-1} \right ) m_{k-1} -  h\nabla_\theta f(\theta_{k}).
\end{align*}
From our discretization, we have that $\sqrt{h} m_{k} \approx \theta_{k} -\theta_{k-1} $ then we obtain
\begin{align*}
	{\theta_{k}} &\approx \theta_{k-1} + \left ( 1 - \frac{3}{k-1} \right ) \left ( \theta_{k-1} -\theta_{k-2} \right )-  h\nabla_\theta f(\theta_{k}).
\end{align*}
We can write the $\theta$ update in terms of $y_k$ (cf.~\eqref{eq:approx_nag}),
\begin{align*}
	{\theta_{k}} &\approx y_{k-1} -  h\nabla_\theta f(\theta_{k}).
\end{align*}
If $f$ is Lipschitz, then we have 
$$
\|h\nabla_\theta f(\theta_k) - h\nabla_\theta f(y_{k-1})\| \le Lh\|\theta_k - y_{k-1}\| \le h^2 L \|\nabla_\theta f(\theta_k)\| \le o(h).
$$
We replace the gradient $\nabla_\theta f(\theta_k)$ with $\nabla_\theta f(y_{k-1})$ to obtain as desired 
\begin{align*}
	{\theta_{k}} &\approx y_{k-1} -  h\nabla_\theta f(y_{k-1}), \quad 
	y_{k} \approx \theta_{k}+\left ( 1-\frac{3  }{k}\right )\left(\theta_{k}-\theta_{k-1}\right).
\end{align*}
Interestingly, for given step size $h$, the (approximate) NAG iterations approximate the flow for time $\sqrt{h}$ as opposed to $h$ \citep[Section 3.4]{su2014differential}

\subsection{Nesterov and Cheng's discretization of MPD}
\label{sec:cheng_nag_mpd}

In this section, we show how to discretize the MPD using a combination of Nesterov's (see \Cref{sec:nag_discretization}) and \citet{cheng2018underdamped}'s discretization. Recall, we have
\begin{align*}
	\rm{d}{\theta}_t &= \eta_\theta m_t \, \rm{d}t,\\
	\rm{d}{\tmo}_t &= -\nabla_\theta \cal{E} (\theta_t, q^M_{t,X})\,dt - \gamma_\theta \eta_\theta m_t\, \rm{d}t,\\
	\forall i \in [M]:\rm{d}{X}^i_t &= \eta_x {\Qmo^i_t} \rm{d}t, \\
	\forall i \in [M]:\rm{d}{\Qmo}^i_t &= \nabla_x \ell (\theta_t, X_t^i)\rm{d}t - \gamma_x \eta_x {\Qmo}^i_t\rm{d}t + \sqrt{2\gamma_x}\,\rm{d}W^i_t.
\end{align*}

In this section, we show how discretizing the $\theta$ component in the style of Nesterov (specifically, \citet{sutskever2013importance}'s formulation) and $q$ component in the style of \citet{cheng2018underdamped} obtains the MPD-NC (Nesterov--Cheng) algorithm. The MPD-NC algorithm is described as follows: given previous values $(\theta_k, v_k, \{X_k^i, \Qmo_k^i\}_{i=1}^M)$ and step-size $h>0$, we iterate
\begin{subequations}
	\begin{align}
		\label{eq:nc_algo}
		\theta_{k+1} &= \theta_k + v_k,\\
		v_{k+1} &= \mu v_{k} - h^2 \nabla_\theta \cal{E}(\theta_k +\mu v_{k}, q^M_{k,X}),\\
		\forall i \in [M]:  X^i_{k+1} &= X_k^{i} + \frac{1}{\gamma_x} \left [ ( 1- \omega_x(h))  {\Qmo}^i_{k} + \nabla_x \ell (\theta_{k+1}, X^i_{k}) \left ( h - \frac{1- \omega_x(h)}{\gamma_x\eta_x} \right ) \right ] + L_\Sigma^{XX} \xi^i_k, \\ 
		\forall i \in [M]: {\Qmo}^i_{k+1} &= \omega_x(h) {\Qmo}_k^i +  \frac{1- \omega_x(h)}{\gamma_x \eta_x}\nabla_x \ell (\theta_{k+1}, X^i_{k}) + L_\Sigma^{X\Qmo} \xi^i_k + L_\Sigma^{\Qmo\Qmo}\xi_k'^{i},
	\end{align}
\end{subequations}
for all $i \in [N]$, where $\mu := 1 - h \gamma_\theta \eta_\theta$, $\omega_x (t) := $; $L_\Sigma^{XX}, L_\Sigma^{X\Qmo}, L_\Sigma^{\Qmo\Qmo}$ is described in \eqref{eq:sampling_constants}, and $\{\xi^i_k, \xi'^i_k\} \overset{i.i.d.}{\sim} \cal{N}(0, I_{d_x})$. Each iteration corresponds to (approximately) solving \eqref{eq:mpd_flow} for time $h$.

In \Cref{sec:nesterov_aaflow},  we show how Nesterov's discretization can be used to produce the update
\begin{align*}
		\theta_{k+1} &= \theta_k + v_k,\\
	v_{k+1} &= \mu v_{k} - h^2 \nabla_\theta \cal{E}(\theta_k +\mu v_{k}, q^M_{k,X}).
\end{align*}
In \Cref{sec:transition_derivation}, given $(\theta_{k+1}, \{X_k^i, \Qmo_k^i\}_{i=1}^M)$, we show that the transition is described by
\begin{align*}
	\forall i \in [M]: X^i_{k+1} &= X_k^{i} + \frac{1}{\gamma_x} \left [ ( 1- \omega_x(h))  {\Qmo}^i_{k} + \nabla_x \ell (\theta_{k+1}, X^i_{k}) \left ( h - \frac{1- \omega_x(h)}{\gamma_x\eta_x} \right ) \right ] + L_\Sigma^{XX} \xi^i_k, \\ 
	\forall i \in [M]: {\Qmo}^i_{k+1} &= \omega_x(h) {\Qmo}_k^i +  \frac{1- \omega_x(h)}{\gamma_x \eta_x}\nabla_x \ell (\theta_{k+1}, X^i_{k}) + L_\Sigma^{X\Qmo} \xi^i_k + L_\Sigma^{\Qmo\Qmo}\xi_k'^{i}.
\end{align*}
\subsubsection{Discretization of $(\dot{\theta}_t, \dot{m}_t)$}
\label{sec:nesterov_aaflow}
This discretization follows similarly to that in \Cref{sec:nag_discretization}. Similarly, let $k: = \frac{t}{\sqrt{h}}$ with the time rescaling $t \mapsto \frac{t}{\sqrt{h}}$, then we consider the following discretization:
\begin{align*}
	m_{k} &= m_{k-1} - \sqrt{h} \left (\gamma_\theta   \eta_\theta  m_{k-1} + \nabla_\theta \cal{E}(\theta_k, q^M_{k-1,X}) \right ) \\
	&= \left ( 1 - \sqrt{h} \gamma_\theta \eta_\theta \right ) m_{k-1} - \sqrt{h} \nabla_\theta \cal{E}(\theta_k, q^M_{k-1,X}), \\ 
	\theta_{k} &= \theta_{k-1} + \sqrt{h} \eta_\theta  m_{k}, 
\end{align*}
Expanding $m_k$, we obtain
\begin{align*}
	\theta_{k} &= \theta_{k-1} + \sqrt{h}  \eta_\theta  \left ( 1 - \sqrt{h}  \gamma_\theta \eta_\theta \right )  m_{k-1} -  h  \nabla_\theta \cal{E}(\theta_k, q^M_{k-1,X}) \\
	&= \theta_{k-1} + \left ( 1 - \sqrt{h}  \gamma_\theta  \eta_\theta \right )  (\theta_{k-1} - \theta_{k-2}) - h \nabla_\theta \cal{E}(\theta_k, q^M_{k-1,X}),
\end{align*}
Let $v_t = \theta_k - \theta_{k-1}$, then we have
\begin{align*}
	\theta_{k} &= \theta_{k-1} + v_{k} \\
	v_{k} &= \bar{\mu}  v_{k-1} - h \nabla_\theta \cal{E}(\theta_k, q^M_{k-1,X}),
\end{align*}
where $\bar{\mu} := 1 - \sqrt{h}  \gamma_\theta  \eta_\theta $. Then, as before, using the approximation $h\nabla \cal{E}(\theta_{k}, q^M_{k-1,X}) \approx h \nabla {\cal{E}}(\theta_{k-1} + \bar{\mu} v_{k-1}, q^M_{k-1,X}) + o(h)$, we obtain that
\begin{align*}
	\theta_{k} &= \theta_{k-1} + v_{k}, \\
	v_{k} &= \bar{\mu} v_{k-1} - h \nabla_\theta  \cal{E}(\theta_{k-1} + \bar{\mu} v_{k-1}, q^M_{k-1,X}).
\end{align*}
Similarly to \Cref{sec:nag_discretization} this approximates the flow for time $\sqrt{h}$ instead of $h$.
Hence, we have to apply the appropriate rescaling to obtain the following iteration with the desired behaviour:
\begin{align*}
	\theta_{k} &= \theta_{k-1} + v_{k}, \\
	v_{k} &= {\mu} v_{k-1} - h^2 \nabla_\theta  \cal{E}(\theta_{k-1} + \mu v_{k-1}, q^M_{k-1,X}),
\end{align*}
where ${\mu} :=  1 - h \gamma_\theta \eta_\theta$. This is exactly that of \citet{sutskever2013importance}'s characterization of NAG (see Equations 3 and 4 in their paper).
\subsubsection{Discretization of $(\dot{X}_t, \dot{\Qmo}_t)$}
\label{sec:transition_derivation}
We describe the discretization scheme of \citet{cheng2018underdamped}. For simplicity, we derive the transition of a single particle $(X_t, \Qmo_t)$ since there are no interactions between the particles given $\theta$. Furthermore, for brevity and without loss in generality, we derive the transition for some step size $h>0$ given initial values $(\theta, X_0, U_0)$, which can be easily generalized to future transitions.

Consider a time interval $t \in (0, h]$ and given $(X_{0}, \Qmo_{0})$, we first approximate the gradient $\nabla_x \ell (\theta, X_t)$ with $\nabla_x \ell (\theta ,X_{0})$ to arrive at the following linear SDE:
\begin{align}
	\rm{d}\begin{pmatrix} X_t \\
		{\Qmo}_t
	\end{pmatrix} 
	&= \left [
	\begin{pmatrix}
		0 \\
		\nabla_x \ell (\theta, X_0)
	\end{pmatrix}
	+
	\begin{pmatrix}
		0 & \eta_x I \\
		0 & -\gamma_x\eta_x I
	\end{pmatrix}
	\begin{pmatrix}
		X_t \\
		{\Qmo}_t
	\end{pmatrix} \right ] \,\rm{d}t
	+
	\sqrt{2\gamma_x}
	\begin{pmatrix}
		0 \\
		1
	\end{pmatrix} \,\rm{d}W_t.
	\label{eq:ula_as_linearsde}
\end{align}
A $2d_x$-dimensional linear SDE is given by:
\begin{equation}
	\rm{d}\bm{X}_t = (\bm{A}\bm{X}_t + \bm{\alpha})\, \rm{d}t + \bm{\beta} \,\rm{d}W_t
	\label{eq:linear_sde},
\end{equation}
where $\bm{A} \in \mathbb{R}^{2d_x\times 2d_x}$ and $\bm{\alpha}, \bm\beta \in \mathbb{R}^{2d_x}$ are  fixed matrices. It is clear that if we set
\begin{equation*}
	\bm{X}_t =
	\begin{pmatrix}
		X_t \\
		{\Qmo}_t
	\end{pmatrix},\quad 
	\bm{A} =
	\begin{pmatrix}
		0_{d_x} & \eta_x I_{d_x} \\
		0_{d_x} & -\gamma_x\eta_x I_{d_x}
	\end{pmatrix}, \quad \bm{\alpha} =
	\begin{pmatrix}
		0_{d_x \times 1} \\
		\nabla_x \ell (\theta , X_{0})
	\end{pmatrix}, \quad \bm{\beta} = \sqrt{2\gamma_x}
	\begin{pmatrix}
		0_{d_x \times 1} \\
		1_{d_x\times 1}
	\end{pmatrix},
\end{equation*}
then the discretized underdamped Langevin SDE of \eqref{eq:ula_as_linearsde} falls within the class of linear SDE of the form specified in \eqref{eq:linear_sde}. These SDEs
admits the following explicit solution \citep[see pages 48, 101]{platen2010numerical},
\begin{equation}
	\bm{X}_t = \bm{\Psi}_t \left (\bm{X_0} + \int_0^t \bm{\Psi}^{-1}_s \bm{\alpha} \,\rm{d}s + \int_0^t \bm{\Psi}^{-1}_s\bm{\beta}\,\rm{d}W_s\right ),
	\label{eq:linearsde_sol}
\end{equation}
where 
$$
\bm{\Psi}_t := \exp \left (\bm{A}t \right ),
$$
with $\exp$ to be understood as the matrix exponential. In our case, the matrix exponential and its inverse is given by
\begin{subequations}
	\begin{align}
		\bm{\Psi}_t &=I_{2d_x}+\sum_{i=1}^\infty \frac{1}{k!}\begin{bmatrix}
			0_{d_x} & \frac{(-\gamma_x\eta_xt)^k}{-\gamma_x} I_{d_x} \\
			0_{d_x} & (- \gamma_x \eta_x t)^k I_{d_x}
		\end{bmatrix} = 
		 \begin{pmatrix}
			I_{d_x} & \frac{1- \omega_x (t)}{\gamma_x} I_{d_x} \\
			0_{d_x} & \omega_x (t) I_{d_x}
		\end{pmatrix},\\ \bm{\Psi}^{-1}_t &= \begin{pmatrix}
			I_{d_x} & \frac{1- \omega_x (-t) }{\gamma_x}I_{d_x} \\
			0_{d_x} & \omega_x (-t) I_{d_x}
		\end{pmatrix},
	\end{align}
	\label{eq:fundamental_matrx}
\end{subequations}
where $\omega_x (t) := \exp(-\gamma_x\eta_x t)$.
It can be verified that they are indeed the inverse of each other, i.e., $\bm{\Psi}_t \bm{\Psi}^{-1}_t = I_{2d_x}$. 

Plugging in \eqref{eq:fundamental_matrx} into \eqref{eq:linearsde_sol}, we obtain the following:
\begin{align*}
	\begin{pmatrix} X_t \\
		{\Qmo}_t
	\end{pmatrix}
	&= \begin{pmatrix}
		I_{d_x} & \frac{1- \omega_x (t)}{\gamma_x}I_{d_x} \\
		0_{d_x} & \omega_x (t) I_{d_x}
	\end{pmatrix} \left ( \begin{bmatrix} X_0 \\
		{\Qmo}_0 
	\end{bmatrix} + \int_0^t \begin{pmatrix}
		\frac{1- \omega_x (-s) }{\gamma_x}\nabla_x \ell (\theta, X_{0}) \\
		\omega_x (-s) \nabla_x \ell (\theta, X_{0})
	\end{pmatrix} \,\rm{d}s  + \sqrt{2} \int_0^t \begin{pmatrix}
		{\frac{1- \omega_x (-s)}{\sqrt{\gamma_x}}}  1_{d_x \times 1} \\
		\sqrt{\gamma_x}\omega_x (-s) 1_{d_x \times 1}
	\end{pmatrix}\,\rm{d}W_s\right ).
\end{align*}
Or, equivalently, we have
\begin{align*}
	{\Qmo}_t &= \omega_x (t) {\Qmo}_{0} +  \nabla_x \ell(\theta, X_{0}) \int_0^t \omega_x (t-s) \,\rm{d}s +\sqrt{2\gamma_x}\int_0^t \omega_x (t-s) \,\rm{d}W_s \\
	&= \omega_x (t) {\Qmo}_{0} +  \frac{1- \omega_x (t)}{\gamma_x\eta_x}\nabla_x \ell (\theta ,X_{0})  +\sqrt{2\gamma_x}\int_0^t \omega_x (t-s) \,\rm{d}W_s,  \\
	X_t &= X_0 + \frac{1}{\gamma_x} \left [ ( 1- \omega_x (t))  {\Qmo}_0 + \nabla_x \ell (\theta, X_{0}) \left [ \int_0^t 1- \omega_x (t-s) \,\rm{d}s \right ]  + \sqrt{\frac{2}{\gamma_x}}\int_0^t 1 - \omega_x (t-s) \,\rm{d}W_s \right  ]  \\
	&= X_0 + \frac{1}{\gamma_x} \left [ ( 1- \omega_x (t))  {\Qmo}_0 + \nabla_x \ell (\theta, X_{0}) \left [ t - \frac{1- \omega_x (t)}{\gamma_x \eta_x} \right ]  + \sqrt{\frac{2}{\gamma_x}}\int_0^t 1- \omega_x (t-s)\,\rm{d}W_s \right  ].
\end{align*}

As noted by \citet{cheng2018underdamped}, this is a Gaussian transition. To characterize it, we need to calculate the first two moments:

For the first moments, we have
\begin{align*}
	\mu_{\Qmo}(X_0, \Qmo_0, t) :=& \mathbb{E}[{\Qmo}_t] \\
	=& \omega_x(t) {\Qmo}_0 +  \frac{1- \omega_x(t)}{\gamma_x\eta_x}\nabla_x \ell(\theta, X_{0}),  \\
	\mu_X(X_0, \Qmo_0, t)  :=& \mathbb{E}[X_t] \\
	=&  X_0 + \frac{1}{\gamma_x} \left [ ( 1- \omega_x(t))  {\Qmo}_0 + \nabla_x \ell (\theta, X_{0}) \left [ t - \frac{1- \omega_x (t)}{\gamma_x\eta_x} \right ] \right ].
\end{align*}
For the second moments, we have
\begin{align*}
	\Sigma_{\Qmo \Qmo}(t) &:= \mathbb{E}[({\Qmo}_t - \mathbb{E}[{\Qmo}_t])( {\Qmo}_t - \mathbb{E}[{\Qmo}_t])^\top] \\
	&= 2\gamma_x\mathbb{E} \left [ \left (\int_0^t \omega_x(t-s) \,\rm{d}W_s \right ) \left ( \int_0^t \omega_x(t-s)\,\rm{d}W_s\right )^\top  \right ] \\
	&= 2\gamma_x\left (\int_0^t \omega_x(2(t-s))\,\rm{d}s \right )  I_{d_x} \\
	&= \left (\frac{1-\omega_x(2t)}{\eta_x } \right ) I_{d_x}, \\
	\Sigma_{XX}(t) &:= \mathbb{E}[(X_t - \mathbb{E}[X_t])( X_t - \mathbb{E}[{X_t}])^\top] \\
	&= \frac{2}{\gamma_x}  \mathbb{E} \left [ \left (\int_0^t 1 - \omega_x (t-s)\,\rm{d}W_s \right ) \left ( \int_0^t 1- \omega_x(t-s) \,\rm{d}W_s\right )^\top  \right ] \\
	&= \frac{2}{\gamma_x} \left (\int_0^t \left [ 1-\omega_x(t-s) \right ]^2 \,\rm{d}s \right )  I_{d_x} \\
	&= \frac{1}{\gamma_x}  \left [2t -\frac{\omega_x(2t)}{\gamma_x\eta_x } +  \frac{4\omega_x(t)}{\gamma_x\eta_x } - \frac{3}{\eta_x \gamma_x} \right ] I_{d_x}, \\
	\Sigma_{\Qmo X}(t) &:= \mathbb{E}[({\Qmo}_t - \mathbb{E}[{\Qmo}_t])( X_t - \mathbb{E}[{X_t}])^\top] \\
	&= 2 \mathbb{E}\left [ \left (\int_0^t \omega_x(t-s) \,\rm{d}W_s \right )\left (\int_0^t 1 - \omega_x(t-s)\,\rm{d}W_s \right )^\top \right ] \\
	&=2\left (\int_0^t \omega_x(t-s) (1 - \omega_x(t-s))\,\rm{d}s \right )  I_{d_x}\\
	&=\frac{1}{\gamma_x\eta_x}  \left (  1- 2 \omega_x(t) + \omega_x(2t)  \right )  I_{d_x}.
\end{align*}
Hence, the transition can be described as 
$
\begin{pmatrix} X_t \\
	{\Qmo}_t
\end{pmatrix}
\sim \mathcal{N}\left (
\mu(t),
\Sigma(t)
\right )
$
where
$$
\mu(t) = \begin{pmatrix} \mu_X (X_k, \Qmo_k, t) \\
	\mu_{\Qmo}(X_k, \Qmo_k, t) 
\end{pmatrix}, \quad
\Sigma (t) =
\begin{pmatrix}
	\Sigma_{XX}(t) & \Sigma_{\Qmo X}(t)  \\
	\Sigma_{\Qmo X}(t) & \Sigma_{\Qmo \Qmo}(t)
\end{pmatrix}.
$$
Therefore, given some point $(X_k, \Qmo_k)$ and some step-size $h>0$, we have that 
\begin{align}
	\begin{pmatrix} X_{k+1} \\
		{\Qmo}_{k+1}
	\end{pmatrix} \sim \mathcal{N}\left (
	\mu (h), \Sigma (h)\right ).
	\label{eq:transition}
\end{align}
\textbf{Sampling from the transition}. Using samples from a standard Gaussian ${z} \sim \mathcal{N}(0_d, {I_d})$, one may produce samples from of a multivariate distribution $\mathcal{N}({\mu}, {\Sigma})$ by using the  fact that
$$
{X} = {\mu} + {L} {z} \sim \mathcal{N}({\mu}, {\Sigma}),
$$
where ${L}$ is a lower triangular matrix with positive diagonal entries obtained via the Cholesky decomposition of ${\Sigma}$, i.e, ${\Sigma} = {L}{L}^\top$.

The Cholesky decomposition of the covariance matrix of \eqref{eq:transition} will be described here. Consider the LDL decomposition of a symmetric matrix
$$
\Sigma = \begin{bmatrix}
	A & B \\
	B^\top & C
\end{bmatrix} = \begin{bmatrix}
	I & 0 \\
	B^\top A^{-1} & I
\end{bmatrix}
\begin{bmatrix}
	A & 0 \\
	0 & S
\end{bmatrix}
\begin{bmatrix}
	I & A^{-1} B\\
	0  & I
\end{bmatrix},
$$
where $S = D-B^\top A^{-1}B$ is the Schur complement. Given Cholesky factorization of $A$ and $B$, written as $A = L_A L_A^\top$ and $S = L_S L_S^\top$ respectively, we can write the Cholesky decomposition of $\Sigma$ as 
$$
\Sigma = \begin{bmatrix}
	L_A & 0 \\
	B^\top L_A^{-\top } & L_S
\end{bmatrix}
\begin{bmatrix}
	L_A^\top & L_A^{-1} B\\
	0  & L^\top_S
\end{bmatrix} = L_\Sigma L_\Sigma^\top.
$$
In our case, we compute the Cholesky decomposition of $\Sigma$ as follows:
$$
L_\Sigma =
\begin{pmatrix}
	L_\Sigma^{XX} \, I_{d_x}
	& 0 \\
	L_\Sigma^{X\Qmo}\, I_{d_x} &
	L_\Sigma^{\Qmo\Qmo}\, I_{d_x}
\end{pmatrix},
$$
where the constants are defined as
\begin{subequations}
	\label{eq:sampling_constants}
	\begin{align}
		L_\Sigma^{XX} &= \sqrt{\frac{1}{\gamma_x}  \left [2h -\frac{\omega_x (2h)}{\gamma_x\eta_x } +  \frac{4\omega_x (h)}{\gamma_x\eta_x } - \frac{3}{\eta_x \gamma_x} \right ]}, \\
		L_\Sigma^{X\Qmo} &=  \frac{\frac{1}{\gamma_x\eta_x}  \left (  1- 2 \omega_x (h) + \omega_x (2h)  \right )}{\sqrt{\frac{1}{\gamma_x}  \left [2h -\frac{\omega_x (2h)}{\gamma_x\eta_x } +  \frac{4\omega_x (h)}{\gamma_x\eta_x } - \frac{3}{\eta_x \gamma_x} \right ]}}, \\
		L_\Sigma^{\Qmo\Qmo } &= \sqrt{\frac{1-\omega_x (2h)}{\eta_x } - \frac{\left ( \frac{1}{\gamma_x\eta_x}  \left (  1- 2 \omega_x (h) + \omega_x (2h)  \right ) \right )^2}{\frac{1}{\gamma_x}  \left [2h -\frac{\omega_x (2h)}{\gamma_x\eta_x } +  \frac{4\omega_x (h)}{\gamma_x\eta_x } - \frac{3}{\eta_x \gamma_x} \right ]}}.
	\end{align}
\end{subequations}
Therefore, we the transition can be written as follows:
\begin{align*}
	\forall i \in [M]: X^i_{k+1} &= X_k^{i} + \frac{1}{\gamma_x} \left [ ( 1- \omega_x (h))  {\Qmo}_k + \nabla_x \ell (\theta, X^i_{k}) \left [ h - \frac{1-\omega_x (h)}{\gamma_x\eta_x} \right ] \right ] + L_\Sigma^{XX} \xi^i_k, \\
	\forall i \in [M]:  {\Qmo}^i_{k+1} &= \omega_x (h) {\Qmo}_k^i +  \frac{1-\omega_x (h)}{\gamma_x\eta_x}\nabla_x \ell (\theta, X^i_{k}) + L_\Sigma^{X \Qmo} \xi^{i}_k + L_\Sigma^{\Qmo\Qmo}\xi^{\prime, i}_k,
\end{align*}
where $\{\xi^i_k, \xi^{\prime, i}_k\}_{i\in [M], k} \overset{\text{i.i.d.}}{\sim} \mathcal{N}(0_{d_x}, I_{d_x})$.

\subsection{Proposed discretization of \eqref{eq:approximate_sde}}
\label{sec:our_discretization}
Recall, our approximating SDE in \eqref{eq:approximate_sde} is given by:
\begin{align*}
	\mathrm{d} \tilde{\theta}_t &=\eta_\theta \tilde{m}_t \, \mathrm{d} t \\
	\mathrm{d} \tilde{m}_t &= - \gamma_\theta \eta_\theta  \tilde{m}_t \, \mathrm{d} t - \nabla_\theta \cal{E} \left( \bar{\theta}_0, \tilde{q}^M_{0,X} \right) \, \mathrm{d} t \\
	\text{for } i \in [M], \quad \mathrm{d} \tilde{X}_t^i &= \eta_x \tilde{U}_t^i \, \mathrm{d} t\\
	\text{for } i \in [M], \quad \mathrm{d} \tilde{U}_t^i &= -\gamma_x \eta_x \tilde{U}_t^i \, \mathrm{d} t + \nabla_{x} \ell \left( \bar{\bar{\theta}}_0, \tilde{X}_0^i \right) \, \mathrm{d} t + \sqrt{2 \gamma_{x}} \, \mathrm{d} W_t^i.
\end{align*}

For simplicity and without loss in generality, we assume there is a single particle, i.e., $M=1$, and derive the transition for a single step. Clearly, \eqref{eq:approximate_sde} can be written as an linear $2d_x + 2d_\theta$ SDE:
\begin{equation*}
	\rm{d}\bm{X}_t = (\bm{A}\bm{X}_t + \bm{\alpha})\,\mathrm{d}t + \bm{\beta}\, \mathrm{d}W_t,
\end{equation*}
where
\begin{align*}
\bm{X}_t
=
\begin{pmatrix}
	\tilde{\theta}_t \\
	\tilde{\tmo}_t \\
	\tilde{X}_t \\
	\tilde{\Qmo}_t
\end{pmatrix}, \quad
\bm{A} =
\begin{pmatrix}
\begin{matrix}
	0_{d_\theta} & \eta_\theta I_{d_\theta} \\
	0_{d_\theta} & -\eta_\theta \gamma_\theta I_{d_\theta}
\end{matrix} &
0_{2d_\theta \times 2d_x} \\
0_{2d_x \times 2d_\theta} & \begin{matrix}
	0_{d_x} & \eta_x I_{d_x} \\
	0_{d_x} & -\eta_x \gamma_x  I_{d_x}
\end{matrix}
\end{pmatrix},
\\
\bm{\alpha} = \begin{pmatrix}
0_{d_\theta \times 1} \\
-\nabla_\theta \cal{E}(\bar{\theta}_0, \tilde{q}^M_{0,X}) \\
0_{d_x \times 1} \\
\nabla_{x} \ell (\bar{\bar{\theta}}_0, \tilde{X}_0) \\
\end{pmatrix}, \quad
\bm{\beta} = 
\begin{pmatrix}
0_{d_\theta \times 1} \\
0_{d_\theta \times 1} \\
0_{d_x \times 1} \\
1_{d_x \times 1} 
\end{pmatrix}.
\end{align*}
Hence it admits the following explicit solution \citep[see page 48, 101]{platen2010numerical}:
\begin{equation*}
	\bm{X}_t = \bm{\Psi}_t \left (\bm{X_0} + \int_0^t \bm{\Psi}^{-1}_s \bm{\alpha} \,\mathrm{d}s + \int_0^t \bm{\Psi}^{-1}_s\bm{\beta}\,\mathrm{d}W_s\right ),
\end{equation*}
where 
$$
\bm{\Psi}_t := \exp \left (\bm{A}t \right ),
$$
with $\exp$ to be understood as the matrix exponential. In our case, similarly to \eqref{eq:fundamental_matrx}, we have
\begin{align*}
	\bm{\Psi}_t=\begin{pmatrix}
		\begin{matrix}
			I_{d_\theta} & \frac{1- \omega_\theta (t)}{\gamma_\theta} I_{d_\theta} \\
			0_{d_\theta} & \omega_\theta (t) I_{d_\theta}
		\end{matrix} & 0_{2d_\theta \times 2d_x} \\
		0_{2d_x \times 2d_\theta} & \begin{matrix}
			I_{d_x} & \frac{1- \omega_x (t)}{\gamma_x} I_{d_x} \\
			0_{d_x} & \omega_x (t) I_{d_x}
		\end{matrix} \\
	\end{pmatrix},\\ 
	\bm{\Psi}^{-1}_t =
	\begin{pmatrix}
		\begin{matrix}
			I_{d_\theta} & \frac{1-\omega_\theta (-t)}{\gamma_\theta} I_{d_\theta} \\
			0_{d_\theta} & \omega_\theta (-t) I_{d_\theta}
		\end{matrix} & 0_{2d_x \times 2d_\theta} \\
		0_{2d_\theta \times 2d_x} & \begin{matrix}
			I_{d_x} & \frac{1-\omega_x (-t)}{\gamma_x} I_{d_x} \\
			0_{d_x} & \omega_x (-t) I_{d_x}
			\end{matrix}
	\end{pmatrix},
\end{align*}
where $\omega_x(t) := \exp(-\gamma_x \eta_x t)$, and similarly $\omega_\theta(t):= \exp(-\gamma_\theta \eta_\theta t)$. Hence, we have
\begin{align*}
	\tilde{\theta}_t &= \tilde{\theta}_0 + \frac{1}{\gamma_\theta} \left [ (1-\omega_\theta(t))\tilde{m}_0 - \left (t - \frac{1-\omega_\theta(t)}{\gamma_\theta \eta_\theta}\right )\nabla_\theta \cal{E}(\bar{\theta}_0, \tilde{q}^M_{0,X}) \right ], \\ 
	\tilde{\tmo}_t &= \omega_\theta (t) \tilde{\tmo}_0 - \frac{1- \omega_\theta(t)}{\gamma_\theta \eta_\theta} \nabla_\theta \cal{E}(\bar{\theta}_0, \tilde{q}^M_{0,X}), \\
	\forall i \in [M]: \tilde{X}_{t}^i &= \tilde{X}^i_0 + \frac{1}{\gamma_x} \left [ ( 1- \omega_x (t))  \tilde{\Qmo}^i_0 +  \left ( t - \frac{1-\omega_x (t)}{\gamma_x\eta_x} \right ) \nabla_x \ell (\bar{\bar{\theta}}_0, \tilde{X}^i_0) \right ] + L_\Sigma^{XX} \xi, \\
	\forall i \in [M]: \tilde{\Qmo}_{t}^i &= \omega_x (t) \tilde{\Qmo}^i_0 +  \frac{1-\omega_x (t)}{\gamma_x\eta_x}\nabla_x \ell (\bar{\bar{\theta}}_0, \tilde{X}^i_0) + L_\Sigma^{X \Qmo} \xi + L_\Sigma^{\Qmo\Qmo}\xi^\prime,
\end{align*}
where the constants are exactly the same as those from \Cref{eq:sampling_constants}, and $\xi,\xi^\prime \sim \mathcal{N}(0,I_{d_x})$.

It can be seen that taking setting $t=h$, the general iterations are given by
\begin{align*}
	\tilde{\theta}_{k+1} &= \tilde{\theta}_k + \frac{1}{\gamma_\theta} \left [ (1-\omega_\theta(k))\tilde\tmo_k - \left (t - \frac{1-\omega_\theta(k)}{\gamma_\theta \eta_\theta}\right )\nabla_\theta \cal{E}(\bar{\theta}_k,\tilde{q}^M_{k,X}) \right ], \\ 
	\tilde{\tmo}_{k+1} &= \omega_\theta (k) \tilde{\tmo}_k - \frac{1- \omega_\theta(k)}{\gamma_\theta \eta_\theta} \nabla_\theta {\cal{E}}(\bar{\theta}_k, \tilde{q}^M_{k,X}), \\
	\forall i \in [M]: \tilde{X}_{k+1}^i &= \tilde{X}^i_k + \frac{1}{\gamma_x} \left [ ( 1- \omega_x (k))  \tilde{\Qmo}^i_k +  \left ( t - \frac{1-\omega_x (k)}{\gamma_x\eta_x} \right ) \nabla_x \ell (\bar{\bar{\theta}}_{k}, \tilde{X}_k^i) \right ] + L_\Sigma^{XX} \xi^i_k, \\
	\forall i \in [M]: \tilde{\Qmo}_{k+1}^i &= \omega_x (k) \tilde{\Qmo}_k^i +  \frac{1-\omega_x (k)}{\gamma_x\eta_x}\nabla_x \ell (\bar{\bar{\theta}}_{k}, \tilde{X}^i_k) + L_\Sigma^{X \Qmo} \xi_k^i + L_\Sigma^{\Qmo\Qmo}\xi_k^{\prime i},
\end{align*}
where $\{\xi^i_k,\xi^{\prime i}_k\}_{k, i \in [M]} \sim \mathcal{N}(0,I_{d_x})$.

\section{Practical Concerns and Experimental Details}

Here, we describe practical considerations and experiment details. First, in \Cref{sec:rmsprop_preconditioner}, we describe the RMSProp precondition \citep{tieleman2012lecture,staib2019escaping} used in our experiments. Then, in \Cref{appendix:subsampling}, we describe the subsampling procedure used for our image generation task. After, in \Cref{appen:heuristic}, we discuss the heuristic we introduced for choosing momentum parameters $(\gamma_\theta, \gamma_x, \eta_\theta, \eta_x)$. Finally, in \Cref{appendix:exp_details}, we detail all the parameters and models used in our experiment for reproducibility.

\subsection{RMSProp Preconditioner}
\label{sec:rmsprop_preconditioner}
Given $0<\beta<1$, the RMSProp preconditioner $G_k$ \citep{tieleman2012lecture,staib2019escaping} at iteration $k$ is defined by the following iteration:
\begin{equation*}
	G_{k+1} = \beta G_{k} + (1-\beta) \nabla_\theta \cal{E}(\bar{\theta}_k, q^M_{k,X})^2,
\end{equation*}
where $G_{0} = 0_{d_\theta}$ and $x^2$ denotes element-wise square.

\subsection{Subsampling}
\label{appendix:subsampling}
In the presence of a large dataset, it is common to develop computationally cheaper implementations by appropriate using of data sub-sampling. Such a scheme was devised in \citet[Appendix E.4]{Kuntz2022}. We utilize the same subsampling scheme for $(\theta, \tmo)$-components. However, we found it necessary to devise a new scheme for $(x, \qmo)$-components. We found that this alternative subsampling scheme substantially improved the performance of MPD but did not noticeably affect the PGD algorithm.

It is desirable to update the whole particle cloud. However, in cases where each sample has its own posterior cloud approximation, as in the generative modelling task, the dataset can become prohibitively large for updating the whole cloud. In \citet{Kuntz2022}, the subsampling scheme only updated the portion of the particle cloud associated with the mini-batch, which neglects the remainder of the posterior approximation. As such, we propose to update the subsampled cloud at the beginning of each iteration at a step proportional to the ``time''/steps it has missed. This is described more succinctly in \Cref{alg:subsampled_alg}.

\begin{algorithm}
\caption{A single subsampled step. In \textcolor{mypink2}{pink}, we indicate the existing subsampling scheme of \citet{Kuntz2022}. We indicate our proposed additions in \textcolor{teal}{teal}.}\label{alg:subsampled_alg}
\begin{algorithmic}
\REQUIRE Subsampled indices $\cal{I}$, subsampled data $\{y^{i}\}_{i\in \cal{I}}$, step-size $h$, previous iterates $(\theta_k, \tmo_k, \{X^i_k, \Qmo_k^i\}_{i\in \cal{I}})$
\STATE // Updated cloud for missed time
\STATE \textcolor{teal}{$\{X^i_{k+\epsilon}, \Qmo^i_{k+\epsilon}\}_{i\in \cal{I}} \gets$ Solve \Cref{eq:x_approx_sde,eq:qmo_approx_sde} with step-size $(h\times missed[\cal{I}])$ and $(\theta_k, \tmo_k, \{X^i_k, \Qmo_k^i\}_{i\in \cal{I}})$. }
\STATE // Reset time missed
\STATE \textcolor{teal}{$missed[\cal{I}] = 0$}
\STATE // Take a step with step-size $h$
\STATE \textcolor{mypink2}{$\theta_{k+1}, \tmo_{k+1} \gets$ Solve \Cref{eq:theta_approx_sde,eq:tmo_approx_sde} with step-size $h$ with $(\theta_k, \tmo_k, \{X^i_{k+\epsilon}, \Qmo_{k+\epsilon}^i\}_{i\in \cal{I}})$.}
\STATE \textcolor{mypink2}{$X^\cal{I}_{k+1}, \qmo^\cal{I}_{k+1} \gets$ Solve \Cref{eq:x_approx_sde,eq:qmo_approx_sde} with step-size $h$ with $(\theta_k, \tmo_k, \{X^i_{k+\epsilon}, \Qmo_{k+\epsilon}^i\}_{i\in \cal{I}})$.}.
\STATE // Increment missed time for particles that are not updated
\STATE \textcolor{teal}{$missed[\text{not in }\cal{I}] =missed[\text{not in }\cal{I}] + 1$}.
\end{algorithmic}
\end{algorithm}

\subsection{Heuristic}
\label{appen:heuristic}
We (partially) circumvent the choice of momentum parameters by leveraging the relationship between \Cref{eq:dhamil_euclidean} and NAG to define a heuristic. For completeness, we briefly describe the heuristic here. For a given step-size $h$, one can define the ``momentum coefficient'' $\mu_\theta = 1 - h\gamma_\theta \eta_\theta$ (and similarly, for $\mu_x$). Since, in NAG, $\mu$ has typical values, we can use say $\mu_\theta=0.9$ with a fixed value of $\gamma_\theta$ to find a suitable value of $\eta_\theta$ (and similarly, for $\eta_x, \gamma_x$). In our experiments, we found that $\gamma_\theta \in [\frac{1}{2}, 5]$ performed well. Another possible approach to handling hyperparameters is to borrow inspiration from adaptive restart methods \citep{o2015adaptive}. While some practical heuristics exist, it seems that to a large extent, the problem of tuning these hyperparameters remains open; we leave this topic for future work.

In \Cref{sec:nesterov_aaflow}, we show how we can discretize \Cref{eq:mpd_flow} to obtain the following update in the $(\theta, \tmo)$-components:
\begin{align*}
	\theta_{k} &= \theta_{k-1} + v_{k}, \\
	v_{k} &= {\mu} v_{k-1} - h^2 \nabla_\theta  {\cal{E}}(\theta_{k-1} + \mu v_{k-1}, q_{k-1,X}^M),
\end{align*}
where $\mu_\theta  = 1 - h \gamma_\theta \eta_\theta$, and $h>0$ is the step size. We refer to the resulting iterations as MPD-NAG-Cheng. This is exactly the update used in \citet{sutskever2013importance}'s characterization of NAG with step-size $h^2$. Since there are some well-accepted choices for choosing $\mu_\theta$ in NAG (e.g., \citet{nesterov1983method} advocating for $\mu_t = 1-3/(t+5)$), we can leverage the formula $\mu_\theta  = 1 - h \gamma_\theta \eta_\theta$ to define a suitable choice of $\eta_\theta$ for a fixed $\gamma_\theta$. One can, of course, fix $\eta_\theta$ instead, but we found that suitable values for $\gamma_\theta$ are easier to choose from than $\eta_\theta$. We note that for MPD-NAG-Cheng, it depends only on the value of $\mu_\theta$ which is not the case for our discretization. In \Cref{fig:momentum_cofficient}, we show the effect of varying the momentum coeffect $\mu_\theta$, while keeping the damping parameter $\gamma_\theta$ fixed (and vice versa). In \Cref{fig:appendix_fix_damping_vary_momentum}, it can be seen that for a fixed damping parameter $\gamma_\theta$ and varying the momentum coefficient MPD-NAG-Cheng changes significantly while ours does not. As for the vice versa case, in \Cref{fig:appendix_fix_momentum_vary_damping}, it can be observed that MPD-NAG-Cheng does not change while ours varies significantly instead. This suggests that the momentum coefficient in our discretization no longer has the same interpretation as that in MPD-NAG-Cheng. Nonetheless, it can be observed that for a suitably chosen $\gamma_\theta$ and momentum parameter $\mu_\theta$, our discretization performs better than MPD-NAG-Cheng. Hence, we use this as a heuristic when choosing momentum parameters $(\gamma_\theta, \gamma_x, \eta_\theta, \eta_x)$.

\begin{figure}
    \centering
    \begin{subfigure}[b]{0.4\textwidth}
		\centering
		\includegraphics[width=\textwidth]{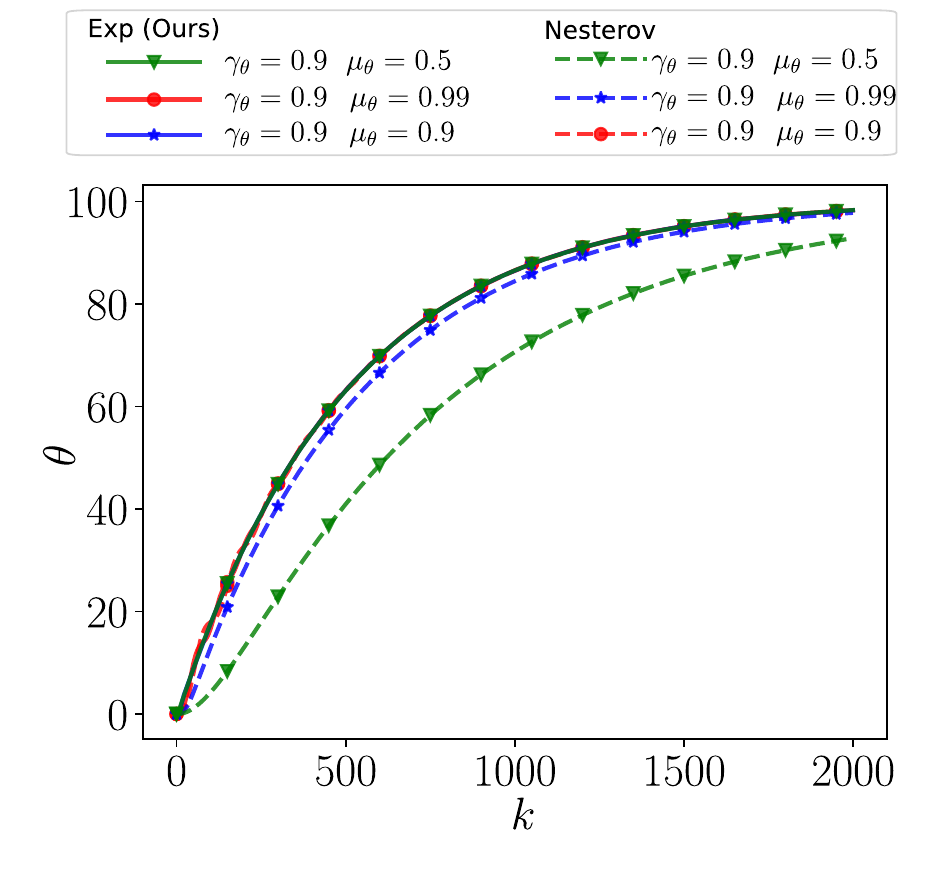}
		\caption{Fixed $\gamma_\theta$, varying $\mu_\theta$.}
        \label{fig:appendix_fix_damping_vary_momentum}
	\end{subfigure}
	\begin{subfigure}[b]{0.4\textwidth}
		\centering
		\includegraphics[width=\textwidth]{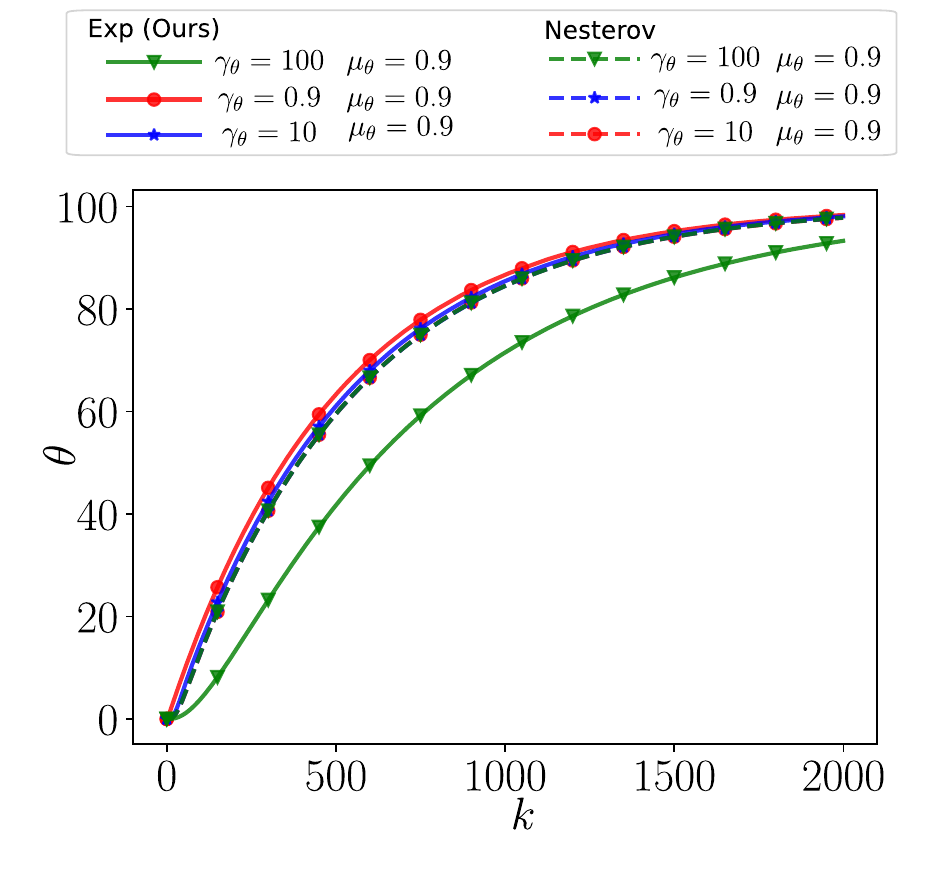}
		\caption{Fixed $\mu_\theta$, varying  $\gamma_\theta$.}
            \label{fig:appendix_fix_momentum_vary_damping}
	\end{subfigure}
    \caption{The effect of the damping parameter $\gamma_\theta$ and momentum coefficient $\mu_\theta$ on MPD.}
    \label{fig:momentum_cofficient}
\end{figure}

\label{sec:exp_details}

\subsection{Experimental Details}
\label{appendix:exp_details}
Here, we outline for reproducibility the settings: for the Toy Hierarchical Model (\Cref{appendix:toyhmm}); image generation experiment (\Cref{appendix:vae}); and, additional figures are in \Cref{app:add_figs}. Much akin to \citet[Section 2]{Kuntz2022}, we found it beneficial to consider separate time-scales for the $(\theta_t, \tmo_t)$ and $(X_t,\Qmo_t)$ which we denote $h_\theta$ and $h_x$ respectively.
\subsubsection{Toy Hierarchical Model}
\label{appendix:toyhmm}

\textbf{Model.}
The model is described in \citet[Example 1]{Kuntz2022}. For completeness, we describe it here. The model is defined by
\begin{equation}
    p_\theta(x, y) = \prod_{i=1}^{d_x} \cal{N}(y_i; x_i, 1) \cal{N}(x_i; \theta, \sigma^2).
    \label{eq:toyhmm_model}
\end{equation}
Thus, for the log-likelihood with $\sigma^2=1$, we have
$$
\log p_\theta(x, y) = C - \frac{1}{2} \sum_{i=1}^{d_x} \left ( (x_i - \theta)^2 + (y_i - x_i)^2 \right ),
$$
where $C$ is a constant independent of $\theta$ and $x$.

\textbf{\Cref{fig:different_regimes}}.
For the different regimes, we use the following momentum parameters:
\begin{itemize}
    \item \textbf{Underdamped}:  $\gamma_\theta=0.1, \eta_\theta=403.96, \gamma_x=0.1, \eta_x = 403.96$.
    \item \textbf{Overdamped}: $\gamma_\theta=1, \eta_\theta=403.96, \gamma_x=1, \eta_x = 403.96$.
    \item \textbf{Critically Damped} $\gamma_\theta=0.7, \eta_\theta=403.96, \gamma_x=0.7, \eta_x = 403.96$.
\end{itemize}
The step-sizes are $h_\theta = 10^{-2}/ N = 10^{-4}$, $h_x = 10^{-2}$, with number of samples $N=100$, and number of particles $M=100$.

\textbf{\Cref{fig:theta_integrators}}. We compare the discretization outlined in \Cref{sec:cheng_nag_mpd}, and \Cref{sec:mt_discretization}. We keep all parameters equal except for changing the momentum coefficient $\mu$. For step-sizes, we have $h_\theta = 10^{-5/2} \approx 5.8 \times 10^{-3}, h_x = 10^{-3}$, $\gamma_x=\gamma_\theta = 0.5$, with the momentum coefficient set to $\mu_\theta = \mu_x = \mu$ where $\mu \in \{0.9, 0.8, 0.5\}$.

\textbf{Lipschitz Constant and Strong log concavity.}
For the toy Hierarchical model with $\sigma^2=1$, one can show that it satisfies our Lipschitz assumption, as well as being strongly log-concave. We have
$$
\nabla_{(\theta, x)}^2 \log p_\theta
=
\begin{pmatrix}
	-d_x & 1_{1 \times d_x} \\
	1_{d_x \times 1} & -2I_{d_x} \\
\end{pmatrix}.
$$
The characteristic equation can be written as 
$$
\det(\nabla_{(\theta, x)}^2 \log p_\theta - lI_{d_{x + 1}}) = 0.
$$
The determinant can be evaluated by expanding the minors. Thus, we obtain
\begin{align*}
	\det(\nabla_{(\theta, x)}^2 \log p_\theta - lI_{d_{x + 1}}) =& (- d_x- l) (-2 - l)^{d_x} - d_x(-2 - l)^{d_x - 1} \\
	=&(-2 - l)^{d_x - 1} ((2 + l)(d_x + l) - d_x) \\
	=&(-2 - l)^{d_x - 1} (l^2 + (2 + d_x)l + d_x).
\end{align*}
Hence, the characteristic equation is satisfied when $l\in \{ -2, \frac{-(2+d_x) \pm \sqrt{d_x^2 + 4}}{2}\}$. Note that $\frac{-(2+d_x) - \sqrt{d_x^2 + 4}}{2} \le -2$.

From the above calculations, it can be seen that $\nabla_{(\theta, x)}^2 \log p_\theta$ is strongly-log concave, i.e., we have $\nabla_{(\theta, x)}^2 \log p_\theta \preceq - 2 I$.

In order to calculate its Lipschitz constant of \Cref{ass:gradlip}, by \citet[Lemma 1.2.2]{nesterov2003introductory}, this is equivalent to finding a constant $K_{hm} > 0$ such that 
$$
\| \nabla_{(\theta, x)}^2 \log p_\theta \| \le  K_{hm}.
$$
Since $\nabla_{(\theta, x)}^2 \log p_\theta$ is symmetric, then the matrix norm is given by the largest absolute value of the eigenvalue of $\nabla_{(\theta, x)}^2 \log p_\theta$ (i.e., it's spectral radius).

Hence, we have $K_{hm} = \frac{(2+d_x) + \sqrt{d_x^2 + 4}}{2}$.

\subsubsection{Image Generation}
\label{appendix:vae}
The dataset of both MNIST and CIFAR-10 is processed similarly. We use $N=5000$ images for training and $3000$ for testing. The model is updated $6280$ times using subsampling with a batch size of $32$. Hence, it completes $40$ epochs. The model is defined as:
$$
p_\theta(y, x) = \prod_{i=1}^N p_\theta(y^i, x^i),
$$
where:
\begin{itemize}
    \item The datum $(y^i, x^i) \in \mathbb{R}^{d_y} \times \mathbb{R}^{d_x}$ denotes a single image and its corresponding latent variable. For MNIST, we have $d_y = 28 \times 28 = 784$ and for CIFAR, we have $d_y = 32 \times 32 \times 3 = 3072$. For both datasets, we set $d_x = 64$. Thus, we have that $(y,x) \in \mathbb{R}^{d_y \times N} \times \mathbb{R}^{d_x \times N}$.
    \item For the VAE model, we have $p(y^i, x^i) = \cal{N}(y^i| g_\psi(x^i), \sigma^2I) p_\phi(x^i)$, where $\psi$ and $\phi$ are parameters of the generator and prior respectively. Thus, the parameter of the model is given by $\theta = \{\psi, \phi\}$. We set $\sigma^2 = 0.1$. The following specifies the details regarding the generator $g_\psi$ and prior $p_\phi$:
    \begin{itemize}
        \item \textbf{Generator $g_\psi$}. For CIFAR, we use a Convolution Transpose network (as in \citet{pang2020learning}) shown in \Cref{table:cifar_generator}. For MNIST, we use a fully connected network specified in \Cref{tab:mnist_generator} with $d_{in} = d_x$ and $d_{out} = d_y$.
    	\item \textbf{Prior $p_\phi (x)$.} The prior is inspired by the VampPrior \citep{tomczak2018vae} and is defined as:
    	$$
    	p_\phi(x) = \frac{1}{K}\sum_{i=1}^K \mathcal{N}(x|\mu_\nu(z_i), \sigma^2_\nu(z_i) I_{d_z}),
    	$$
    	where $\mu_\nu,\sigma^2_\nu:\mathbb{R}^{d_z}\rightarrow \mathbb{R}^{d_x}$ are neural networks (with parameters $\nu$) that parameterized the mean and variance; pseudo-points $\{z_i\}_{i=1}^K$ are optimized; and so the prior parameters are $\phi := \{\nu\} \cup \{z_i\}_{i=1}^K$. The architecture can be found in \Cref{tab:vi_mlp} where $d_{in}=d_z$ and $d_{out} =d_z$
    \end{itemize}

\end{itemize}

\begin{table}[]
\centering
\begin{tabular}{@{}cccc@{}}
\toprule
Layers                    & Size           & Stride & Pad \\ \midrule
Input                     & 1x1x128        & -      & -   \\
8x8 ConvT(ngf x 8), LReLU & 8x8x(ngf x 8)  & 1      & 0   \\
4x4 ConvT(ngf x 4), LReLU & 16x16x(ngf x 4) & 2      & 1   \\
4x4 ConvT(ngf x 2), LReLU & 32x32x(ngf x 2) & 2      & 1   \\
3x3 ConvT(3), Tanh        & 32x32x3  & 1      & 1   \\ \bottomrule
\end{tabular}
\caption{Generator network for the VAE used for CIFAR-10 (ngf = 256).}
\label{table:cifar_generator}
\end{table}
\begin{table}[]
    \centering
\begin{tabular}{ll}
\hline
Layers                    & Size      \\ \hline
Input                     & $d_{x}$  \\
Linear($d_{x}$, 512), LReLU & 512       \\
Linear(512,512), LReLU    & 512       \\
Linear(512,512), LReLU    & 512       \\
Linear(512,$d_{y}$), Tanh     & $d_{y}$ \\ \hline
\end{tabular}
    \caption{Generator network for the VAE used for MNIST.}
    \label{tab:mnist_generator}
\end{table}

\begin{table}[]
\centering
\begin{tabular}{ll}
\hline
Layers                               & Size         \\ \hline
Input                                & $d_{in}$         \\
Linear($d_{in}$, 512), LReLU            & 512          \\
Linear(512,512), LReLU               & 512          \\
Linear(512,512), LReLU               & 512          \\
Linear(512,$d_{out} \times 2$)             & $d_{out} \times 2$ \\
Output: Id([:$d_{out}$]), Softplus([$d_{out}$:]) & $d_{out}, d_{out}$     \\ \hline
\end{tabular}
\caption{Neural network parametrization of the mean and variance of a Gaussian Distribution used in VI and VAMPprior. ``Id'' is the identity function, and ``[:$d_{out}$]'' is a standard pseudo-code notation that refers to the operation that extracts the first $d_{out}$ dimensions and similarly for  [$d_{out}$:].}
\label{tab:vi_mlp}
\end{table}

\subsection{Hyperparameter settings}
Unless specified otherwise, we use the same parameters for both MNIST and CIFAR datasets. Our approach to selecting the step size involved first setting it to $10^{-3}$, then we would decrease it appropriately if instabilities arise.
\begin{itemize}
    \item \textbf{MPD}. For step sizes, we have $h_\theta = h_x= 10^{-4}$. The number of particles is set to $M=5$. For the momentum parameters, we use $\gamma_\theta= \gamma_x=0.9$ with the momentum coefficient $\mu_\theta = 0.95, \mu_x=0$ (or equivalently, $\eta_\theta\approx 556, \eta_x\approx 11,111$ )for MNIST and $\mu_\theta=0.95, \mu_x=0.5$  (or equivalently, $\eta_\theta\approx 556,, \eta_x\approx 55,555$ ) for CIFAR. We use the RMSProp preconditioner (see \Cref{sec:rmsprop_preconditioner}) with $\beta = 0.9$.
    \item \textbf{PGD}. We have $h_\theta = 10^{-4}, h_x =10^{-3}$. The number of particles is set to $M=5$. We use the RMSProp preconditioner (see \Cref{sec:rmsprop_preconditioner}) with $\beta = 0.9$.
    \item \textbf{ABP}. We use SGD optimizer with step size $h_\theta = 10^{-4}$, and run a length $5$ ULA chain $5$ with step-size $h_x =10^{-3}$. We found that RMSProp worked well in the transient period but failed to achieve better performance than SGD.
    \item \textbf{SR}. We used RMSProp optimizer with step size $h_\theta = 10^{-4}$, and ran five $20$-length ULA chain with step-size $h_x =10^{-3}$
    \item \textbf{VI}. We use RMSProp optimizer with a step size of $h = 10^{-5}$. We found that this resulted in the best performance. For the encoder, we use a Mean Field approximation as in \citet{kingma2013auto}. The neural network is a $3$-layer fully connected network with Leaky Relu activation functions. In the last layer, we use a softmax to parameterize the variance. See \Cref{tab:vi_mlp}.
\end{itemize}
\subsection{Additional Figures}
\label{app:add_figs}
\begin{figure}[ht]
	\begin{subfigure}[b]{0.24\textwidth}
		\centering
		\includegraphics[width=\textwidth]{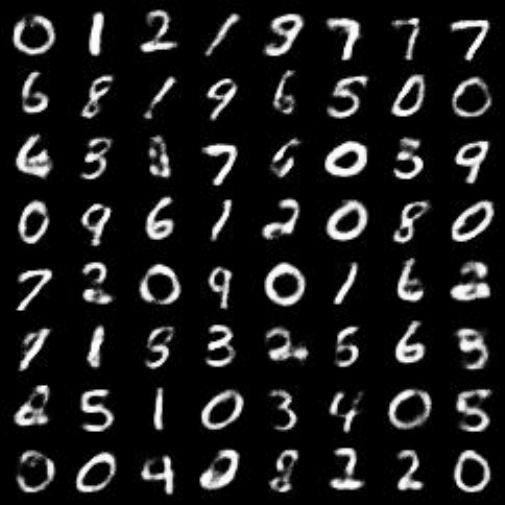}
		\caption{MPD.}
	\end{subfigure}
	\begin{subfigure}[b]{0.24\textwidth}
		\centering
		\includegraphics[width=\textwidth]{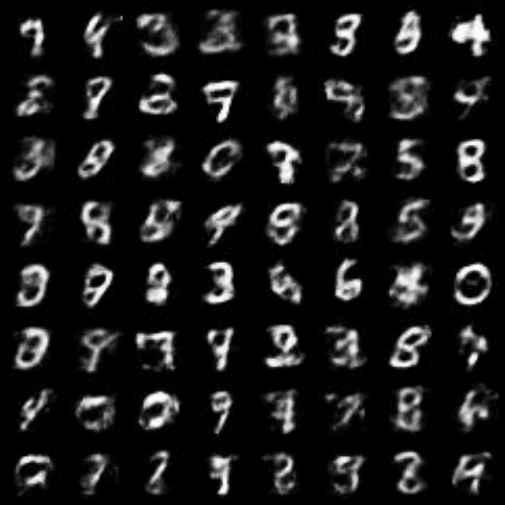}
		\caption{PGD.}
	\end{subfigure}
	\begin{subfigure}[b]{0.24\textwidth}
		\centering
		\includegraphics[width=\textwidth]{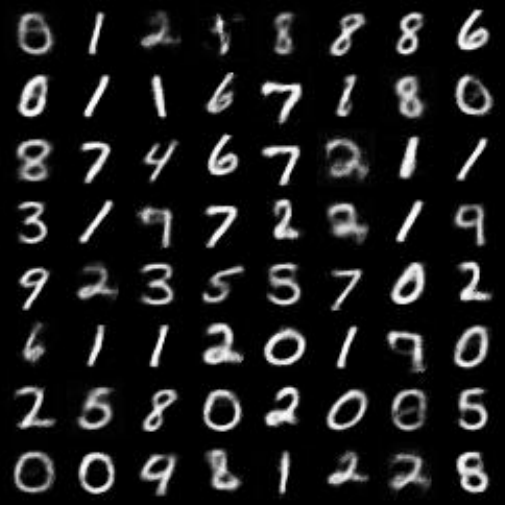}
		\caption{VI.}
	\end{subfigure}
	\begin{subfigure}[b]{0.24\textwidth}
		\centering
		\includegraphics[width=\textwidth]{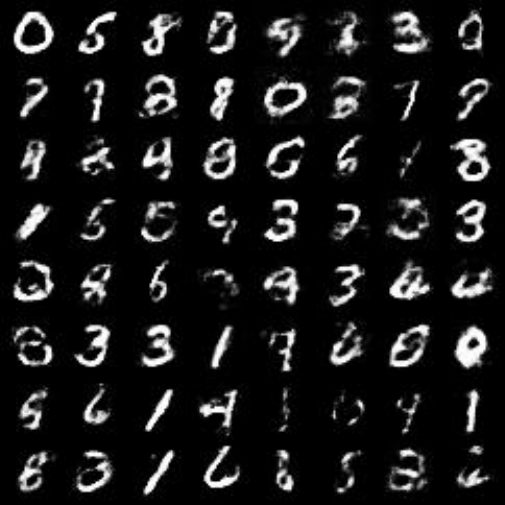}
		\caption{ABP.}
		\label{fig:mnist_}
	\end{subfigure}
	\caption{\textbf{MNIST}. Samples generated from various algorithms.}
	\label{fig:mnist_samples}
\end{figure}

\begin{figure*}[ht]
\centering
	\begin{subfigure}[b]{0.24\textwidth}
		\centering
		\includegraphics[width=\textwidth]{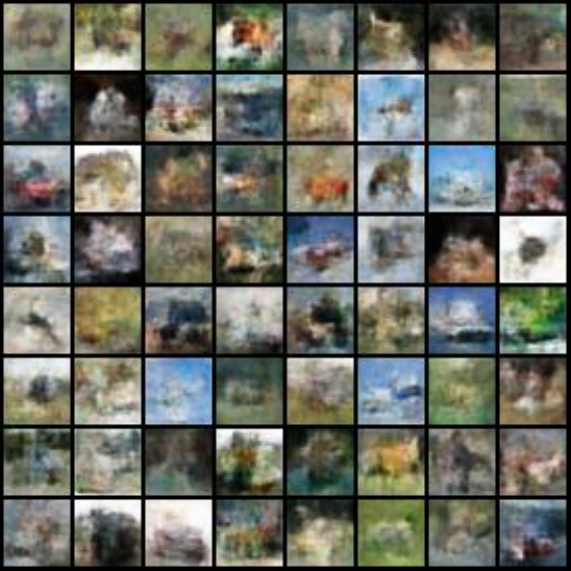}
		\caption{MPD.}
	\end{subfigure}
	\begin{subfigure}[b]{0.24\textwidth}
		\centering
		\includegraphics[width=\textwidth]{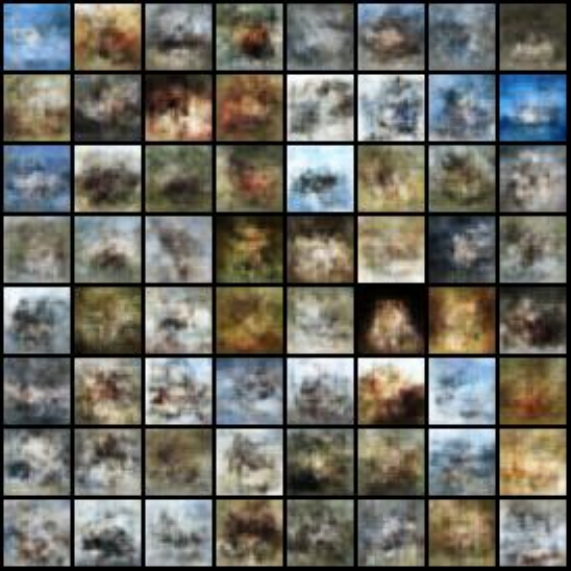}
		\caption{PGD.}
	\end{subfigure}
	\begin{subfigure}[b]{0.24\textwidth}
		\centering
		\includegraphics[width=\textwidth]{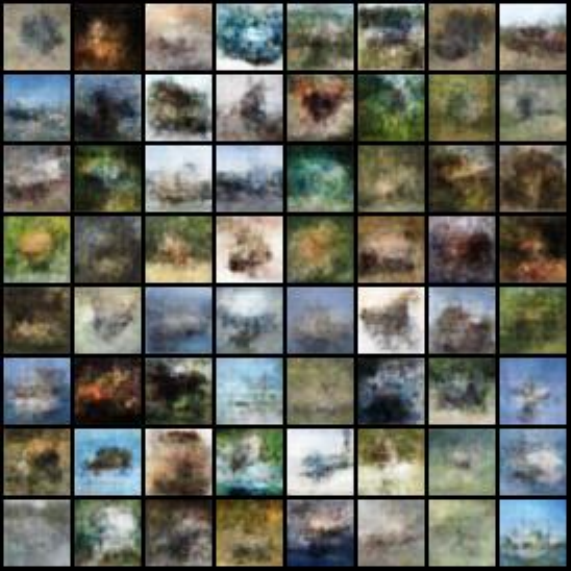}
		\caption{VI.}
	\end{subfigure}
	\begin{subfigure}[b]{0.24\textwidth}
		\centering
		\includegraphics[width=\textwidth]{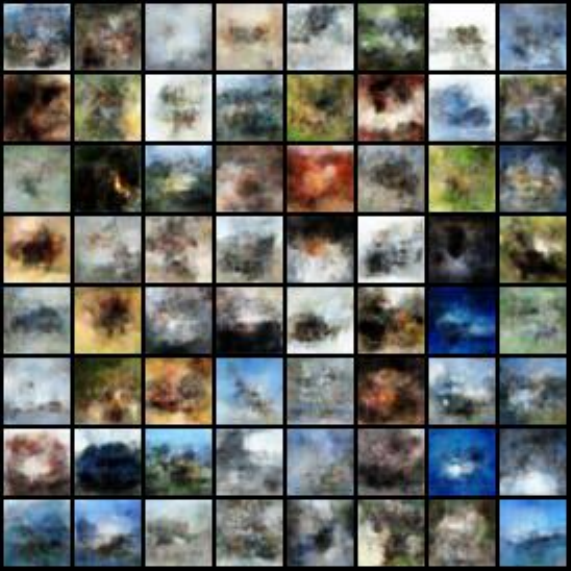}
		\caption{ABP.}
	\end{subfigure}
	\caption{\textbf{CIFAR-10}. Samples generated from various algorithms.}
	\label{fig:cifar_samples}
 \vspace{-2mm}
\end{figure*}

\section{Additional Experiments}
\label{appendix:more_experiments}
In this section, we show additional experiments particularly concerned with comparing MPD with PGD and other variants of MPD that only accelerated one component of the space instead of two. We say $X$-enriched to indicate algorithms with momentum incorporated in $X$ with either momentum included/excluded in $\theta$. To separate the two cases, we call $X$-only-enriched to refer to the algorithm with gradient descent in $\theta$. We use similar terminology for $\theta$-enriched and $\theta$-only-enriched.

We consider two settings: in \Cref{appendix:additional_thm}, further results with the toy Hierarchical model (see \Cref{appendix:toyhmm}); and, in \Cref{sec:appendix:additional_mog} a density estimation task using VAEs on a Mixture of Gaussian dataset.

\subsection{Toy Hierarchical Model}
\label{appendix:additional_thm}
\textbf{\Cref{fig:toy_hmm_scale}.} As the model, we consider a Toy Hierarchical model for different $\sigma$ values in \eqref{eq:toyhmm_model} where $\sigma$ is chosen to be the same as the data generating process. The data is generated from a toy Hierarchical model with parameters $(\theta =10, \sigma)$ where $\sigma = \{5, 10, 12\}$. As noted by \citet[Eq.\ 51]{Kuntz2022}, the marginal maximum likelihood of $\theta$ is the empirical average of the observed samples, i.e., $\frac{1}{d_x}\sum_{i=1}^{d_x} y_i$. For each trial, we process the data $(y_i)_{i=1}^{d_x}$ to have an empirical average of $10$ for simplicity.

\textbf{\Cref{fig:why_accelerate_both} (a), (b), (c)}. We use a toy Hierarchical model with $\sigma =12$, and, similarly, the data-generating processing is a toy Hierarchical model with $(\theta = 10, \sigma=12)$.  We are interested in how the initialization of the particle cloud affects MPD, PGD and other algorithms that only momentum-enriched one component. Each subplot shows the evolution of the parameter $\theta$ for the initialization from $\cal{N}(\mu, I)$ where $\mu \in \{-5, -20,-100\}$. This example serves as an illustration of the typical case where the cloud is initialized badly. It can be seen that all methods are affected by poor initialization. For $X$-enriched, the algorithms can recover faster to achieve almost identical performances (at the later iterations) between different initializations. In other words, for $X$-enriched algorithms transient phase is short. Methods without such $X$-enrichment take longer to recover and are unable to achieve similar performances compared to the cases where they are better-initialized. Overall, it can be seen that MPD perform better than PGD.

\begin{figure}[ht]
    \centering
    \begin{subfigure}[b]{0.32\textwidth}
        \centering
        \includegraphics[width=\textwidth]{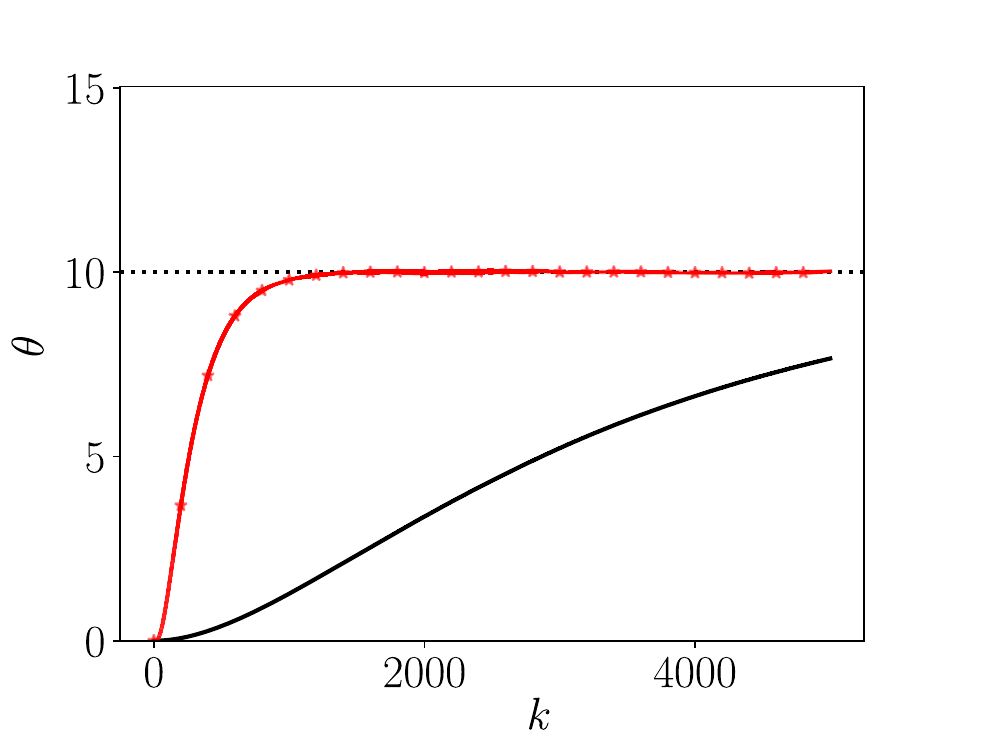}
        \caption{$\sigma =5$.}
    \end{subfigure}
    \begin{subfigure}[b]{0.32\textwidth}
        \centering
        \includegraphics[width=\textwidth]{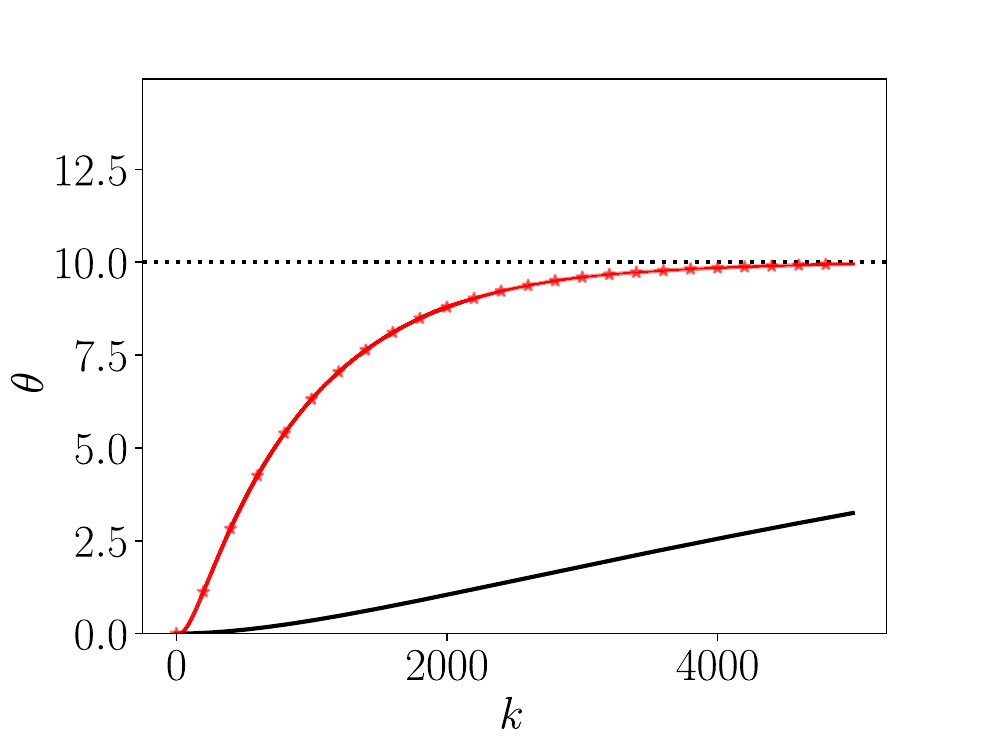}
        \caption{$\sigma = 10$.}
    \end{subfigure}
        \begin{subfigure}[b]{0.32\textwidth}
        \centering
        \includegraphics[width=\textwidth]{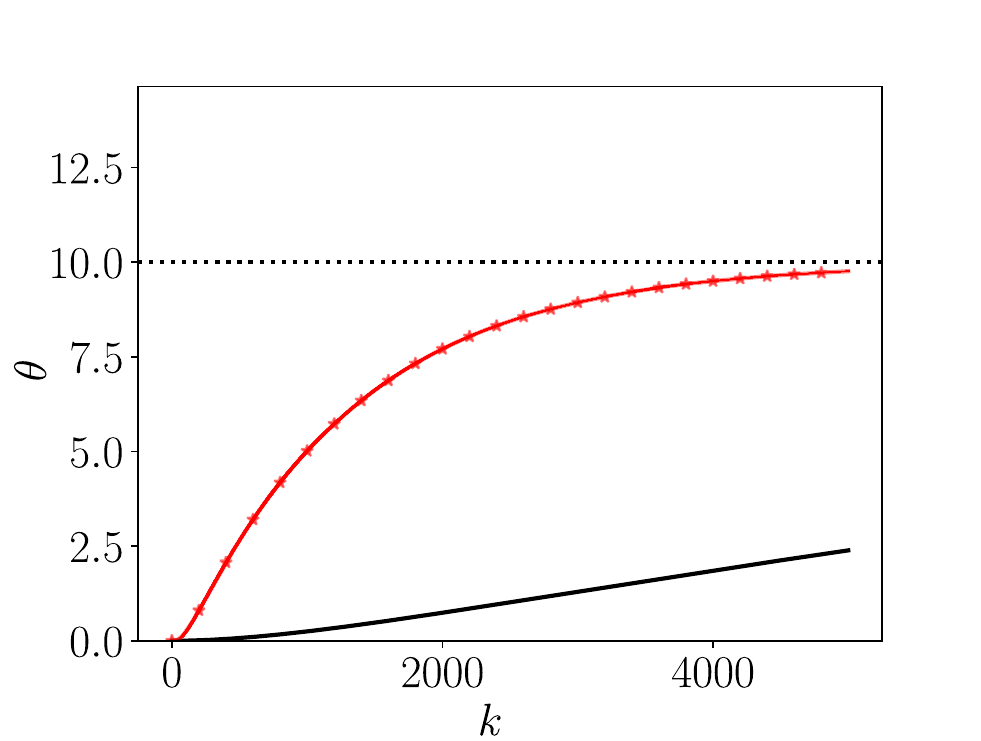}
        \caption{$\sigma = 12$.}
    \end{subfigure}
    \caption{Comparison between MPD and PGD on the Toy Hierarchical Model for different choices of $\sigma$. The plot shows the average and standard deviation of $\theta$ across iterations computed over $10$ independent trials. MPD is shown in \textcolor{red}{red} and PGD in \textcolor{black}{black}. The dashed line shows the true value.}
    \label{fig:toy_hmm_scale}
\end{figure}

\subsection{Density Estimation}
\label{sec:appendix:additional_mog}
In this problem, we show the result of training a VAE with VAMP Prior on a 1-d Mixture of Gaussians (MoG) dataset with PDF shown in \Cref{fig:mog_pdf}. The dataset is composed of $100$ samples. In  \Cref{fig:why_accelerate_both} (d), we show the performance of various acceleration algorithms. The empirical Wasserstein-1 distance (denoted by $\hat{\sf{W}}_1$) is estimated from the empirical CDF produced from $1000$ samples across each iteration. The results show the average (with standard error bars which are barely noticeable) computed over $100$ trials.  PGD is shown in \textcolor{black}{black}, and MPD is shown in \textcolor{green}{green}. In this case, all accelerated algorithms perform better than PGD, and MPD performs best of all.

\textbf{Hyperparameters}: No subsampling used. We have $d_y=1, d_x=10$.  We set $\gamma_x = \gamma_\theta = 0.4$, and use the momentum heuristic with $\mu_\theta = \mu_x = 0.1$ (see \Cref{appen:heuristic}). We use the same VAE architecture in \Cref{appendix:vae} with likelihood noise $\sigma=0.01$, an MLP decoder (see \Cref{tab:vi_mlp}), and VAMP prior with  $d_z=2$ and $K=20$.
\begin{figure}[ht]
\centering
    \includegraphics[width=0.4\textwidth]{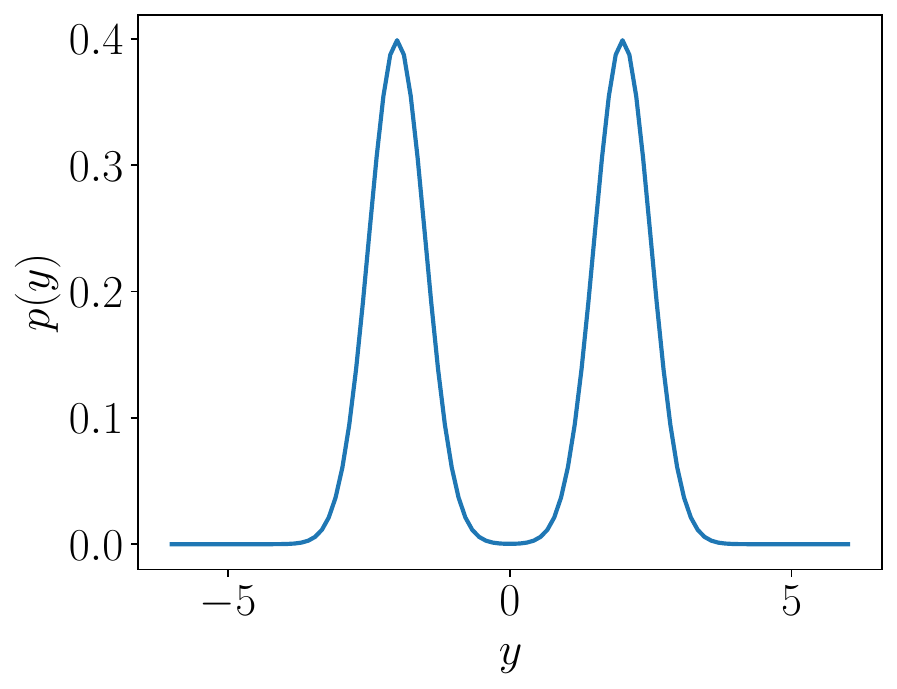}
    \caption{The true density of the MoG dataset.}
    \label{fig:mog_pdf}
    \vspace{-3mm}
\end{figure}
\subsection{Sensitivity Analysis}
\label{app:extra_sensitivity}

We are interested in investigating the sensitivity of MPD to different momentum parameters. One way to quantify whether you are in a regime where MPD is an improvement to PGD is by computing an approximation to the (signed and weighted) area between PGD and the MPD curves called \textit{Area Between Curves} (ABC) More specifically, we compute
\begin{equation*}
    \mathrm{ABC} := \frac{1}{K}\sum_{k=1}^K w(k)[\mathrm{PGD}(k) - \mathrm{MPD}(k)],
\end{equation*}
where $\mathrm{PGD}: \mathbb{Z}_{+} \rightarrow \bb{R}$ and $\mathrm{MPD}: \mathbb{Z}_{+} \rightarrow \bb{R}$ denote the loss/metric curves of PGD and MPD respectively, and $w(k)$ is a weighting function such that $\sum_{k=1}^Kw(k)=1$. The role of the factor $\frac{1}{K}$ is to normalize and $w(k)$ is to reweight the sum to have more importance at the end of the run than the beginning (we use ${w}(k)= \frac{2}{K(1+K)}\cdot k/K$). If MPD is in a desirable regime and performs better than PGD then the curve of $\mathrm{MPD}$ will lower bound the curve of $\mathrm{PGD}$. In this case, ABC will obtain large positive values and in the converse case, ABC will have small negative, hence a higher ABC indicates better hyperparameters.
\begin{figure}[ht]
    \centering
    \includegraphics[width=0.4\linewidth]{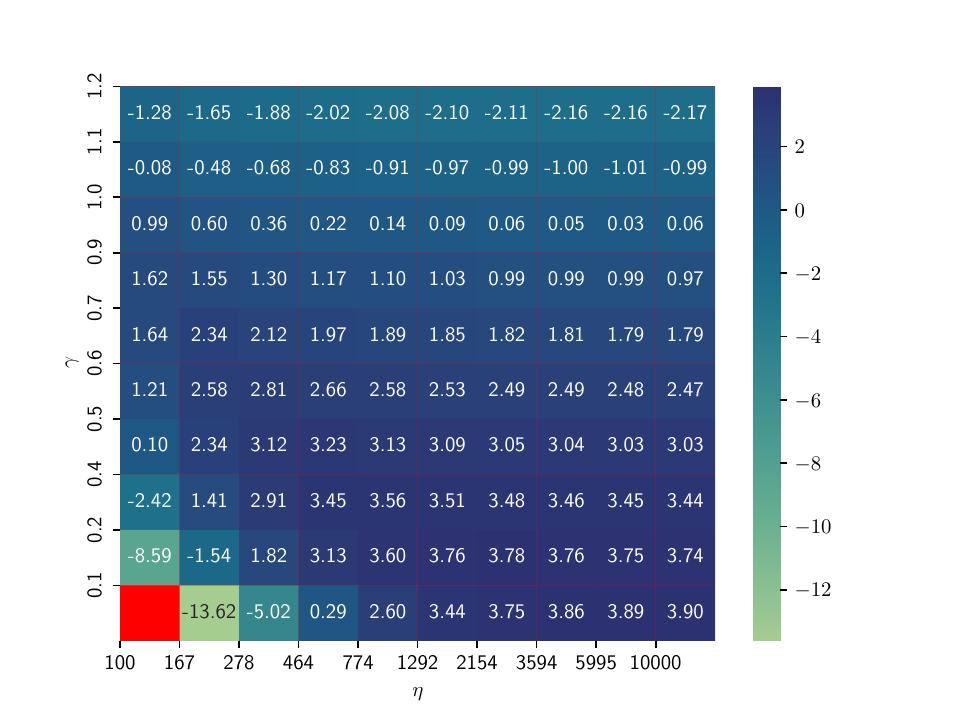}
    \caption{The Area Between the Curve of PGD and MPD for different momentum parameters.}
    \label{fig:abc}
\end{figure}
\Cref{fig:abc} shows the ABC for different momentum parameters $\gamma$ and $\eta$ while keeping all other hyperparameters the same between PGD and MPD. We use the toy Hierarchical model which allows us to have access to the true parameter $\theta$, so ABC is computed from curves $\rm{MPD}$ and $\rm{PGD}$ that returns the absolute difference between the model and true $\theta$. We note that a similar heatmap can be produced using the loss curve. This heatmap supports our recommendation that if we set $\eta$ to be sufficiently large, we can easily tune $\gamma$ to achieve better performance than PGD. Note that $\eta$-axis is in log scale and that this plot only applies to toy Hierarchical Model.
\end{document}